\documentclass[10pt]{article}
\usepackage[top=1in, bottom=1in, left=0.9in, right=0.9in]{geometry}

\usepackage{zhiyuz}

\title{\textbf{Adversarial Tracking Control via Strongly Adaptive Online Learning with Memory}}
\author{
  Zhiyu Zhang \\
  Boston University\\
  \texttt{zhiyuz@bu.edu}\\
  \and
  Ashok Cutkosky \\
  Boston University\\
  \texttt{ashok@cutkosky.com}\\
  \and
  Ioannis Ch. Paschalidis\\
  Boston University\\
  \texttt{yannisp@bu.edu}\\
}
\date{\vspace{-5ex}}

\begin{document}
\maketitle

\begin{abstract}
We consider the problem of tracking an adversarial state sequence in a linear dynamical system subject to adversarial disturbances and loss functions, generalizing earlier settings in the literature. To this end, we develop three techniques, each of independent interest. First, we propose a comparator-adaptive algorithm for \emph{online linear optimization with movement cost}. Without tuning, it nearly matches the performance of the optimally tuned gradient descent in hindsight. Next, considering a related problem called \emph{online learning with memory}, we construct a novel strongly adaptive algorithm that uses our first contribution as a building block. Finally, we present the first reduction from adversarial tracking control to strongly adaptive online learning with memory. Summarizing these individual techniques, we obtain an adversarial tracking controller with a strong performance guarantee even when the reference trajectory has a large range of movement.
\end{abstract}

\section{Introduction}

Regulation and tracking are two iconic branches of linear control problems based on the system equation
\begin{equation*}
x_{t+1}=A_tx_t+B_tu_t+w_t.
\end{equation*}
By designing the action $u_t$, a \emph{regulation} controller rejects the disturbance $w_t$ such that the state $x_t$ remains close to the origin. In comparison, a \emph{tracking} controller aims at steering the state $x_t$ to follow a reference trajectory $x^*_t$. Recently, there have been growing efforts applying online learning ideas to linear control, including the online \emph{Linear Quadratic Regulator} (LQR) \citep{abbasi2011regret,cohen2018online,dean2019sample}, its adversarial generalizations \citep{agarwal2019online,agarwal2019logarithmic} and model-predictive control \citep{li2019online,yu2020power}. However, most of these advances are based on the regulation problem, and their application to tracking requires that the controller already knows the reference trajectory.

In this paper, we address this gap by first solving a general online learning problem: strongly adaptive online learning with movement cost. Ordinary (non-adaptive) algorithms aim to produce actions whose performance is strong on average over the entire operation of the algorithm. In contrast, a strongly-adaptive algorithm's performance must be strong \emph{over any time interval} of operation. This additional requirement significantly complicates the algorithm design. In fact, standard approaches to achieving strong adaptivity fail to account for movement costs, and cannot be easily modified to incorporate this extra performance metric. 

After presenting our results in online learning, we come back to linear control and consider a general \emph{adversarial tracking} problem with the following challenges. 
\begin{enumerate}
\item The system dynamics $(A_t,B_t)$ are time-varying.
\item The reference trajectory is fully adversarial. That is, $x^*_t$ can freely adapt to past actions of the controller, and we do not impose any assumption on its movement speed $\norms{x^*_t-x^*_{t-1}}$. 
\item The loss function $l^*_t$ that quantifies the tracking performance is adversarial, and we do not require its minimizer to be unique. This generalizes the quadratic loss from existing works on adversarial tracking, and allows the modeling of \emph{target regions}. 
\item The disturbance $w_t$ is adversarial, possibly combining noise, modeling error and (minor) nonlinearity. 
\end{enumerate}

Such a setting is useful for many practical problems, especially when the target to be tracked is hard to model and predict. However, due to the confluence of these challenges, existing controllers either cannot be applied, or cannot produce a regret bound that competes with a strong enough baseline. Taking a conceptual leap, we will provide a solution by exploiting a novel connection between adversarial tracking control and strongly adaptive online learning. 

\subsection{Our contribution} \label{sec:cont}

In this paper, we develop three techniques, each using the previous one as its building block. 
\begin{enumerate}
\item We propose the first comparator-adaptive algorithm for \emph{Online Linear Optimization (OLO) with movement cost}. This is nontrivial as the per-step movement of existing comparator-adaptive OLO algorithms can be exponentially large in $T$. (Section~\ref{section:contribution1})
\item We propose a novel strongly adaptive algorithm for \emph{Online Convex Optimization with Memory (OCOM)}, and the obtained bound further adapts to the observed gradients. (Section~\ref{section:saocom})
\item We propose the first reduction from adversarial tracking control to strongly adaptive OCOM. Our approach establishes a connection between two separate notions of tracking from online learning and linear control, which could facilitate the application of online learning ideas in a wider range of control problems. (Section~\ref{section:tracking})
\end{enumerate}

Combining these individual techniques, we design a strongly adaptive adversarial tracking controller: on any time interval $\I$ contained in the time horizon $[1:T]$, the proposed controller suffers $\tilde O(\sqrt{\abs{\I}})$ regret against the best $\I$-dependent static controller, where $\abs{\I}$ is the length of this time interval. More intuitively, on any time interval $\I$, the proposed controller always pursues \emph{the best fixed action for $\I$}. Such a performance guarantee significantly improves existing results, especially when the reference trajectory has a large range of movement. Finally, our theoretical results are supported by experiments.

\subsection{Background and notation}\label{subsection:related}

\paragraph{Linear tracking control} Tracking control is a decades old problem in linear control theory. Despite the empirical success of heuristic approaches (e.g., the PID controller), classical theoretical analysis typically requires strong assumptions on the reference trajectory: either (i) the reference trajectory is generated by a known linear system \citep{Astolfi2015}; or (ii) predictions are available \citep{Limon2015}. 

For us, the most relevant works are the learning-based approaches with regret guarantees of the form
\begin{equation*}
L(\textrm{alg})-\min_{C\in\C}L(C)\leq \textrm{Regret~bound}.
\end{equation*}
$\C$ is a set of baseline controllers called \emph{comparator class}. $L(\textrm{alg})$ and $L(C)$ are the cumulative loss of the proposed algorithm and a comparator $C\in\C$, respectively. A strong guarantee requires not only a small regret bound, but also a comparator class that contains a good tracking baseline. From this perspective, we discuss the limitation of existing works as follows. 
\begin{enumerate}
\item Abbasi-Yadkori et al. \cite{abbasi2014tracking}; Foster and Simchowitz \cite{foster2020logarithmic} proposed algorithms for tracking fully adversarial targets, and a nonconstructive minimax guarantee was proposed by Bhatia and Sridharan \cite{bhatia2020online}. However, regret bounds are only established on the entire time horizon $[1:T]$, and the comparator controllers are static and affine in the state ($u_t=-Kx_t+c$) which only perform well if the reference trajectory is roughly constant (on $[1:T]$). 
\item Another line of research \citep{agarwal2019online,agarwal2019logarithmic,simchowitz2020making,simchowitz2020improper,minasyan2021online} considered nonstochastic control, a general control setting with adversarial disturbances and loss functions. The comparator class is a collection of stabilizing linear controllers, therefore the implicitly assumed goal is \emph{disturbance rejection} (i.e., regulation) rather than tracking. We will provide a detailed discussion in Section~\ref{subsection:difference}. 
\end{enumerate}
In summary, designing an adversarial tracking controller with a strong theoretical guarantee remains an open problem. Next, we review classical settings of online learning and a special \emph{tracking} concept therein.

\paragraph{Basic online learning models} There are two standard online learning models \citep{zinkevich2003online} relevant to our purpose: \emph{Online Convex Optimization} (OCO) and \emph{Online Linear Optimization} (OLO). OCO is a two-person game: in each round, a player makes a prediction $x_t$ in a convex set $\V$, observes a convex loss function $l_t$ selected by an adversary and suffers the loss $l_t(x_t)$. If $l_t$ is linear, then the problem is also called OLO. The standard performance metric is the static regret: $\reg_{[1:T]}=\sum_{t=1}^Tl_t(x_t)-\min_{u\in \V} l_t(u)$. In general, OCO can be converted into OLO through the inequality $\reg_{[1:T]}\le \max_{u\in \V}\sum_{t=1}^T \langle g_t, x_t -u\rangle$ where $g_t\in \partial l_t(x_t)$, so it suffices to only consider OLO. 

\paragraph{Adaptive online learning} In this paper, we call adaptivity the property of an OLO algorithm such that on any time interval $\I\subset[1:T]$, it guarantees \emph{small} regret bound $\reg_\I$ against the best $\I$-dependent static comparator. Early works \citep{hazan2009efficient,adamskiy2016closer} studied \emph{weakly adaptive algorithms} where $\reg_\I=\tilde O(\sqrt{T})$. Improving on those, recent advances \citep{daniely2015strongly,jun2017improved,zhang2019adaptive,zhang2019dual} focused on a more powerful concept called \emph{strong adaptivity}: an algorithm is strongly adaptive if for all $\I\subset[1:T]$, $\reg_\I=O(\textrm{poly}(\log T)\cdot \sqrt{|\I|})$. This is much stronger than weak adaptivity, especially on short time intervals. 

To associate adaptivity with adversarial tracking control, let us consider the \emph{tracking regret} \citep{herbster1998tracking,bousquet2002tracking} in online learning as an intermediate step, where an OLO algorithm is compared to all sequences with bounded amount of switching. This generalizes the static regret, and interestingly, Daniely et al. \cite{daniely2015strongly} showed that near-optimal tracking regret can be derived from strong adaptivity. The key idea is that strongly adaptive OLO algorithms can quickly respond to the incoming losses, resulting in a near-optimal regret on the entire time horizon compared to \emph{nonstationary comparators}. This bears an intriguing similarity to tracking \emph{nonstationary targets} in linear control, which we exploit later. 

As for the design of adaptive OLO algorithms, the predominant approach is a two-level composition pioneered by Hazan and Seshadhri \cite{hazan2009efficient}. Notably, Cutkosky \cite{cutkosky2020parameter} proposed an alternative framework based on \emph{comparator-adaptive} online learning \citep{mcmahan2014unconstrained,orabona2016coin,foster2018online,van2019user,mhammedi2020lipschitz}. Our construction will incorporate movement cost into the latter, which is a highly nontrivial task.

\paragraph{Strongly adaptive OCOM} The performance of control suffers from past mistakes, therefore when we reduce it to online learning the resulting setting should also model this behavior, leading to a popular problem called \emph{Online Convex Optimization with Memory} (OCOM) \citep{anava2015online}. A weakly adaptive OCOM algorithm was proposed in \citep{gradu2020adaptive}, but achieving strong adaptivity is a much more challenging task due to two contradictory requirements: (i) strong adaptivity requires the predictions to move (i.e., respond to incoming information) very quickly; but (ii) movement cost requires the predictions to move slowly. 

Recently, Daniely and Mansour \cite{daniely2019competitive} proposed a strongly adaptive algorithm for OCOM with one-step memory, and its key component is an asymmetrical expert algorithm from Kapralov and Panigrahy \cite{kapralov2010prediction}. In comparison, our approach (Contribution 2) is based on a fundamentally different mechanism and analysis. Our obtained bound adapts to the observed gradients, and more importantly provides an alternative line of intuition to the regret-movement trade-off in strongly adaptive online learning.

For conciseness, further discussion on existing works is deferred to Appendix~\ref{section:additionalreview}, including a series of related but incomparable works on \emph{linear control with prediction}. 

\paragraph{Notation}We use $\norms{\cdot}$ for the Euclidean norm of vectors and the spectral norm of matrices. These are the default norms throughout this paper. Let $0$ be a zero vector or matrix. Let $\Pi_{\V}(x)$ be the Euclidean projection of $x$ to a set $\V$. $\ball^d(x,r)$ denotes the Euclidean norm ball centered at $x\in\R^d$ with radius $r$. 

For two integers $a\leq b$, $[a:b]$ is the set of all integers $c$ such that $a\leq c\leq b$; the brackets are removed when on the subscript, denoting a finite sequence with indices in $[a:b]$. Let $\abs{\cdot}$ be the cardinality of a finite set. Given square matrices $M_{a:b}$, define their product as $\prod_{i=a}^bM_i=M_b\cdots M_a$. (When $b<a$, the product is the identity matrix.) Finally, $\log$ denotes natural logarithm when the base is omitted. 

\section{Comparator-adaptive OLO with movement cost}\label{section:contribution1}

Starting with our first contribution, we introduce a comparator-adaptive algorithm for a variant of OLO called \emph{OLO with movement cost}. The difference from standard OLO is that in each round, besides suffering the instantaneous loss $l_t(x_t)$, the player also suffers a movement cost $\lambda\abs{x_t-x_{t-1}}$ where $\lambda$ is a known constant. Movement penalties have been studied in online learning in various forms \citep{kalai2005efficient,cesa2013online,gofer2014higher,bhaskara2021power,sherman2021lazy}, sometimes under the name \emph{switching cost} originated from the bandit problems. Since this paper focuses on the continuous domain, we name it as movement cost to avoid confusion. Notably, our setting is different from another classical problem called \emph{Smoothed OCO} \citep{chen2018smoothed,goel2019beyond}, where the loss function is observed \emph{before} making the prediction. 

Our algorithm is first developed on a one-dimensional domain $[0,\bar R]$, and then extended to higher dimensions.

\subsection{The one-dimensional algorithm}\label{subsection:1d}

We present the one-dimensional version in Algorithm~\ref{algorithm:1d}. It critically relies on a duality between OLO and the \emph{coin-betting game} \citep{orabona2016coin}, which we summarize in Appendix~\ref{subsection:coinbettingOLO}. Four hyperparameters are required: $\lambda$ is the weight of movement costs, $\gamma$ is a regularization weight, $G$ is the Lipschitz constant (assumed known) of the OLO losses, and $\eps$ is the ``budget'' for the cumulative cost and movement. 

\begin{algorithm*}[ht]
\caption{One-dimensional comparator-adaptive OLO with movement cost.\label{algorithm:1d}}
\begin{algorithmic}[1]
\REQUIRE Hyperparameters $(\lambda,\gamma,\eps,G)$, with $\lambda,\gamma\geq 0$ and $\eps,G>0$; a 1-dimensional domain $\V_{1d}=[0,\bar R]$; loss gradients $g_1, g_2,\ldots\in\R$ with $\abs{g_t}\leq G$, $\forall t$.  
\STATE \label{line:Cdef}Initialize internal variables as $\wel_0=\eps$, and $\beta_1,x_1,\tilde x_1=0$. Define $C=G+\lambda+\gamma$.
\FOR{$t=1,2,\ldots$}
\STATE \label{line:1d_surrogate}Make a prediction $x_t$, observe a loss gradient $g_t$. Define the surrogate loss $\tilde g_t$ as
\begin{equation*}
\tilde g_t=\begin{cases}
g_t, &\textrm{if~}g_t\tilde x_t\geq g_tx_t,\\
0, &\textrm{otherwise}.
\end{cases}
\end{equation*}

\STATE Let $\hat\beta_{t+1}=-\sum_{i=1}^t\tilde g_i/(2C^2t)$. Define $\B_{t+1}=[0,1/(C\sqrt{2t})]$ and let $\beta_{t+1}=\Pi_{\B_{t+1}}(\hat\beta_{t+1})$.

\STATE \label{line:wel}Assign $\wel_t$ as the solution to the following equation (uniqueness shown in Lemma~\ref{lemma:unique_solution}),
\begin{equation}
\wel_t=(1-\tilde g_t\beta_t-\gamma\beta_t/\sqrt{t})\wel_{t-1}-\lambda|\beta_t\wel_{t-1}-\beta_{t+1}\wel_t|.\label{eq:wel_update}
\end{equation}

\STATE \label{line:1d_projection}Let $\tilde x_{t+1}=\beta_{t+1}\wel_{t}$ and $x_{t+1}=\Pi_{\V_{1d}}(\tilde x_{t+1})$.
\ENDFOR
\end{algorithmic}
\end{algorithm*}

To get the gist of this algorithm, let us briefly ignore the surrogate loss $\tilde g_t$ from Line~\ref{line:1d_surrogate} and the projection of $\tilde x_{t+1}$ from Line~\ref{line:1d_projection} (i.e., assume $\bar R= \infty$). With $g_t=\tilde g_t$ and $x_{t+1}=\tilde x_{t+1}$, Algorithm~\ref{algorithm:1d} becomes an unconstrained OLO algorithm with predictions recommended by the following betting scheme: A bettor has money $\wel_t$ in the $t$-th round. After choosing a betting fraction $\beta_{t+1}$, he bets money $x_{t+1}=\beta_{t+1}\wel_t$ on the next loss gradient $g_{t+1}$. The favorable outcome is $g_{t+1}x_{t+1}$ being negative which means the OLO algorithm suffers \emph{negative loss}. Therefore, after observing $g_{t+1}$, the bettor treats $-g_{t+1}x_{t+1}$ as the money he gains and updates his wealth accordingly. Since large movement is also undesirable, the bettor further loses money proportional to the change of his betting amount; this is an important and novel step in our approach. Using this procedure, regret minimization is converted to wealth maximization. By choosing the betting fraction $\beta_t$ properly, one can simultaneously ensure low cost and low movement in OLO. 

\begin{restatable}{theorem}{oneD}\label{thm:1d}
For all $\lambda,\gamma\geq 0$, $G>0$ and $0<\eps\leq G\bar R$, with any loss sequence such that $|g_t|\leq G$ for all $t$, applying Algorithm~\ref{algorithm:1d} yields the following guarantee.
\begin{enumerate}
\item For all $T\in\N_+$ and $u\in\V_{1d}$, with $C$ defined in Line~\ref{line:Cdef} of the algorithm,
\begin{equation*}
\sum_{t=1}^T\left(g_tx_t-g_tu+\lambda\abs{x_t- x_{t+1}}+\frac{\gamma}{\sqrt{t}}\abs{x_t}\right)\leq \eps+uC\sqrt{2T}\rpar{\frac{3}{2}+\log\frac{\sqrt{2}uCT^{5/2}}{\eps}}.
\end{equation*}
\item For all $a\leq b$, $\sum_{t=a}^b\abs{x_t-x_{t+1}}\leq 48\bar R\sqrt{b-a+1}$.
\end{enumerate}
\end{restatable}

The highlights of Theorem~\ref{thm:1d} are the following. 

\begin{enumerate}
\item Part 1 provides the first comparator-adaptive bound for OLO with movement cost: the sum of movement cost and regret with respect to the null comparator $u=0$ is at most a user-specified constant, and the sum grows almost linearly in $|u|$ which is the optimal rate \citep[Chapter 5]{orabona2019modern}. This leads to an important \emph{parameter-free} property: \emph{without knowing the optimal comparator $u^*$ in advance}, Algorithm~\ref{algorithm:1d} automatically adapts to it, and the performance bound almost matches the optimally-tuned \emph{Online Gradient Descent} (OGD) whose learning rate depends on $u^*$. Note that the latter is a hypothetical (unimplementable) baseline, since the optimal comparator $u^*$ in hindsight is unknown before all the losses are revealed. Nonetheless, our algorithm is still able to (nearly) match it using a perfectly implementable procedure. 

Furthermore, Part 1 does not need a bounded domain; the same bound holds even with $\bar R=\infty$, making Algorithm~\ref{algorithm:1d} an appealing approach for general unconstrained settings as well. 

\item As for Part 2, we bound the movement cost alone over \emph{any} time interval, which is also technically nontrivial. Our surrogate loss $\tilde g_t$ (Line~\ref{line:1d_surrogate}) is due to an existing black-box reduction from unconstrained OLO to constrained OLO (see Appendix~\ref{subsection:constraint}). However, the proof of Part 2 requires a \emph{non-black-box} use of this procedure: we investigate how using the surrogate loss $\tilde g_t$ instead of the true loss $g_t$ changes the growth rate of $\wel_t$, an \emph{internal} quantity of the unconstrained OLO algorithm. To the best of our knowledge, this is the first analysis that takes this perspective. The revealed insights could be of separate interest. 
\end{enumerate}

\subsection{Extension to higher dimensions}

After the one-dimensional analysis, we present Algorithm~\ref{algorithm:higherd}, which extends Algorithm~\ref{algorithm:1d} to a higher dimensional ball $\ball^d(0,\bar R)$ via a polar decomposition. Intuitively, Algorithm~\ref{algorithm:higherd} learns the direction and magnitude separately: the former via standard OGD on the unit norm ball, and the latter via Algorithm~\ref{algorithm:1d}. Such an idea was first proposed by Cutkosky and Orabona \cite{cutkosky2018black}; here we further incorporate movement cost into its analysis. The performance guarantee has a similar flavor as Theorem~\ref{thm:1d}; for conciseness, we defer it to Appendix~\ref{subsection:alg2_analysis}.

\begin{algorithm*}[ht]
\caption{Extension of Algorithm~\ref{algorithm:1d} to $\ball^d(0,\bar R)$.\label{algorithm:higherd}}
\begin{algorithmic}[1]
\REQUIRE Hyperparameters $(\lambda,\eps,G)$ with $\lambda\geq 0$ and $\eps,G>0$; $g_1, g_2,\ldots\in\R^{d}$ with $\norms{g_t}\leq G$, $\forall t$.
\STATE Define $\A_{r}$ as Algorithm~\ref{algorithm:1d} on the domain $[0,\bar R]$, with hyperparameters $(\lambda,\lambda,\eps,G)$.
\STATE Define $\A_B$ as \emph{Online Gradient Descent} (OGD) on $\ball^{d}(0,1)$ with learning rate $\eta_t=1/(G\sqrt{t})$, initialized at the origin $0$.
\FOR{$t=1,2,\ldots$}
\STATE Obtain $y_t\in\R$ from $\A_{r}$ and $z_t\in\R^d$ from $\A_B$. Predict $x_t=y_tz_t\in\R^d$, observe $g_t\in\R^d$.
\STATE Return $\langle g_t,z_t\rangle$ and $g_t$ as the $t$-th loss gradient to $\A_{r}$ and $\A_B$, respectively.
\ENDFOR
\end{algorithmic}
\end{algorithm*}

\section{Strongly adaptive OCOM}\label{section:saocom}

Next, we introduce our second contribution - a novel strongly adaptive algorithm for \emph{Online Convex Optimization with Memory} (OCOM) \citep{anava2015online}. After introducing the problem setting, we present our approach step-by-step which builds on Algorithms~\ref{algorithm:1d} and \ref{algorithm:higherd}. 

\subsection{Problem setting of OCOM}\label{subsection:ocomsetting}

Consider a convex and compact domain $\V\subset\ball^d(0,R)$ with $R>0$. Without loss of generality, assume $\V$ contains the origin $0$.\footnote{By shifting the coordinate system, this can be achieved for any nonempty set $\V$. } In each round, a player makes a prediction $x_t\in\V$, observes a loss function $l_t:\V^{H+1}\rightarrow\R$ and suffers the loss $l_t(x_{t-H},\ldots,x_{t})$ that depends on the $H$-round prediction history. For all $t\leq 0$, $x_t=0$. 

We define an instantaneous loss function as $\tilde l_t(x)= l_t(x,\ldots,x)$. Two assumptions are imposed: (i) $l_t$ is $L$-Lipschitz with respect to each argument separately; (ii) $\tilde l_t(x)$ is convex and $\tilde G$-Lipschitz, with $0<\tilde G\leq L(H+1)$. 

For this OCOM problem, our goal is a strongly adaptive regret bound on the policy regret: for \emph{all} time intervals $\I=[a:b]\subset[1:T]$, 
\begin{equation}\label{eq:regret}
\sum_{t=a}^bl_t(x_{t-H:t})-\min_{x\in \V}\sum_{t=a}^b\tilde l_t(x)=O\rpar{\textrm{poly}(\log T)\cdot\sqrt{\abs{\I}}},
\end{equation}
where $O(\cdot)$ subsumes polynomial factors on the problem constants. In other words, on any time interval $\I\subset[1:T]$, the regret compared to the best \emph{$\I$-dependent} fixed prediction should be $\tilde O(\sqrt{\abs{\I}})$. 

\subsection{Preliminary: GC intervals}

First of all, we review an important concept. Similar to achieving strong adaptivity without memory \citep{daniely2015strongly,cutkosky2020parameter}, our OCOM algorithm has a hierarchical structure. It maintains a subroutine on each \emph{Geometric-Covering} (GC) interval, and the overall prediction combines the outputs from all the active subroutines. Such a structure benefits from a nice property \citep{daniely2015strongly}: an online learning algorithm is strongly adaptive if it has the desirable strongly adaptive guarantee \emph{on all the GC intervals}. Consequently for our objective (\ref{eq:regret}), we can only focus on achieving this bound on GC intervals instead of general intervals $\I\subset[1:T]$. 

\begin{figure}[ht]
\centering
\includegraphics[width=0.6\textwidth]{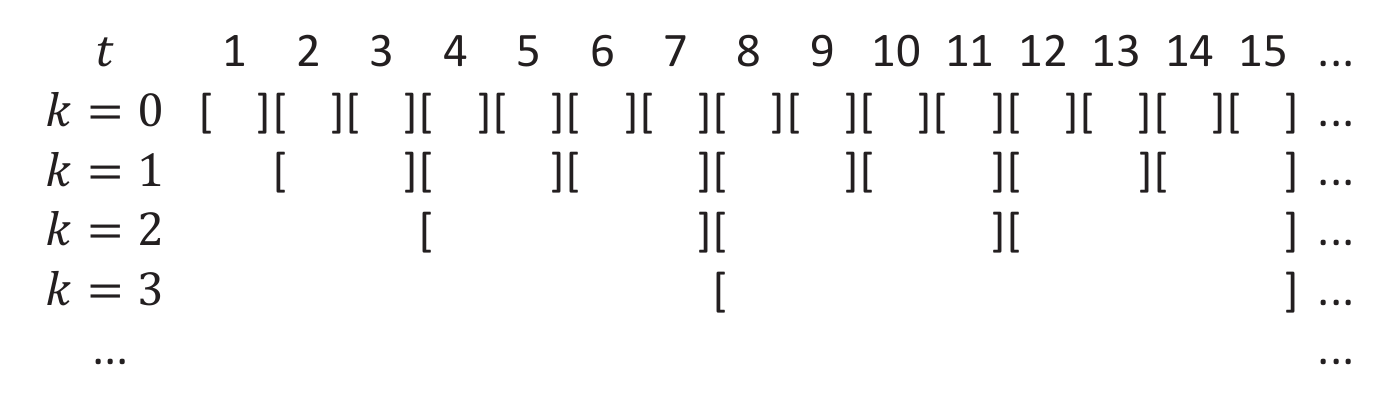}
\caption{Geometric-Covering intervals.\label{figure:GC}}
\end{figure}

The class of GC intervals is visualized in Figure~\ref{figure:GC}. Concretely, for all $k\in\N$ and $i\in\N_+$, a GC interval is defined as $\I^{k,i}=[2^ki:2^k(i+1)-1]$. If it contains $t$, then we say it is active in the $t$-th round. 

\subsection{Subroutine on GC intervals}

The next step is to construct the subroutine on each GC interval. It consists of two parts:

\begin{enumerate}
\item \emph{Subroutine-1d}, an OLO algorithm operating on the one-dimensional domain $[0,1]$. 
\item \emph{Subroutine-ball}, an OLO algorithm operating on the ball $\ball^d(0,R)$ that contains $\V$. 
\end{enumerate}

Intuitively, each Subroutine-1d produces the ``confidence'' on its corresponding Subroutine-ball. Then, the Subroutine-ball with higher confidence contributes a larger portion in the prediction of the meta-algorithm. Algorithms~\ref{algorithm:1d} and \ref{algorithm:higherd} constitute the basis of these two parts respectively, but we need one extra step (Algorithm~\ref{algorithm:adagradient}): Subroutine-1d is the version of Algorithm~\ref{algorithm:adagradient} with Line~1(a) and $g_t\in\R$, while Subroutine-ball is the version with Line~1(b) and $g_t\in\R^d$. Note that the time index $t$ in the pseudo-code represents the local clock counting from the start of the considered GC interval. That is, if we consider $\I^{k,i}$ starting from the $2^ki$-th round, then the index $t$ in Algorithm~\ref{algorithm:adagradient} represents the $(2^ki-1+t)$-th round globally. 

\begin{algorithm*}[ht]
\caption{Subroutine on GC intervals.\label{algorithm:adagradient}}
\begin{algorithmic}[1]
\REQUIRE Hyperparameters $(\lambda,\eps,G)$ with $\lambda\geq 0$ and $\eps,G>0$; gradients $g_1, g_2,\ldots$, with $\norms{g_t}\leq G$, $\forall t$.
\STATE (a) Subroutine-1d: Define $\A$ as Algorithm~1 with hyperparameters $(\lambda,0,\eps,\max\{\lambda,G\}+G)$, on the domain $[0,1]\subset\R$. 

(b) Subroutine-ball: Define $\A$ as Algorithm~2 with hyperparameters $(\lambda,\eps,\max\{\lambda,G\}+G)$, on the domain $\ball^d(0,R)$.
\STATE Initialize $i=1$ and an accumulator $Z_i=0$. Query the first output of $\A$ and assign it to $w_i$. 
\FOR{$t=1,2,\ldots$}
\STATE Predict $x_t\leftarrow w_i$, observe $g_t$, let $Z_i\leftarrow Z_i+g_t$. 
\IF{$\norms{Z_i}> \max\{\lambda,G\}$}\label{line:ada_threshold}
\STATE Send $Z_i$ to $\A$ as the $i$-th loss. Let $i\leftarrow i+1$. 
\STATE Set $Z_i=0$. Query the $i$-th output of $\A$ and assign it to $w_i$. 
\ENDIF
\ENDFOR
\end{algorithmic}
\end{algorithm*}

Algorithm~\ref{algorithm:adagradient} serves two purposes: (i) improving the dependence on hyperparameters $G$ and $\lambda$ (ultimately, problem constants of OCOM); and (ii) achieving adaptivity to the observed gradients, which leads to better practical performance. Its key mechanism is to \emph{adaptively ``slow down''} the base algorithm $\A$. To this end, an accumulator $Z_i$ tracks the sum of the received loss gradients. The base algorithm $\A$ is only queried when $Z_i$ exceeds a threshold $\max\{\lambda,G\}$. Using this procedure, we essentially replace the time horizon $T$ in the performance guarantee of $\A$ by an adaptive quantity $\sum_{t=1}^T\norms{g_t}/\max\{\lambda,G\}$. 

\subsection{Meta-algorithm}

Given the two-part subroutine, we now introduce our OCOM meta-algorithm. Compared to online learning without memory \citep{cutkosky2020parameter}, our technical improvement is the incorporation of movement cost which is a nontrivial task. The complete pseudo-code is deferred to Appendix~\ref{subsection:metafull}, and an abridged version (Algorithm~\ref{algorithm:metaabridged}) is provided here. Specifically, Algorithm~\ref{algorithm:metaabridged} simplifies a complicated projection scheme by allowing improper predictions ($x_t\notin \V$).

\begin{algorithm*}[ht]
\caption{The OCOM meta-algorithm. (Abridged from Algorithm~\ref{algorithm:meta} in Appendix~\ref{subsection:metafull})\label{algorithm:metaabridged}}
\begin{algorithmic}[1]
\REQUIRE $T\geq 1$; a hyperparameter $\eps_0>0$.
\FOR{$t=1,\ldots,T$}
\STATE Find the $(k,i)$ index pair for all the GC intervals that start in the $t$-th round. For each pair, initialize $\A^{k}_B$ as a copy of Subroutine-ball and $\A^{k}_{1d}$ as a copy of Subroutine-1d, with \emph{some} hyperparameters that depend on $k$, $\eps_0$ and problem constants. If $\A^k_B$ and $\A^k_{1d}$ already exist in the memory, overwrite them. 
\STATE Define $K_t=\lceil \log_2 (t+1)\rceil-1$; $x^{(K_t+1)}_t=0\in\R^d$. 

\FOR{$k=K_t,\ldots,0$}
\STATE Query a prediction from $\A^k_B$ and assign it to $w^{(k)}_t$; query a prediction from $\A^k_{1d}$ and assign it to $z^{(k)}_t$. 
\STATE \label{line:metacombination}Let $x^{(k)}_t=(1-z^{(k)}_t)x^{(k+1)}_t+w^{(k)}_t$.
\ENDFOR
\STATE Predict $x_t=x^{(0)}_t$, suffer $l_t(x_{t-H:t})$, receive $l_t$, obtain a subgradient $g_t\in\partial \tilde l_t(x_{t})$. 
\FOR{$k=0,\ldots,K_t$}
\STATE Return $g_t$ to $\A^k_B$ and $-\langle g_t,x^{(k+1)}_t\rangle$ to $\A^k_{1d}$ as the loss gradients respectively.  
\ENDFOR
\ENDFOR
\end{algorithmic}
\end{algorithm*}

In each round, Algorithm~\ref{algorithm:metaabridged} combines the subroutines by recursively running Line~\ref{line:metacombination}. Such a procedure is different from the well-known \emph{boosting} strategy \citep{freund1997decision,beygelzimer2015online} applied in \citep{daniely2019competitive}, as the updated temporary prediction $x^{(k)}_t$ is \emph{not a convex combination} of the old temporary prediction $x^{(k+1)}_t$ and the output $w^{(k)}_t$ from Subroutine-ball. By plugging the comparator-adaptive property of the subroutines into Line~\ref{line:metacombination}, Algorithm~\ref{algorithm:metaabridged} achieves an important property: for all $k$, $x^{(k)}_t$ matches the performance of $w^{(k)}_t$ on time intervals of length $2^k$ while achieving the performance of $x^{(k+1)}_t$ on longer time intervals.\footnote{Compared to \citep{daniely2019competitive}, this intuitively generalizes the ``easy-to-combine'' idea from expert problems to OLO.} As the result, the final prediction $x^{(0)}_t$ matches the performance of \emph{any} subroutine on its corresponding GC interval. 

To recap, we demonstrate the structure of our OCOM algorithm in Figure~\ref{figure:strategy}. Collecting all the pieces, we state the performance guarantee in Theorem~\ref{thm:ocom}.

\begin{figure}[ht]
    \centering
    \includegraphics[width=0.5\textwidth]{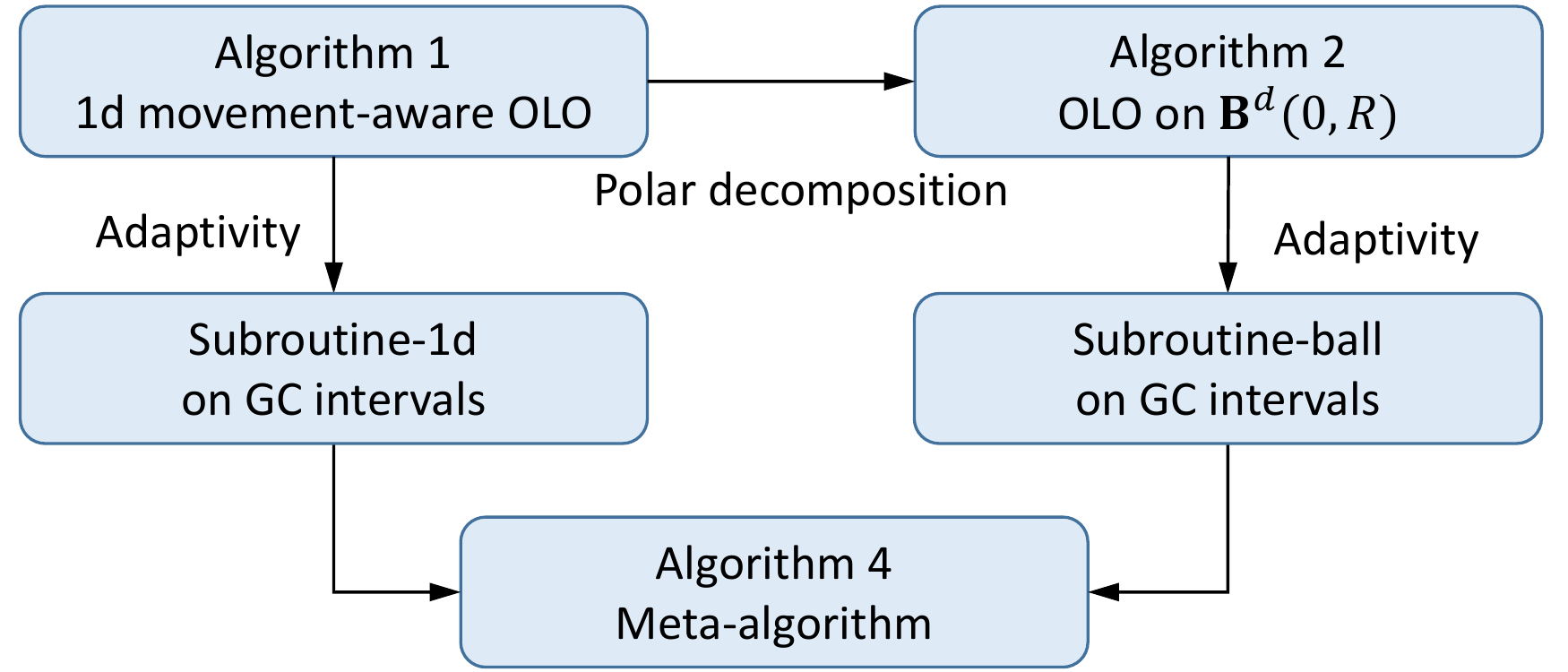}
    \caption{An overview of our OCOM strategy. \label{figure:strategy}}
\end{figure}

\begin{restatable}{theorem}{ocom}\label{thm:ocom}
Consider running our OCOM algorithm (the complete version, Algorithm~\ref{algorithm:meta}) for $T$ rounds. If $\eps_0=\tilde GR/(T+1)$, then on any time interval $\I=[a:b]\subset[1:T]$,
\begin{equation*}
\sum_{t=a}^bl_t(x_{t-H:t})-\min_{x\in \V}\sum_{t=a}^b\tilde l_t(x)=O(RLH^3\log\abs{\I})+\tilde O\rpar{RLH^2+RH\sqrt{L\sum_{t=a}^b\norm{g_t}}},
\end{equation*}
where $g_t\in\partial \tilde l_t(x_{t})$, $O(\cdot)$ subsumes absolute constants, and $\tilde O(\cdot)$ subsumes poly-logarithmic factors on problem constants and $T$. 
\end{restatable}

Notice that the obtained bound is not only strongly adaptive according to Equation~(\ref{eq:regret}), but also adaptive to the observed gradients. In easy environments, it would be a lot better than $\tilde O(\sqrt{|\I|})$.

\begin{remark}\label{remark:adding}
Strongly adaptive regret is not the only performance metric that compares to dynamic comparators; alternatives include dynamic regret and competitive ratio (see Appendix~\ref{section:additionalreview} for an overview). If we have an algorithm $\A$ with such (alternative) guarantees on $[1:T]$ and a slow-moving property similar to Part~2 of Theorem~\ref{thm:1d}, then we can assign the prediction of $\A$ to $x^{(K_t+1)}_t$. The resulting algorithm would not only remain strongly adaptive, but also essentially achieve the dynamic regret or competitive ratio guarantee of $\A$ on $[1:T]$. 
\end{remark}

\section{Adversarial tracking control}\label{section:tracking}

Finally we present our third contribution: a reduction from adversarial tracking control to strongly adaptive OCOM. Let us start with the problem setting. 

\subsection{Problem setting of adversarial tracking}

We consider a time-varying linear system
\begin{equation*}
x_{t+1}=A_tx_t+B_tu_{t}+w_t.
\end{equation*}
Matrices $A_t\in\R^{d_x\times d_x}$ and $B_t\in\R^{d_x\times d_u}$ are known. For all $t\leq 0$, $x_t=0$, $u_t=0$; for all $t<0$, $w_t=0$. 

The system has the following interaction protocol. At the beginning of the $t$-th round, after observing $x_t$, the controller commits to an action $u_t$. Then, an adversary selects the disturbance $w_t$, a reference state-action pair $(x^*_t, u^*_t)$ and a loss function $l_t$, possibly depending on past controller actions $u_1,\ldots,u_t$. $(x^*_t, u^*_t)$ and $l_t$ together induce a tracking loss function $l^*_t(x,u|x^*_t,u^*_t)\defeq l_t(x-x^*_t,u-u^*_t)$ for all $(x,u)\in\R^{d_x}\times\R^{d_u}$, which is revealed to the controller and incurs a loss $l^*_t(x_t,u_{t})$. After that, the state evolves to $x_{t+1}$ following the system equation. Intuitively, $l_t$ represents the shape of the loss function and $(x^*_t,u^*_t)$ is the location parameter; an example is the quadratic control problem with $l_t(x,u)=\norms{x}^2+\norms{u}^2$.

Our goal is a strongly adaptive tracking guarantee with the following shape: on any time interval $\I$ contained in the time horizon $[1:T]$, for all action sequences $u^C_{1:T}$ that are fixed on $\I$, 
\begin{equation*}
\sum_{t\in\I}l^*_t\rpar{x_t,u_t}\Big|_{\textrm{our algorithm}}-\sum_{t\in\I}l^*_t\rpar{x^C_t,u^C_{t}}\Big|_{\textrm{induced by }u^C_{1:T}}
=\tilde O\rpar{\sqrt{\abs{\I}}}. 
\end{equation*}
Such a guarantee subsumes the conventional static regret bound as one can choose $\I=[1:T]$. Moreover, the key benefit is that on any time interval $\I$, the optimal comparator is optimized for $\I$ instead of the entire time horizon $[1:T]$. From this perspective, we aim at a considerably stronger goal than existing works \citep{abbasi2014tracking,foster2020logarithmic}. 

For our setting, we impose the following assumptions. $\kappa$, $\gamma$, $U$ and $L^*$ are assumed to be known. 

\begin{assumption}[On the system]\label{assumption:system}
There exist $\kappa\geq 1$ and $U,W,\gamma>0$ such that for all $t$, $\norms{B_t}\leq \kappa$, $\norms{u_t}\leq U$, $\norms{w_t}\leq W$ and $\norms{A_t}\leq 1-\gamma$. 
\end{assumption}

\begin{assumption}[On the losses]\label{assumption:loss}
For all $t$, $l^*_t$ is convex. In addition, $l^*_t(x,u)$ is $L^*$-Lipschitz with respect to each argument separately, on the set $\{(x,u);\norms{x}\leq \gamma^{-1}(\kappa U+W),\norms{u}\leq U\}$.
\end{assumption}

\begin{remark}
The assumption $\norms{A_t}\leq 1-\gamma$ may seem restrictive as many real world systems are not open-loop stable. However, such an assumption allows a simplified exposition without excessively altering the essence of the problem. For general (open-loop unstable) systems, we can assume oracle stabilizing controllers (matrices) $K_{1:\infty}$ such that $\norms{\prod_{t=s}^{s+k}(A_t+B_tK_t)}\leq \textrm{const}\cdot (1-\gamma)^k$ for all $s$ and $k$, and the multiplying constant can be larger than 1. Such an extension is somewhat standard in the analysis of linear time-varying systems \citep[Appendix A.2]{minasyan2021online}. Given $K_{1:\infty}$, we can replace our action $u_t$ with $K_tx_t + u_t$ so that a similar analysis follows. 
\end{remark}

\subsection{Difference with nonstochastic regulation}\label{subsection:difference}

Before proceeding, we (re)-emphasize the difference between our work and a series of nonstochastic regulation controllers (most notably, \cite{agarwal2019online}). For a clear comparison, consider time-invariant dynamics ($A_t=A$, $B_t=B$) and state-tracking ($l^*_t$ only depends on $x_t$). The procedure of \citep{agarwal2019online} can be summarized as follows. 
\begin{enumerate}
\item Before observing any data, the controller computes a stabilizing feedback matrix $K$ based on $(A,B)$. 
\item The actions are determined by a specific parameterization called \emph{Disturbance-Action Controller} (DAC):
\begin{equation*}
u_t=-Kx_t+\sum_{i=1}^HM^{[i]}_tw_{t-i},
\end{equation*}
where $H$ is a constant, $w_{t-i}$ is a past disturbance, and $M^{[1]}_t,\ldots,M^{[H]}_t$ are parameter matrices updated via online gradient descent. The idea is to stabilize the system by $-Kx_t$, and adapt to the disturbances by applying their linear combinations. 
\item It can be shown that the DAC class approximates a class of stabilizing linear controllers, therefore the regret guarantee can be stated with respect to the latter (as the comparator class).
\end{enumerate}

Such an approach works well for the regulation problem, but in tracking it has a substantial limitation. Consider a simple example: what if the system is disturbance-free? In that case, the controller reduces to a static linear feedback, and the gain matrix is determined without seeing any data. In other words, \emph{nothing is learned}. The state sequence would converge to the origin, therefore the tracking loss can be always high as long as the target state $x^*_t$ is far away from the origin. 

If $x^*_t$ is known a priori, there is a standard remedy \citep[Section~2]{yu2020power}: define a shifted state $\tilde x_t$ as the tracking error $x_t-x^*_t$ and apply the DAC on the shifted system to determine $u_t$. However, this is not applicable in our \emph{adversarial} tracking problem, as $x^*_t$ is not revealed before $u_t$ is committed. (Even worse, $x^*_t$ can \emph{adapt} to $u_t$ and sabotage any controller that selects $u_t$ based on an \emph{assumed or predicted} $x^*_t$.) In this paper, instead of fixing this framework, we propose an approach with a different principle. 

Finally, our approach can be complementary to \citep{agarwal2019online} in two ways: (i) Our strongly adaptive OCOM algorithm (Algorithm~\ref{algorithm:meta}) can be combined with DAC to improve a recent regulation controller for time-varying systems \citep{gradu2020adaptive}. On all $\I\subset[1:T]$, the regret of regulation is improved from $\tilde O(\sqrt{T})$ to $\tilde O(\sqrt{|\I|})$. (ii) Our adversarial tracking controller could be \emph{added} to a regulation controller to achieve both goals simultaneously. 

\subsection{Reduction to strongly adaptive OCOM}

Now we sketch the key idea of our reduction, which is to truncate history and directly optimize on the action space. To the best of our knowledge, our approach is the first that uses the ``tracking'' property of online learning algorithms in tracking control.

To begin with, note that old actions have diminishing effect on future states due to the stability of the system. Therefore, given a large enough memory constant $H$, the actual state $x_t$ can be approximated by an ideal state
\begin{equation*}
y_{t}(u_{t-H:t-1})=\sum_{i=t-H}^{t-1}\rpar{\prod_{j=i+1}^{t-1}A_j}\rpar{B_iu_{i}+w_i},
\end{equation*}
which is the value $x_t$ would take if $x_{t-H}=0$. Using $y_t$ to replace $x_t$, the actual loss $l^*_t(x_t,u_t)$ can also be approximated by an ideal loss
\begin{equation}
f_t(u_{t-H:t})=l^*_t (y_{t}(u_{t-H:t-1}),u_{t}).\label{eq:ideal_loss}
\end{equation}
Compared to $l^*_t(x_t,u_t)$, the ideal loss $f_t(u_{t-H:t})$ only depends on a finite length action history $u_{t-H:t}$ instead of all the past actions. Therefore, one may use a strongly adaptive OCOM algorithm to \emph{dynamically track} the optimal input that minimizes $f_t$, which should be close to the optimal action that minimizes $l^*_t$. Formally, we present the pseudo-code in Algorithm~\ref{algorithm:controller}. 

\begin{algorithm*}[ht]
\caption{A reduction from adversarial tracking control to strongly adaptive OCOM. \label{algorithm:controller}}
\begin{algorithmic}[1]
\REQUIRE Time horizon $T>1$ and a strongly adaptive OCOM algorithm.
\STATE Initialize the strongly adaptive OCOM algorithm as $\A$, with time horizon $T$. Problem constants for OCOM are defined using those for adversarial tracking: $\V\leftarrow\ball^{d_u}(0,U)$, $R\leftarrow U$, $H\leftarrow\max\{\left\lceil-\log T/\log(1-\gamma)\right\rceil,2\gamma^{-1}\}$, $L\leftarrow\kappa L^*$ and $\tilde G\leftarrow 2\kappa\gamma^{-1}L^*$.

\FOR{$t=1,\ldots,T$}
\STATE Observe $x_t$ and compute $w_{t-1}=x_t-A_{t-1}x_{t-1}-B_{t-1}u_{t-1}$. 
\STATE Obtain $u_t$ from $\A$, apply it, observe the loss function $l^*_t$ and suffer $l^*_t(x_t,u_{t})$.
\STATE Compute the ideal loss function $f_t$ from (\ref{eq:ideal_loss}), and return it to $\A$. 
\ENDFOR
\end{algorithmic}
\end{algorithm*}

Technically, the main benefit of our approach is that \emph{on any time interval} it guarantees a regret bound against an \emph{interval-dependent} comparator class. 

\begin{definition}[Interval-dependent comparator class]\label{definition:comparator}
Given any time interval $\I=[a:b]\subset[H+1:T]$, the comparator class $\mathcal{C}_\I$ is defined as the set of action sequences $u^C_{1:T}$ such that for all $t\in[a-H:b]$, $u^C_{t}=u^C_{b}$. 
\end{definition}

In other words, the comparator class $\mathcal{C}_\I$ contains all action sequences that are essentially fixed on the investigated time-interval $\I$, but arbitrarily varying elsewhere. 

The performance guarantee of Algorithm~\ref{algorithm:controller} is stated in Theorem~\ref{thm:tracking}. We write $x_t(u^A_{1:t-1})$ and $u^A_{t}$ as the state-action pair induced by Algorithm~\ref{algorithm:controller}. Similarly, $x_t(u^C_{1:t-1})$ is the state induced by a comparator. (Superscripts $A$ and $C$ represent ``Adversarial tracking'' and ``Comparator''.)

\begin{restatable}{theorem}{tracking}\label{thm:tracking}
Given any strongly adaptive OCOM algorithm satisfying Equation~(\ref{eq:regret}), for all $\I=[a:b]\subset[H+1:T]$, Algorithm~\ref{algorithm:controller} guarantees
\begin{equation*}
\sum_{t=a}^bl^*_t\rpar{x_t(u^A_{1:t-1}),u^A_{t}}-\min_{u^C_{1:T}\in\mathcal{C}_\I}\sum_{t=a}^bl^*_t\rpar{x_t(u^C_{1:t-1}),u^C_{t}}=\tilde O\rpar{\sqrt{\abs{\I}}},
\end{equation*}
where $\tilde O(\cdot)$ subsumes problem constants and $\textrm{poly}(\log T)$. 
\end{restatable}

Theorem~\ref{thm:tracking} can be interpreted as: on \emph{any} time interval, the cumulative tracking loss approaches that of the best \emph{interval-dependent} action. If Algorithm~\ref{algorithm:controller} uses our strongly adaptive OCOM algorithm, then the obtained bound further adapts to the observed gradients. Notably, Theorem~\ref{thm:tracking} improves existing results on adversarial tracking (e.g., \cite{abbasi2014tracking}), especially when the reference trajectory has a large range of movement. To make it clear, consider tracking a piecewise constant reference trajectory. In that case, existing regret bounds are only established \emph{on the entire time horizon} $[1:T]$, and the comparator class only contains static linear controllers which are weak baselines for tracking this moving target. In comparison, Theorem~\ref{thm:tracking} induces a regret bound \emph{on any time interval}, including $[1:T]$ and its much shorter sub-intervals. The regret bound on $[1:T]$ suffers from the same problem (the comparator class is weak). However, on all time intervals where the target is fixed, the interval-dependent comparator class is strong, and the regret bound makes much more sense. 

To make the above discussion even more concrete, we construct the following example. Here we can further derive a \emph{non-comparative} tracking error bound. 

\begin{example}\label{example}
Consider a time interval $\I=[a:b]\subset[H+1:T]$. For all $t\in\I$, we assume

\begin{enumerate}
\item $(I-A_t)^{-1}B_t=B_\I$ for some time-invariant matrix $B_\I$, which includes static $A_t$ and $B_t$ as a special case. Note that $I-A_t$ is invertible since $\norms{A_t}<1$. 
\item The target $x^*_t=x^*_\I$ for some time-invariant $x^*_\I\in\{B_\I u;\norms{u}\leq U\}$, and $l^*_t(x,u)=\norms{x-x^*_t}$.
\end{enumerate}
\end{example}

\begin{restatable}{corollary}{special}\label{thm:special}
Consider running Algorithm~\ref{algorithm:controller} on an adversarial tracking problem that satisfies Example~\ref{example} on a time interval $\I$. For all $t\in\I=[a:b]$,
\begin{equation*}
\frac{1}{t-a+1}\sum_{i=a}^{t}\norm{x_{i}(u^A_{1:i-1})-x^*_\I}\leq \gamma^{-1}W+\tilde O\rpar{(t-a+1)^{-1/2}}.
\end{equation*}
\end{restatable}
Corollary~\ref{thm:special} directly characterize the tracking error \emph{without any comparator} which is the performance metric of interest in most classical control-theoretic literature. Further applying Jensen's inequality, the time-average of the states \emph{on any time interval satisfying Example~\ref{example}} converges to a norm ball around the target. Notably, Algorithm~\ref{algorithm:controller} does not need to know any favorable problem structure a priori: when running on a long time horizon $[1:T]$, it can \emph{automatically} exploit the inactivity of the target (if any) on shorter sub-intervals. This is fundamentally different from the classical idea in tracking control where a generative model of the target is hard-coded into the controller.

\paragraph{Experiments} For conciseness, we defer experimental results to Appendix~\ref{section:experiments}. All three components of our contribution (cf. Section~\ref{sec:cont}) are tested numerically there. 

\section{Conclusion}\label{section:conclusion}

We consider tracking adversarial targets in a general linear system. Three techniques are developed in a hierarchical manner, and their combination is a strongly adaptive tracking controller that significantly improves existing results. Our approach could facilitate the application of online learning ideas to a wider range of linear control problems. 

\section*{Acknowledgements}

We thank the anonymous reviewers for their constructive feedback. This research was partially supported by the NSF under grants IIS-1914792, DMS-1664644, and CNS-1645681, by the ONR under grants N00014-19-1-2571 and N00014-21-1-2844, by the DOE under grants DE-AR-0001282 and DE-EE0009696, by the NIH under grants R01 GM135930 and UL54 TR004130, and by Boston University. 

\bibliography{SA_tracking}

\newpage
\section*{Appendix}
\appendix

\paragraph{Organization} Appendix~\ref{section:additionalreview} contains additional discussion of existing works omitted in the main paper. Appendix~\ref{section:prelimocom}, \ref{section:proofocom} and \ref{section:appendixtracking} contain details of our three technical contributions. Finally, empirical results are provided in Appendix~\ref{section:experiments}. 

\section{Additional discussion of existing works}\label{section:additionalreview}

As reviewed in Section~\ref{subsection:related}, our approach to adversarial tracking control relies on its connection to tracking nonstationary comparators in online learning. There are multiple performance metrics to quantify the latter goal. In this paper we choose strong adaptivity. Other than this, one may use dynamic regret or competitive ratio. We briefly review them as follows. 

\paragraph{Dynamic regret} In the context of OLO, dynamic regret \citep{jadbabaie2015online,zhang2016improved,zhang2018dynamic,zinkevich2003online} is the regret that directly compares to a nonstationary prediction sequence $u_{1:T}$ on the entire time horizon $[1:T]$. Such bounds in general depend on the cumulative variation of the comparator over $[1:T]$ (the \emph{path length}), and sometimes also the variation of the loss function. The path length can be defined in multiple ways; when defined as $P=\sum_{t=1}^T\norms{u_t-u_{t+1}}$, the optimal dynamic regret bound is $O(\sqrt{PT})$. The idea is that if the comparator is static ($P=0$), then the dynamic regret reduces to the static regret; with a large path length ($P=O(T)$), the dynamic regret becomes vacuous. 

Existing works \citep{cutkosky2020parameter,zhang2018dynamic,zhang2020minimizing} have investigated the relation between dynamic regret and strongly adaptive regret in OLO. It has been suggested that the former could be a slightly weaker notion than the latter, as dynamic regret is derived from strongly adaptive regret in \citep{cutkosky2020parameter,zhang2018dynamic} while no result in the opposite direction has been given (to the best of our knowledge). Generalizing it from OLO to online learning with memory, \citep{zhao2021non} provided a dynamic regret analysis of OCOM and nonstochastic control. It is possible that such a result could be achieved via strongly adaptive approaches (such as our Algorithm~\ref{algorithm:meta} or \citep{daniely2019competitive}), although detailed analysis is beyond the scope of this paper. 

\paragraph{Competitive ratio} Competitive ratio is a largely different performance metric in online learning compared to the regret framework. For online control, the relevant setting for competitive ratio analysis is \emph{Smoothed Online Convex Optimization} (SOCO) \citep{chen2015online,chen2018smoothed,goel2019beyond}. It has two key differences with OCO:

\begin{enumerate}
\item The loss function $l_t$ is revealed to the player before his prediction $x_t$ is made. 
\item In addition to the loss $l_t(x_t)$, the player further suffers a movement cost $c(x_{t-1},x_t)$ in each round, where $c(\cdot,\cdot)$ is some penalty function for large movements.
\end{enumerate}

From this setting, SOCO and the accompanying competitive ratio analysis are particularly suitable for \emph{online control with predictions}, as we review later. The general form of a competitive ratio guarantee is
\begin{equation*}
\textrm{Cost on }[1:T]\leq \alpha\cdot\textrm{Comparator cost on }[1:T]+\beta,
\end{equation*}
where $\alpha$ and $\beta$ are constants, and $\alpha$ is defined as the competitive ratio. The comparator class contains all the prediction sequences, therefore intrinsically the benchmarks are nonstationary. Compared to a dynamic regret bound, the competitive ratio analysis (i) does not depend on the path length; and (ii) characterizes the cost of the algorithm in a multiplicative manner with respect to the comparator cost, instead of an additive one. 

\paragraph{Advantages of strong adaptivity in adversarial tracking} Following the above discussion, we next discuss the advantages of strong adaptivity over the other two performance metrics in adversarial tracking. First, compared to the other two, a strongly adaptive regret guarantee is \emph{local}: on \emph{all} sub-intervals in $[1:T]$ we have a regret bound that compares to the \emph{interval-dependent optimal} comparator, and the bound depends on the sub-interval length instead of $T$. In contrast, the other two performance metrics are stated for the entire time horizon. Second, both \citep{daniely2019competitive} and our Algorithm~\ref{algorithm:meta} can incorporate algorithms with dynamic regret or competitive ratio guarantees (for our approach, see Remark~\ref{remark:adding}). The resulting algorithm guarantees the best of both worlds. Third, as we discussed above, dynamic regret could be conceptually weaker than strongly adaptive regret, and competitive ratio analysis requires a different setting (predictions). 

\paragraph{Linear control with predictions} Inspired by the classical idea of model-predictive control, a series of recent works \citep{li2019online,shi2020online,yu2020power} considered learning-based approaches for linear control with predictions. Specifically for tracking, predictions of the reference trajectory are typically required, which is less general than our fully adversarial setting. Furthermore, the loss functions are strongly convex, resulting in less modeling power (e.g., for modeling target regions, where the minimizer of the loss function is not unique). 

Notably, for online learning, \citep{shi2020online} presented an interesting generalization of SOCO called \emph{OCO with structured memory}:

\begin{enumerate}
\item The one-step memory in SOCO is generalized to longer memory similar to OCOM. 
\item Accurate prediction of the loss function is not required. In each round, the adversary first reveals a function $h_t$. After the player picks $x_t\in\V$, the adversary further reveals $v_t\in\V$ and induces a loss $h_t(x_t-v_t)$. In other words, the player only needs to accurately predict the \emph{shape} of the loss function; the actual incurred loss is further shifted by an adversarial component. 
\end{enumerate}

Based on this new setting, \citep{shi2020online} provided a competitive ratio analysis of the regulation control problem. It is possible that such analysis could be extended to track fully adversarial targets, but still, (i)~accurate predictions of strongly convex loss functions are required; (ii) the resulting algorithm could be combined with our approach, as discussed in Remark~\ref{remark:adding}. 

\section{Details on comparator-adaptive OLO with movement cost}\label{section:prelimocom}

This section presents details on our first contribution, movement-aware OLO. We rely heavily on a duality between unconstrained OLO and the coin-betting game, which is summarized in Appendix~\ref{subsection:coinbettingOLO}. After that, Appendix~\ref{subsection:constraint} introduces an existing reduction from constrained OLO to unconstrained OLO, adopted in our Algorithm~\ref{algorithm:1d} and Algorithm~\ref{algorithm:meta}. The last two subsections provide detailed analysis of Algorithm~\ref{algorithm:1d} and \ref{algorithm:higherd}, respectively. 

\subsection{An overview of coin-betting and unconstrained OLO}\label{subsection:coinbettingOLO}

We start from the definition of the coin-betting game: A player has initial wealth $\wel_0=\eps$. In each round, he picks a betting fraction $\beta_t\in[-1,1]$ and bets an amount $x_t=\beta_t\wel_{t-1}$. Then, an adversarial coin tossing $c_t\in[-1,1]$ is revealed, and the wealth of the player is changed by $c_tx_t$. In other words, the player wins money if $c_tx_t>0$, and loses money if $c_tx_t<0$. The goal of the player is to design betting fractions $\beta_1, \beta_2,\ldots$ such that his wealth in the $T$-th round is maximized. We are particularly interested in \emph{parameter-free} betting strategies, for example the Krichevsky-Trofimov (KT) bettor: $\beta_t=\sum_{i=1}^{t-1}c_i/t$. Notice that it does not rely on any hyperparameters. 

We can associate the coin-betting game to one-dimensional unconstrained OLO \cite[Theorem 9.6]{orabona2019modern}. For an OLO problem with loss gradient $g_t\in\R$, one can maintain a coin-betting algorithm with $c_t=-g_t$, and predict \emph{exactly} its betting amount $x_t$ in OLO. The wealth lower bound for coin-betting is equivalent to a regret upper bound for OLO. Induced by a parameter-free bettor (such as KT), the resulting OLO algorithm can enjoy the following benefits: (i) There are no hyperparameters to tune. (ii) The regret bound has optimal dependence on the comparator norm. (iii) When the comparator is the null comparator $0$, the regret upper bound reduces to a constant. In other words, the \emph{cumulative cost} is at most a constant. Properties (ii) and (iii) are often called \emph{comparator-adaptivity}. 

To further appreciate the power of such approach, let us compare the resulting 1d unconstrained OLO algorithm to standard Online Gradient Descent (OGD). 
\begin{enumerate}
\item Analytically, with an unconstrained domain, $L$-Lipschitz losses and learning rate $\eta$, OGD has the regret bound
\begin{equation*}
\sum_{t=1}^T\inner{g_t}{x_t-u}\leq \frac{\abs{u-x_1}^2}{2\eta}+\frac{\eta L^2T}{2},~\forall u\in\R.
\end{equation*}
Since the optimal comparator $u$ is unknown beforehand, one has to choose $\eta=O(1/(L\sqrt{T}))$, leading to the sub-optimal regret bound $O(|u|^2L\sqrt{T})$. In comparison, KT-based OLO algorithm guarantees a regret bound $\tilde O(|u|L\sqrt{T})$, matching the lower bound up to logarithmic factors. 

\item Intuitively, assume the loss gradients are
\begin{equation*}
g_t=\begin{cases}
-1, &\textrm{if~}x_t\leq x^*,\\
1, &\textrm{otherwise,}
\end{cases}
\end{equation*}
where $x^*$ is a fixed ``target''. With a pre-determined learning rate $\eta$, OGD approaches the target linearly. However, since $x^*$ is unknown, there are always cases where $x^*$ is far enough from the starting point of OGD, making OGD very slow to find $x^*$. In comparison, KT-based OLO algorithm approaches $x^*$ with \emph{exponentially increasing speed} \citep[Figure~9.1]{orabona2019modern}, finding $x^*$ a lot faster.
\end{enumerate}

As a final note, in this paper we aim to bound the sum of regret and movement in coin-betting-based OLO algorithms. Although the exponentially increasing per-step movement is good for regret minimization, it poses a significant challenge for the control of movement cost. Using a movement-restricted bettor (Algorithm~\ref{algorithm:1d}), we achieve this in Theorem~\ref{thm:1d}. 

\subsection{Adding constraints in OLO}\label{subsection:constraint}

Our approach requires a reduction from constrained OLO to unconstrained OLO, proposed in \citep{cutkosky2020parameter}. The pseudo-code is Algorithm~\ref{algorithm:constraint}. We use this reduction in both the movement-aware OLO algorithm (Algorithm~\ref{algorithm:1d}) and the OCOM meta-algorithm (Algorithm~\ref{algorithm:meta}). 

\begin{algorithm*}[ht]
\caption{Adding constraints in OLO.\label{algorithm:constraint}}
\begin{algorithmic}[1]
\REQUIRE An OLO algorithm $\A$ and an arbitrary nonempty, closed and convex domain $\V$. 
\FOR{$t=1,\ldots,T$}
\STATE Obtain the prediction $\tilde x_t$ from $\A$.
\STATE Predict $x_t=\Pi_{\V}(\tilde x_t)$ and receive the loss subgradient $g_t$. 
\STATE \label{line:surrogate_h}Define a surrogate loss function $h_t$ as
\begin{equation*}
h_t(x)=\begin{cases}
\langle g_t,x\rangle, &\textrm{if~}\langle g_t,\tilde x_t\rangle\geq \langle g_t,x_t\rangle,\\
\langle g_t,x\rangle+\langle g_t,x_t-\tilde x_t\rangle\frac{\norm{x-\Pi_{\V}(x)}}{\norm{x_t-\tilde x_t}}, &\textrm{otherwise}.
\end{cases}
\end{equation*}
\STATE Obtain a subgradient $\tilde g_t\in\partial h_t(\tilde x_t)$ and return it to $\A$ as the $t$-th loss subgradient.  
\ENDFOR
\end{algorithmic}
\end{algorithm*}

\begin{lemma}[\cite{cutkosky2020parameter}, Theorem 2]\label{lemma:constraint}
Algorithm~\ref{algorithm:constraint} has the following properties for all $t$: (1) $h_t$ is a convex function on $\V$. (2) $\norms{\tilde g_t}\leq\norms{g_t}$. (3) For all $u\in \V$, $\inner{g_t}{x_t-u}\leq \inner{\tilde g_t}{\tilde x_t-u}$.
\end{lemma}

\subsection{Analysis of Algorithm~\ref{algorithm:1d}}

This subsection provides analysis of Algorithm~\ref{algorithm:1d}, which is organized as follows. We first show the well-posedness of Line~\ref{line:wel} (the existence and uniqueness of solution). After that, we present a few useful lemmas before proving the performance guarantee of Algorithm~\ref{algorithm:1d} (Theorem~\ref{thm:1d}). 

\begin{lemma}\label{lemma:unique_solution}
For all $t\geq 1$, Equation (\ref{eq:wel_update}) has a unique solution and the solution is positive. 
\end{lemma}

\begin{proof}[Proof of Lemma~\ref{lemma:unique_solution}]
For clarity, Equation (\ref{eq:wel_update}) is copied here. 
\begin{equation*}
\wel_t=(1-\tilde g_t\beta_t-\gamma\beta_t/\sqrt{t})\wel_{t-1}-\lambda|\beta_t\wel_{t-1}-\beta_{t+1}\wel_t|.
\end{equation*}
By definition, $|\lambda\beta_{t+1}|\leq 1/2$. The RHS of (\ref{eq:wel_update}) is $1/2$-Lipschitz with respect to $\wel_t$, and the LHS is $\wel_t$ itself. Therefore, a solution exists and is unique. 

To prove $\wel_t> 0$, we use induction. $\wel_0=\eps>0$. Suppose $\wel_{t-1}> 0$, then
\begin{equation*}
\wel_t\geq(1-\tilde g_t\beta_t-\gamma\beta_t/\sqrt{t})\wel_{t-1}-\lambda\beta_t\wel_{t-1}-\lambda\beta_{t+1}\abs{\wel_t}.
\end{equation*}
Let $z=\lambda\beta_{t+1}\sgn(\wel_t)$. Note that $|z|\leq 1/2$ and $|\tilde g_t+\gamma/\sqrt{t}+\lambda|\beta_{t}\leq 1/2$. Therefore, 
\begin{equation*}
\wel_t\geq\frac{1-\tilde g_t\beta_t-\gamma\beta_t/\sqrt{t}-\lambda\beta_{t}}{1+z}\wel_{t-1}> 0. \qedhere
\end{equation*}
\end{proof}

\subsubsection{Auxiliary lemmas for Algorithm~\ref{algorithm:1d}}

The first auxiliary lemma states that the betting fraction $\beta_t$ changes slowly. 

\begin{lemma}\label{lemma:b_fraction}
For all $t\geq 1$, $|\beta_{t+1}-\beta_t|\leq 2/(Ct)$.
\end{lemma}

\begin{proof}[Proof of Lemma~\ref{lemma:b_fraction}]The result for $t=1$ trivially holds. We only consider $t\geq 2$. 

Since the Euclidean projection to a closed convex set is contractive, we have
\begin{equation*}
\abs{\Pi_{\B_t}(\hat\beta_t)-\Pi_{\B_t}(\hat\beta_{t+1})}\leq \abs{\hat\beta_t-\hat\beta_{t+1}}=\abs{\frac{\tilde g_t+2C^2\hat\beta_t}{2C^2t}}\leq \frac{G}{C^2t}. 
\end{equation*}
Moreover, 
\begin{equation*}
\abs{\Pi_{\B_t}(\hat\beta_{t+1})-\Pi_{\B_{t+1}}(\hat\beta_{t+1})}\leq\abs{\frac{1}{\sqrt{2}C\sqrt{t-1}}-\frac{1}{\sqrt{2}C\sqrt{t}}}\leq \frac{1}{2\sqrt{2}C\sqrt{t}(t-1)}\leq \frac{1}{Ct}.
\end{equation*}
Applying the triangle inequality yields the result.
\end{proof}

The next lemma quantifies the movement of Algorithm~\ref{algorithm:1d} using $\wel_t$. By doing this, bounding the movement cost (Part 2 of Theorem~\ref{thm:1d}) reduces to bounding the growth of $\wel_t$. 

\begin{lemma}\label{lemma:perstep}For all $t\geq 1$, 
\begin{equation*}
\left|\tilde x_t-\tilde x_{t+1}\right|\leq \frac{6}{Ct}\wel_{t-1}.
\end{equation*}
\end{lemma}

\begin{proof}[Proof of Lemma~\ref{lemma:perstep}]Assume $t>1$ for the rest of this proof; the case of $t=1$ can be verified similarly. Starting from (\ref{eq:wel_update}), some simple algebra yields
\begin{align*}
\tilde x_{t+1}-\tilde x_t&=\beta_{t+1}\wel_t-\beta_{t}\wel_{t-1}\\
&=\rpar{\beta_{t+1}-\beta_t-\beta_{t+1}\tilde g_t\beta_t-\beta_{t+1}\beta_t\frac{\gamma}{\sqrt{t}}}\wel_{t-1}-\lambda\beta_{t+1}\abs{\beta_{t+1}\wel_t-\beta_{t}\wel_{t-1}}.
\end{align*}
From Lemma~\ref{lemma:unique_solution}, $\wel_{t-1}> 0$, therefore, 
\begin{equation*}
(1-\lambda\beta_{t+1})\abs{\beta_{t+1}\wel_t-\beta_{t}\wel_{t-1}}\leq \abs{\beta_{t+1}-\beta_t-\beta_{t+1}\tilde g_t\beta_t-\beta_{t+1}\beta_t\frac{\gamma}{\sqrt{t}}}\wel_{t-1}.
\end{equation*}
Note that $1-\lambda\beta_{t+1}\geq 1/2$. 
\begin{align*}
\abs{\beta_{t+1}\wel_t-\beta_{t}\wel_{t-1}}&\leq 2\abs{\beta_{t+1}-\beta_t-\beta_{t+1}\tilde g_t\beta_t-\beta_{t+1}\beta_t\frac{\gamma}{\sqrt{t}}}\wel_{t-1}\\
&\leq 2\abs{\beta_{t+1}-\beta_t}\wel_{t-1}+2\beta_t\beta_{t+1}\abs{\tilde g_t+\frac{\gamma}{\sqrt{t}}}\wel_{t-1}.
\end{align*}
Applying Lemma~\ref{lemma:b_fraction} and the definition of $\beta_t$ and $\beta_{t+1}$,
\begin{equation*}
\norm{\tilde x_t-\tilde x_{t+1}}\leq\rpar{\frac{4}{Ct}+\frac{2C}{2C^2\sqrt{t(t-1)}}}\wel_{t-1}\leq\frac{6}{Ct}\wel_{t-1}.\qedhere
\end{equation*}
\end{proof}

Following the reasoning from the previous lemma, we next bound the growth rate of $\wel_t$ in Lemma~\ref{lemma:wealthupper} which could be of special interest. The key idea is that, the surrogate loss (Line~\ref{line:1d_surrogate} of Algorithm~\ref{algorithm:1d}) incentivizes the \emph{unconstrained prediction} $\tilde x_t$ to be bounded. Equivalently, the betting amount in the coin-betting algorithm is bounded, and hence the wealth cannot grow too fast. (For some background knowledge on this argument, Appendix~\ref{subsection:coinbettingOLO} provides an overview of the interplay between coin-betting and OLO.)

As discussed in Section~\ref{section:contribution1}, our proof makes a novel use of the black-box reduction from unconstrained OLO to constrained OLO (Algorithm~\ref{algorithm:constraint}): actually, we \emph{do not use it as a black-box}, but rather analyze its impact on the unconstrained algorithm. To our knowledge, this is the first analysis that takes this perspective. 

\begin{lemma}\label{lemma:wealthupper}For all $t\geq 1$, $\wel_t\leq 4\bar RC\sqrt{t}$. 
\end{lemma}

\begin{proof}[Proof of Lemma~\ref{lemma:wealthupper}]Note that from Lemma~\ref{lemma:unique_solution}, $\wel_t\geq 0$. Additionally from our definition of $\beta_t$, we have $\beta_t,x_t,\tilde x_t\geq 0$.

We prove this lemma in three steps. First, we show a weaker result, $\wel_t\leq G\bar R(t+1)$. Using this result, we then prove that $\tilde x_t\leq 2\sqrt{2}\bar R$. In other words, even though $\tilde x_t$ is the output of a coin-betting-based OLO algorithm that works in the unbounded domain, \emph{it is actually bounded} due to the effect of the surrogate losses. Finally, we revisit wealth and show that $\wel_t\leq 4\bar RC\sqrt{t}$. 

\paragraph{Step 1}Prove that for all $t\geq 0$, $\wel_t\leq G\bar R(t+1)$. \bigskip

Consider the two cases in the definition of $\tilde g_t$. If $g_t\tilde x_t\geq g_tx_t$, then $\tilde g_t=g_t$, and
\begin{align*}
\wel_t&=\wel_{t-1}-\tilde g_t\tilde x_t-\lambda|\tilde x_t-\tilde x_{t+1}|-\frac{\gamma}{\sqrt{t}}|\tilde x_t|\\
&\leq \wel_{t-1}-g_tx_t\leq \wel_{t-1}+|g_t|\bar R.
\end{align*}
If $g_t\tilde x_t< g_tx_t$, then $\tilde g_t=0$ and $\wel_t\leq \wel_{t-1}$. An induction and $\eps\leq G\bar R$ yield the result. 

\paragraph{Step 2}Prove that for all $t\geq 1$, $\tilde x_t\leq 2\sqrt{2}\bar R$. \bigskip

This holds trivially for $t=1$. We use induction: suppose this result holds for $t$, and we need to show $\tilde x_{t+1}\leq 2\sqrt{2}\bar R$. There are two cases: (1) $\tilde x_t\notin \V_{1d}$; (2) $\tilde x_t\in \V_{1d}$. Note that the first case is only possible when $t>1$. 

\begin{itemize}[itemindent=25pt]
\item[Case (1.1)]  $\tilde x_t\notin \V_{1d}$, $g_t\tilde x_t\geq g_tx_t$. 

In this case, $\tilde g_t=g_t\geq 0$ and $g_tx_t\geq 0$. It follows,
\begin{equation*}
\wel_t\leq \wel_{t-1}-g_tx_t\leq \wel_{t-1}. 
\end{equation*}

Next we consider the three cases of $\beta_t$.

(i) First, note that $\beta_t\neq 0$; otherwise $\tilde x_t=\beta_t\wel_{t-1}=0\in\V_{1d}$. 

(ii) If $\beta_t=\hat\beta_t=-\sum_{i=1}^{t-1}\tilde g_i/[2C^2(t-1)]$, then
\begin{equation*}
\beta_{t+1}\leq \abs{\hat\beta_{t+1}}=\frac{1}{2C^2t}\left|-\sum_{i=1}^t\tilde g_i\right|=\frac{\left|2C^2(t-1)\beta_{t}-g_t\right|}{2C^2t}\leq \max\left\{\frac{t-1}{t}\beta_t,\frac{g_t}{2C^2t}\right\}.
\end{equation*}
The last inequality is due to $\beta_t,g_t\geq 0$. Therefore, 
\begin{equation*}
\tilde x_{t+1}=\beta_{t+1}\wel_t\leq \max\left\{\beta_t\wel_{t-1},G\wel_{t-1}/(2C^2t)\right\}\leq \max\{2\sqrt{2}\bar R,G^2\bar R/(2C^2)\}\leq 2\sqrt{2}\bar R,
\end{equation*}
where we use the result from Step 1. 

(iii) If $\beta_t=1/(C\sqrt{2(t-1)})$, then
\begin{equation*}
\tilde x_{t+1}=\beta_{t+1}\wel_t\leq \frac{1}{C\sqrt{2t}}\wel_{t-1}\leq\frac{1}{C\sqrt{2(t-1)}}\wel_{t-1}=\beta_t\wel_{t-1}\leq 2\sqrt{2}\bar R.
\end{equation*}
\end{itemize}

\begin{itemize}[itemindent=33pt]
\item[Case (1.2)]$\tilde x_t\notin \V_{1d}$, $g_t\tilde x_t< g_tx_t$. 

In this case, $\tilde g_t=0$ and $\wel_t\leq \wel_{t-1}$. Same as Case (1.1), $\beta_t\neq 0$, leading to $\hat\beta_t\geq 0$ and $\beta_t=\min\{\hat\beta_t,1/(C\sqrt{2(t-1)})\}$. Also note that
\begin{equation*}
\abs{\hat\beta_{t+1}}=\frac{1}{2C^2t}\left|-\sum_{i=1}^t\tilde g_i\right|=\frac{1}{2C^2t}\left|-\sum_{i=1}^{t-1}\tilde g_i\right|\leq \frac{1}{2C^2(t-1)}\left|-\sum_{i=1}^{t-1}\tilde g_i\right|=\abs{\hat\beta_t}. 
\end{equation*}
Therefore, 
\begin{equation*}
\beta_{t+1}\leq\min\left\{\abs{\hat\beta_{t+1}},\frac{1}{C\sqrt{2t}}\right\}\leq \min\left\{\abs{\hat\beta_t},\frac{1}{C\sqrt{2(t-1)}}\right\}=\beta_t,
\end{equation*}
and $\tilde x_{t+1}=\beta_{t+1}\wel_t\leq \beta_t\wel_{t-1}\leq\tilde x_t\leq 2\sqrt{2}\bar R$. 
\end{itemize}

\begin{itemize}[itemindent=25pt]
\item[Case (2)]$\tilde x_t\in \V_{1d}$. 

In this case, $\tilde x_t=x_t$ and $\tilde g_t=g_t$. $\tilde x_{t+1}=\beta_{t+1}\wel_t\leq(1-g_t\beta_t)\beta_{t+1}\wel_{t-1}$. 

If $t=1$, then $\tilde x_{t+1}=\beta_{t+1}\wel_t\leq \sqrt{2}G\bar R/C\leq \sqrt{2}\bar R$, where we use $\wel_1\leq 2G\bar R$ from Step 1 and $\beta_{2}\leq 1/(\sqrt{2}C)$. 

If $t>1$, we consider the three cases of $\beta_t$ as follows. (For the rest of the discussion assume $t>1$.)

(i) If $\beta_t=0$, then from Lemma~\ref{lemma:b_fraction} we have $\beta_{t+1}\leq 2/(Ct)$, and $\tilde x_{t+1}\leq(1-g_t\beta_t)\beta_{t+1}\wel_{t-1}=\beta_{t+1}\wel_{t-1}\leq 2G\bar R/C\leq 2\bar R$. 

(ii) If $\beta_t=\hat\beta_t=-\sum_{i=1}^{t-1}\tilde g_i/[2C^2(t-1)]$, then
\begin{equation*}
\beta_{t+1}\leq \abs{\hat\beta_{t+1}}=\frac{1}{2C^2t}\left|-\sum_{i=1}^t\tilde g_i\right|=\frac{\left|2C^2(t-1)\beta_{t}-g_t\right|}{2C^2t}\leq \frac{t-1}{t}\beta_t+\frac{G}{2C^2t}.
\end{equation*}
Note that since $\tilde x_t\in \V_{1d}$, we have $\beta_t\wel_{t-1}\leq \bar R$. Using $\tilde x_{t+1}\leq(1-g_t\beta_t)\beta_{t+1}\wel_{t-1}$ and $|g_t\beta_t|\leq 1/2$ we have
\begin{equation*}
\tilde x_{t+1}\leq \frac{3}{2}\left(\frac{t-1}{t}\beta_t\wel_{t-1}+\frac{G}{2C^2t}\wel_{t-1}\right)
\leq \frac{3}{2}\left(1+\frac{G^2}{2C^2}\right)\bar R\leq2\sqrt{2}\bar R. 
\end{equation*}

(iii) If $\beta_t=1/(C\sqrt{2(t-1)})$, then
\begin{equation*}
\beta_{t+1}\leq 1/(C\sqrt{2t})\leq 1/(C\sqrt{2(t-1)})=\beta_t, 
\end{equation*}
\begin{equation*}
\tilde x_{t+1}\leq(1-g_t\beta_t)\beta_{t+1}\wel_{t-1}\leq 2\beta_t\wel_{t-1}\leq 2\bar R.
\end{equation*}
\end{itemize}

\paragraph{Step 3}Prove that for all $t\geq 1$, $\wel_t\leq 4\bar RC\sqrt{t}$. \bigskip

Considering $\beta_{t+1}$, there are three cases: (1) $\beta_{t+1}=1/(C\sqrt{2t})$; (2) $\beta_{t+1}=\hat\beta_{t+1}$; and (3) $\beta_{t+1}=0$. For the first case, this result follows from $\tilde x_{t+1}=\beta_{t+1}\wel_t\leq 2\sqrt{2}\bar R$. Now consider the second case. 
\begin{align*}
\log\wel_t&\leq\log\eps+\sum_{i=1}^t\log(1-\tilde g_i \beta_i)\\
&\leq \log\eps-\sum_{i=1}^t\tilde g_i \beta_i\\
&= \log\eps-\sum_{i=1}^t\rpar{\tilde g_i \beta_i+C^2\beta_i^2}+C^2\sum_{i=1}^t\beta_i^2.
\end{align*}

$\beta_t$ is the output of Follow the Leader (FTL) on the strongly convex losses $\psi_t(\beta)=\tilde g_t \beta+C^2\beta^2+I\{0\leq\beta\leq 1/(C\sqrt{2t})\}(\beta)$, where $I\{0\leq\beta\leq 1/(C\sqrt{2t})\}(\beta)$ is a convex function of $\beta$ that equals 0 when $0\leq\beta\leq 1/(C\sqrt{2t})$ and infinity otherwise. Therefore we can use standard FTL results to show that the regret is non-negative. 

Let $F_t(\beta)=\sum_{i=1}^{t-1}\psi_i(\beta)$, then $\beta_t\in \argmin F_t(\beta)$. From Lemma 7.1 of \citep{orabona2019modern}, for any $u\in\R$, 
\begin{equation*}
\sum_{i=1}^t\left[\psi_i(\beta_i)-\psi_i(u)\right]=\sum_{i=1}^{t}\left[F_i(\beta_i)-F_{i+1}(\beta_{i+1})+\psi_i(\beta_i)\right]+F_{t+1}(\beta_{t+1})-F_{t+1}(u).
\end{equation*}
Note that if $u=\beta_{t+1}$, we have $\rhs\geq 0$. Therefore, 
\begin{align*}
\log\wel_t&\leq \log\eps-\min_{0\leq\beta\leq 1/(C\sqrt{2t})}\sum_{i=1}^t\rpar{\tilde g_i \beta+C^2 \beta^2}+C^2\sum_{i=1}^t\beta_i^2\\
&\leq \log\eps-\min_{\beta\in\R}\sum_{i=1}^t\rpar{\tilde g_i \beta+C^2\beta^2}+C^2\sum_{i=1}^t\beta_i^2\\
&\leq \log\eps+\frac{\left(\sum_{i=1}^t\tilde g_i\right)^2}{4C^2t}+\frac{1}{2}\sum_{\tau=1}^{t-1}\tau^{-1}.
\end{align*}
The last term is bounded by $(1+\log t)/2$. From the assumption of the second case, $|\sum_{i=1}^t\tilde g_i|< C\sqrt{2t}$. Combining everything we have $\log\wel_t\leq 1+\log\eps+(\log t)/2$ and $\wel_t\leq e\eps\sqrt{t}\leq e\bar RC\sqrt{t}$. 

Finally consider the third case. Same as the above, we have
\begin{equation*}
\log\wel_t\leq \log\eps-\min_{0\leq\beta\leq 1/(C\sqrt{2t})}\sum_{i=1}^t\rpar{\tilde g_i \beta+C^2 \beta^2}+C^2\sum_{i=1}^t\beta_i^2.
\end{equation*}
Since $\beta_{t+1}=0$, we have $\sum_{i=1}^t\tilde g_i\geq 0$. Therefore, 
\begin{equation*}
\log\wel_t\leq \log\eps+C^2\sum_{i=1}^t\beta_i^2\leq \log\eps+\frac{1}{2}(1+\log t), 
\end{equation*}
and $\wel_t\leq \sqrt{e}\bar RC\sqrt{t}$. 
\end{proof}

\subsubsection{Proof of Theorem~\ref{thm:1d}}

Now we are ready to prove Theorem~\ref{thm:1d}, the performance guarantee of Algorithm~\ref{algorithm:1d}. This is our first main theoretical result. 

\oneD*

\begin{proof}[Proof of Theorem~\ref{thm:1d}]
We prove the two parts of Theorem~\ref{thm:1d} separately, starting from the second part. 

Combining Lemma~\ref{lemma:perstep} and Lemma~\ref{lemma:wealthupper}, for all $t\geq 2$, 
\begin{equation*}
\left|\tilde x_t-\tilde x_{t+1}\right|\leq \frac{6}{Ct}\cdot 4\bar RC\sqrt{t-1}\leq 24\bar R\frac{1}{\sqrt{t}}.
\end{equation*}
For $t=1$, the same result can be verified. Therefore, for all $[a:b]\subset[1:T]$, 
\begin{equation*}
\sum_{t=a}^b\abs{x_t-x_{t+1}}\leq 24\bar R\sum_{t=a}^b\frac{1}{\sqrt{t}}\leq 24\bar R\int_{a-1}^b\frac{1}{\sqrt{x}}dx\leq 24\bar R\rpar{2\sqrt{b}-2\sqrt{a-1}}\leq 48\bar R\sqrt{b-a+1}.
\end{equation*}
The fourth inequality is due to $\sqrt{b}-\sqrt{a-1}\leq \sqrt{b-a+1}$. 

Now consider the proof of the first part of the theorem. Due to the complexity, we proceed in steps. 

\paragraph{Step 1} The overall strategy \bigskip

The considered bound does not rely on the bounded domain, therefore the first step is to apply the reduction from constrained OLO to unconstrained OLO (Lemma~\ref{lemma:constraint}) and the contraction property of Euclidean projection to show that
\begin{equation}
\sum_{t=1}^T\left(g_tx_t-g_tu+\lambda\abs{x_t- x_{t+1}}+\frac{\gamma}{\sqrt{t}}\abs{x_t}\right)\leq \sum_{t=1}^T\left(\tilde g_t\tilde x_t-\tilde g_tu+\lambda\abs{\tilde x_t- \tilde x_{t+1}}+\frac{\gamma}{\sqrt{t}}\abs{\tilde x_t}\right).\label{eq:constraint}
\end{equation}
Note that $\wel_{t-1}$ is positive due to Lemma~\ref{lemma:unique_solution}, and $\beta_t\geq 0$ from our construction. Therefore, $\tilde x_t\geq 0$. From here, we can focus on bounding the RHS of (\ref{eq:constraint}) with $\abs{\tilde x_t}$ replaced by $\tilde x_t$. Also note that $\abs{\tilde g_t}\leq \abs{g_t}\leq G$ from Lemma~\ref{lemma:constraint}. 

From (\ref{eq:wel_update}), we can rewrite wealth as
\begin{equation*}
\wel_T=\eps-\sum_{t=1}^T\rpar{\tilde g_t\tilde x_t+\lambda\abs{\tilde x_t-\tilde x_{t+1}}+\frac{\gamma}{\sqrt{t}}\tilde x_t}. 
\end{equation*}
If we guarantee $\wel_T\geq F(-\sum_{t=1}^T\tilde g_t)$ for an arbitrary function $F$, then
\begin{align*}
\sum_{t=1}^T\left(\tilde g_t\tilde x_t-\tilde g_tu+\lambda\abs{\tilde x_t- \tilde x_{t+1}}+\frac{\gamma}{\sqrt{t}}\tilde x_t\right)&=\eps+\inner{-\sum_{t=1}^T\tilde g_t}{u}-\wel_T\\
&\leq \eps+\inner{-\sum_{t=1}^T\tilde g_t}{u}-F\rpar{-\sum_{t=1}^T\tilde g_t}\\
&\leq \eps+\sup_{X\in\R}\rpar{\inner{X}{u}-F\rpar{X}}=\eps+F^*(u),
\end{align*}
where $F^*$ is the Fenchel conjugate of $F$.
Therefore, our goal is to find such an lower bound for $\wel_T$, and then take its Fenchel conjugate. 

\paragraph{Step 2} Recursion on the wealth update \bigskip

Now consider (\ref{eq:wel_update}). There are two cases: (i) $\beta_t\wel_{t-1}\geq\beta_{t+1}\wel_t$; (ii) $\beta_t\wel_{t-1}<\beta_{t+1}\wel_t$. If $\beta_t\wel_{t-1}\geq\beta_{t+1}\wel_t$, then
\begin{equation*}
(1-\lambda\beta_{t+1})\wel_t=(1-\tilde g_t\beta_t-\lambda\beta_t-\gamma\beta_t/\sqrt{t})\wel_{t-1},
\end{equation*}
\begin{equation*}
\log\wel_t=\log\wel_{t-1}+\log[1-\beta_t(\tilde g_t+\lambda+\gamma/\sqrt{t})]-\log(1-\lambda\beta_{t+1}).
\end{equation*}
Note that $\beta_t|\tilde g_t+\lambda+\gamma/\sqrt{t}|\leq 1/2$ and $\lambda\beta_{t+1}<1$. Applying $\log(1-x)\geq -x-x^2$ for all $x\leq 1/2$ and $\log(1+x)\leq x$ for all $x>1$, we have
\begin{align*}
\log\wel_t&\geq\log\wel_{t-1}-\beta_t(\tilde g_t+\lambda+\gamma/\sqrt{t})-\beta^2_t(\tilde g_t+\lambda+\gamma/\sqrt{t})^2+\lambda\beta_{t+1}\\
&\geq\log\wel_{t-1}-\tilde g_t\beta_t-\gamma\beta_t/\sqrt{t}-C^2\beta^2_t+\lambda(\beta_{t+1}-\beta_t).
\end{align*}
Similarly, if $\beta_t\wel_{t-1}<\beta_{t+1}\wel_t$, then
\begin{equation*}
\log\wel_t\geq\log\wel_{t-1}-\tilde g_t\beta_t-\gamma\beta_t/\sqrt{t}-C^2\beta^2_t+\lambda(\beta_{t}-\beta_{t+1}).
\end{equation*}
Therefore, combining both cases, we have
\begin{equation*}
\log\wel_t\geq\log\wel_{t-1}-\tilde g_t\beta_t-\gamma\beta_t/\sqrt{t}-C^2\beta^2_t+\lambda|\beta_{t}-\beta_{t+1}|,
\end{equation*}
and summed over $[1:T]$, 
\begin{equation}
\log\wel_T\geq \log\eps-\sum_{t=1}^T\tilde g_t\beta_t-C^2\sum_{t=1}^T\beta^2_t-\gamma\sum_{t=1}^T\frac{\beta_t}{\sqrt{t}}-\lambda\sum_{t=1}^T|\beta_{t}-\beta_{t+1}|.\label{eq:log_wealth}
\end{equation}

\paragraph{Step 3} Bounding the sums on the RHS of (\ref{eq:log_wealth})\bigskip

We start from the first two sums on the RHS of (\ref{eq:log_wealth}). $\beta_t$ is the output of Follow the Leader (FTL) on the strongly convex losses $\psi_t(\beta)=\tilde g_t \beta+C^2\beta^2+I\{0\leq\beta\leq 1/(C\sqrt{2t})\}(\beta)$, where $I\{0\leq\beta\leq 1/(C\sqrt{2t})\}(\beta)$ is a convex function of $\beta$ that equals 0 when $0\leq\beta\leq 1/(C\sqrt{2t})$ and infinity otherwise. Note that $\psi_t$ is $2C^2$-strongly convex, therefore a standard result shows that the regret of this FTL problem is logarithmic in $T$. Concretely, from Corollary 7.17 of \citep{orabona2019modern}, 
\begin{equation*}
\sum_{t=1}^T\rpar{\tilde g_t\beta_t+C^2\beta_t^2}-\min_{0\leq u\leq 1/(C\sqrt{2T})}\sum_{t=1}^T\rpar{\tilde g_tu+C^2u^2}\leq \frac{G^2}{4C^2}\rpar{1+\log T}. 
\end{equation*}
Moreover, taking $u=1/(C\sqrt{2T})$,
\begin{equation*}
\min_{0\leq u\leq 1/(C\sqrt{2T})}\sum_{t=1}^T\rpar{\tilde g_tu+C^2u^2}\leq \frac{\sum_{t=1}^T\tilde g_t}{C\sqrt{2T}}+\frac{1}{2}.
\end{equation*}

As for the other sums in (\ref{eq:log_wealth}), 
\begin{equation*}
\sum_{t=1}^T\frac{\beta_t}{\sqrt{t}}=\frac{1}{\sqrt{2}C}\sum_{t=1}^T\frac{1}{t}\leq \frac{1}{\sqrt{2}C}(1+\log T).
\end{equation*}
Applying Lemma~\ref{lemma:b_fraction}, 
\begin{equation*}
\sum_{t=1}^T|\beta_{t}-\beta_{t+1}|\leq\frac{2}{C}\sum_{t=1}^T\frac{1}{t}\leq \frac{2}{C}(1+\log T).
\end{equation*}
Plugging the above into (\ref{eq:log_wealth}),
\begin{equation*}
\log\wel_T\geq \log\eps-\frac{\sum_{t=1}^T\tilde g_t}{C\sqrt{2T}}-2(1+\log T)-\frac{1}{2},
\end{equation*}
\begin{equation*}
\wel_T\geq \frac{\eps}{\exp(5/2)\cdot T^2}\exp\rpar{-\frac{\sum_{t=1}^T\tilde g_t}{C\sqrt{2T}}}.
\end{equation*}

\paragraph{Step 4} Taking Fenchel conjugate

From the Fechel conjugate table, if $f(x)=a\exp(bx)$ with $a,b>0$, then for all $\theta\geq 0$,
\begin{equation*}
f^*(\theta)=\frac{\theta}{b}\rpar{\log\frac{\theta}{ab}-1}. 
\end{equation*}
Applying this result on
\begin{equation*}
F(x)=\frac{\eps}{\exp(5/2)\cdot T^2}\exp\rpar{\frac{x}{C\sqrt{2T}}}, 
\end{equation*}
for all $u\geq 0$ we have
\begin{equation*}
F^*(u)=uC\sqrt{2T}\rpar{\frac{3}{2}+\log\frac{\sqrt{2}uCT^{5/2}}{\eps}}.
\end{equation*}
Combining the above with Step 1 completes the proof. 
\end{proof}

\subsection{Analysis of Algorithm~\ref{algorithm:higherd}}\label{subsection:alg2_analysis}

Algorithm~\ref{algorithm:higherd} extends the one-dimensional coin-betting-based OLO algorithm to higher dimensions via a polar decomposition. Here we incorporate movement cost into the analysis of \citep{cutkosky2018black}. 

\begin{theorem}\label{thm:higherd} For all $\lambda\geq 0$, $G>0$ and $0<\eps\leq GR$, applying Algorithm~\ref{algorithm:higherd} yields the following performance guarantee: 
\begin{enumerate}
\item For all $T\in\N_+$ and $u\in\ball^d(0,R)$, 
\begin{equation*}
\sum_{t=1}^T \inner{g_t}{x_t-u} +\lambda\sum_{t=1}^{T-1}\norm{x_t- x_{t+1}}\leq \eps+\norm{u}\tilde O\spar{(G+\lambda)\sqrt{T}}, 
\end{equation*}
where $\tilde O(\cdot)$ subsumes logarithmic factors on $u$, $G$, $\lambda$, $T$ and $\eps^{-1}$.
\item For all $b\geq a\geq 1$, 
\begin{equation*}
\sum_{t=a}^{b-1}\norm{x_t-x_{t+1}}\leq 50R\sqrt{b-a}.
\end{equation*}
\end{enumerate}
\end{theorem}

\begin{proof}[Proof of Theorem~\ref{thm:higherd}]
We only consider the case of $u\neq 0$. If $u=0$, the result can be easily verified. Notice that $|\langle g_t,z_t\rangle|\leq G$, therefore we can apply Theorem~\ref{thm:1d} on $\A_r$. 
\begin{align}
\nonumber&\sum_{t=1}^T\inner{g_t}{y_tz_t-u}+\lambda\sum_{t=1}^{T-1}\norm{y_tz_t- y_{t+1}z_{t+1}}\\
\nonumber\leq~&\sum_{t=1}^T\rpar{\inner{g_t}{z_t} y_t-\inner{g_t}{z_t} \norm{u}}+\norm{u}\sum_{t=1}^T\inner{g_t}{z_t-\frac{u}{\norm{u}}}+\lambda\sum_{t=1}^{T-1}\abs{y_t- y_{t+1}}\norm{z_{t+1}}+\lambda\sum_{t=1}^{T-1}\norm{z_t-z_{t+1}}\abs{y_{t}}\\
\leq~&\sum_{t=1}^T\rpar{\inner{g_t}{z_t} y_t-\inner{g_t}{z_t} \norm{u}}+\lambda\sum_{t=1}^{T-1}\abs{y_t- y_{t+1}}+\sum_{t=1}^{T-1}\frac{\lambda}{\sqrt{t}}y_{t}+\norm{u}\sum_{t=1}^T\inner{g_t}{z_t-\frac{u}{\norm{u}}}.\label{eq:two_parts}
\end{align}
The last inequality is due to $\norms{z_{t+1}}\leq 1$ and $\norms{z_t-z_{t+1}}\leq \eta_t G=1/\sqrt{t}$. 

The first three terms of (\ref{eq:two_parts}) are bounded by Theorem~\ref{thm:1d}, 
\begin{multline*}
\sum_{t=1}^T\rpar{\inner{g_t}{z_t} y_t-\inner{g_t}{z_t} \norm{u}}+\lambda\sum_{t=1}^{T-1}\abs{y_t- y_{t+1}}+\sum_{t=1}^{T-1}\frac{\lambda}{\sqrt{t}}y_{t}\\
\leq \eps+ \norm{u}(G+2\lambda)\sqrt{2T}\rpar{\frac{3}{2}+\log\frac{\sqrt{2}\norm{u}(G+2\lambda)T^{5/2}}{\eps}}.
\end{multline*}
As for the last term of (\ref{eq:two_parts}), we can use the standard OGD regret bound. From Section~4.2.1 of \citep{orabona2019modern}, 
\begin{equation*}
\sum_{t=1}^T\inner{g_t}{z_t-\frac{u}{\norm{u}}}\leq \frac{3}{2}G\sqrt{T}.
\end{equation*}
Combining everything so far yields the first part of the theorem. 

As for the second part of the theorem, for all $b\geq a\geq 1$, 
\begin{equation*}
\sum_{t=a}^{b-1}\norm{x_t-x_{t+1}}\leq \sum_{t=a}^{b-1}\rpar{\abs{y_t-y_{t+1}}+\frac{R}{\sqrt{t}}}\leq 50R\sqrt{b-a}.
\end{equation*}
The last inequality is due to Theorem~\ref{thm:1d} and $\sum_{t=a}^{b-1}1/\sqrt{t}\leq 2\sqrt{b-a}$.
\end{proof}

\section{Details on strongly adaptive OCOM}\label{section:proofocom}

This section provides detailed analysis of our strongly adaptive OCOM algorithm. We first present the performance guarantees of our subroutines based on Algorithm~\ref{algorithm:adagradient}. Then, we introduce the complete version of our meta-algorithm (Algorithm~\ref{algorithm:meta}) and present its analysis. 

\subsection{Analysis of Algorithm~\ref{algorithm:adagradient}}

Algorithm~\ref{algorithm:adagradient} is used to define our two-part subroutine (on GC intervals). The idea of adaptively slowing down the base algorithm is inspired by Algorithm~7 of \citep{cutkosky2018algorithms} for memoryless OLO. Here we make two improvements: (i) incorporating movement costs; (ii) using this framework to achieve better dependence on problem constants. 

\begin{theorem}\label{thm:subroutineoned} For all $\lambda,G> 0$ and $0<\eps\leq G$, Subroutine-1d defined from Algorithm~\ref{algorithm:adagradient} yields the following performance guarantee: 
\begin{enumerate}
\item For all $T\in\N_+$ and $u\in[0,1]$, 
\begin{equation*}
\sum_{t=1}^T g_t(x_t-u) +\lambda\sum_{t=1}^{T-1}\abs{x_t- x_{t+1}}\leq\eps+\abs{u}\tilde O\rpar{\max\{\lambda,G\}+\sqrt{\max\{\lambda,G\}\sum_{t=1}^T\abs{g_t}}}, 
\end{equation*}
where $\tilde O(\cdot)$ subsumes logarithmic factors on $u$, $G$, $\lambda$, $T$ and $\eps^{-1}$.
\item For all $b\geq a\geq 1$, 
\begin{equation*}
\sum_{t=a}^{b-1}\norm{x_t-x_{t+1}}\leq 48\rpar{1+\sqrt{\frac{\sum_{t=a}^{b-1}\abs{g_t}}{\max\{\lambda,G\}}}}.
\end{equation*}
\end{enumerate}
\end{theorem}

\begin{theorem}\label{thm:subroutineball} 
For all $\lambda,G> 0$ and $0<\eps\leq GR$, Subroutine-ball defined from Algorithm~\ref{algorithm:adagradient} yields the following performance guarantee: 
\begin{enumerate}
\item For all $T\in\N_+$ and $u\in\ball^d(0,R)$,
\begin{equation*}
\sum_{t=1}^T \inner{g_t}{x_t-u} +\lambda\sum_{t=1}^{T-1}\norm{x_t- x_{t+1}}\leq \eps+\norm{u}\tilde O\rpar{\max\{\lambda,G\}+\sqrt{\max\{\lambda,G\}\sum_{t=1}^T\norms{g_t}}}, 
\end{equation*}
where $\tilde O(\cdot)$ subsumes logarithmic factors on $u$, $G$, $\lambda$, $T$ and $\eps^{-1}$.
\item For all $b\geq a\geq 1$, 
\begin{equation*}
\sum_{t=a}^{b-1}\norm{x_t-x_{t+1}}\leq 50R\rpar{1+\sqrt{\frac{\sum_{t=a}^{b-1}\norm{g_t}}{\max\{\lambda,G\}}}}.
\end{equation*}
\end{enumerate}
\end{theorem}

We only prove the guarantee on Subroutine-ball (Theorem~\ref{thm:subroutineball}). The guarantee on Subroutine-1d (Theorem~\ref{thm:subroutineoned}) is similar, therefore the proof is omitted. 

\begin{proof}[Proof of Theorem~\ref{thm:subroutineball}]
Consider the first part of the theorem. Let $i_T$ be the index $i$ at the beginning of the $T$-th round, and let $Z_1,\ldots,Z_{i_T}$ be their final value at the end of the algorithm. Notice that
\begin{equation*}
\sum_{t=1}^T \inner{g_t}{x_t-u} +\lambda\sum_{t=1}^{T-1}\norm{x_t- x_{t+1}}=\sum_{i=1}^{i_T}\inner{Z_i}{w_i-u}+\lambda\sum_{i=1}^{i_T-1}\norm{w_i- w_{i+1}}. 
\end{equation*}
For the RHS we can use Theorem~\ref{thm:higherd}, since for all $i$, $\norms{Z_i}\leq \max\{\lambda,G\}+G$. The remaining task is to bound $i_T$. Note that $\sum_{i=1}^{i_T}\norms{Z_i}\leq\sum_{i=1}^T\norms{g_t}$ and $\norms{Z_i}>\max\{\lambda,G\}$ for all $i<i_T$, therefore $i_T\leq 1+(\sum_{t=1}^T\norms{g_t})/\max\{\lambda,G\}$. Plugging this into Theorem~\ref{thm:higherd} completes the proof of the first part. 

As for the second part of the theorem, let $i_a$, $i_b$ be the index $i$ at the beginning of the $a$-th and the $b$-th round. 
\begin{equation*}
\sum_{t=a}^{b-1}\norm{x_t-x_{t+1}}=\sum_{i=i_a}^{i_b-1}\norm{w_i-w_{i+1}}.
\end{equation*}
Next consider $i_b-i_a$. Let $Z^*_{i_a}$ and $Z^*_{i_b}$ be the value of accumulators $Z_{i_a}$ and $Z_{i_b}$ at the beginning of the $a$-th round and the $b$-th round, respectively. Note that
\begin{equation*}
\norm{Z_{i_a}-Z^*_{i_a}}+\norm{Z^*_{i_b}}+\sum_{i=i_a+1}^{i_b-1}\norm{Z_i}\leq \sum_{t=a}^{b-1}\norm{g_t},
\end{equation*}
and $\norms{Z_i}>\max\{\lambda,G\}$ for all $i\in[i_a+1,i_b-1]$. Therefore, $i_b-i_a\leq 1+(\sum_{t=a}^{b-1}\norms{g_t})/\max\{\lambda,G\}$. Applying the second part of Theorem~\ref{thm:higherd} completes the proof. 
\end{proof}

\subsection{Analysis of the meta-algorithm}\label{subsection:metafull}

Now we proceed to our meta-algorithm for strongly adaptive OCOM. The pseudo-code is Algorithm~\ref{algorithm:meta}. Before providing its performance guarantee, we present a lemma that explains the adopted projection scheme. (Line~\ref{line:metaprojection} and \ref{line:metasurrogate})

\begin{algorithm*}[ht]
\caption{The meta-algorithm for strongly adaptive OCOM. (The complete version of Algorithm~\ref{algorithm:metaabridged})\label{algorithm:meta}}
\begin{algorithmic}[1]
\REQUIRE Time horizon $T\geq 1$ and a hyperparameter $\eps_0>0$. 
\STATE Define a constant $\lambda=LH(H+1)$.
\FOR{$t=1,\ldots,T$}
\STATE Find the $(k,i)$ index pair for all the GC intervals that start in the $t$-th round. For each, (i) initialize a copy of Subroutine-ball as $\A^{k}_B$, with hyperparameters $(\lambda,2^k\eps_0,\tilde G)$; and (ii) initialize a copy of Subroutine-1d as $\A^{k}_{1d}$, with hyperparameters $(\lambda R,2^k\eps_0,\tilde GR)$. If $\A^k_B$ and $\A^k_{1d}$ already exist in the memory, overwrite them. 
\STATE Define $K_t=\lceil \log_2 (t+1)\rceil-1$. Let $\tilde x^{(K_t+1)}_t=0\in\R^d$. 
\FOR{$k=K_t,\ldots,0$}
\STATE \label{line:metaprojection}Let $x^{(k+1)}_t=\Pi_{\ball^d(0,R)}(\tilde x^{(k+1)}_t)$.
\STATE Query a prediction from $\A^k_B$ and assign it to $w^{(k)}_t$; query a prediction from $\A^k_{1d}$ and assign it to $z^{(k)}_t$. 
\STATE Let $\tilde x^{(k)}_t=(1-z^{(k)}_t)x^{(k+1)}_t+w^{(k)}_t$.
\ENDFOR
\STATE Let $\tilde x_t=\tilde x^{(0)}_t$, predict $x_t=\Pi_{\V}(\tilde x_t)$, suffer $l_t(x_{t-H:t})$, receive $l_t$.
\STATE \label{line:metaprojectionextra}Obtain a subgradient $g_t\in\partial \tilde l_t(x_{t})$. Define a surrogate loss function $h_t$ as
\begin{equation*}
h_t(x)=\begin{cases}
\langle g_t,x\rangle, &\textrm{if~}\langle g_t,\tilde x_t\rangle\geq \langle g_t,x_t\rangle,\\
\langle g_t,x\rangle+\langle g_t,x_t-\tilde x_t\rangle\frac{\norm{x-\Pi_{\V}(x)}}{\norm{x_t-\tilde x_t}}, &\textrm{otherwise}.
\end{cases}
\end{equation*}
\STATE Obtain a subgradient $\tilde g_t\in\partial h_t(\tilde x_t)$. Let $g^{(0)}_t=\tilde g_t$. 
\FOR{$k=0,\ldots,K_t$}
\STATE \label{line:metalosstosubroutine}Return $g^{(k)}_t$ to $\A^k_B$, and $-\langle g^{(k)}_t,x^{(k+1)}_t\rangle$ to $\A^k_{1d}$ as the loss gradients. 
\STATE \label{line:metasurrogate}Let $e^{(k+1)}_t=\tilde x^{k+1}_t/\norms{\tilde x^{k+1}_t}$, and
\begin{equation*}
g^{(k+1)}_t=\begin{cases}
g^{(k)}_t, &\textrm{if~}\langle g^{(k)}_t,\tilde x^{(k+1)}_t\rangle\geq \langle g^{(k)}_t,x^{(k+1)}_t\rangle,\\
g^{(k)}_t-\inner{g^{(k)}_t}{e^{(k+1)}_t}e^{(k+1)}_t, &\textrm{otherwise}.
\end{cases}
\end{equation*}
\ENDFOR
\ENDFOR
\end{algorithmic}
\end{algorithm*}

\begin{lemma}\label{lemma:metaconstraints}
For all $t$, 
\begin{enumerate}
\item $\norms{g^{(K_t+1)}_t}\leq \norms{g^{(K_t)}_t}\leq\ldots\leq\norms{g^{(0)}_t}\leq \norms{g_t}\leq \tilde G$.
\item For all $k\in[0:K_t]$ and $x\in\V$, $\inner{g^{(k)}_t}{x^{(k+1)}_t-x}\leq \inner{g^{(k+1)}_t}{\tilde x^{(k+1)}_t-x}$.
\end{enumerate}
\end{lemma}

Observe that Line~\ref{line:metaprojection} and \ref{line:metasurrogate} of Algorithm~\ref{algorithm:meta} are essentially applying Algorithm~\ref{algorithm:constraint} on the unprojected prediction $\tilde x^{(k+1)}_t$. Therefore, the proof of Lemma~\ref{lemma:metaconstraints} follows from recursively applying Lemma~\ref{lemma:constraint}. Line~\ref{line:metaprojectionextra} follows a similar principle. 

Now we are ready to prove the performance guarantee. 

\ocom*

\begin{proof}[Proof of Theorem~\ref{thm:ocom}]
Our strategy is to associate the regret of the meta-algorithm on any GC interval with the regret of the corresponding subroutines (Theorem~\ref{thm:subroutineoned} and Theorem~\ref{thm:subroutineball}). Then, applying these performance guarantees yields $\tilde O(\sqrt{|\I^{k,i}|})$ regret on all GC interval $\I^{k,i}\subset[1:T]$. This can be further extended to all general intervals $\I\subset[1:T]$ using an argument similar to \citep{daniely2015strongly}. 

To this end, we proceed in steps. Let $\I^{k^*,i^*}=[q:s]\subset [1:T]$ be a GC interval with indices $k^*$ and $i^*$. Since the amount of active GC intervals cannot increase in the duration of any GC interval, we can replace $K_t$ for all $t\in\I^{k^*,i^*}$ by a constant $K^*$. In other words, for all $t\in\I^{k^*,i^*}$, $K_t=K^*\leq \lfloor \log_2 (T+1)\rfloor-1$. 

\paragraph{Step 1}Reducing to one-step movement.\bigskip

We start from the Lipschitzness of $l_t$. For all $t$,
\begin{equation*}
l_t(x_{t-H:t})\leq \tilde l_t(x_{t})+L\sum_{h=1}^H\norm{x_{t-h}-x_{t}}\leq \tilde l_t(x_{t})+L\sum_{h=1}^H\sum_{j=1}^h\norm{x_{t-j}-x_{t-j+1}}.
\end{equation*}
Using the convexity of $\tilde l_t$, for all $x\in \V$, 
\begin{equation*}
\sum_{t=q}^s\spar{l_t(x_{t-H:t})-\tilde l_t(x)}\leq\sum_{t=q}^s\inner{g_t}{x_{t}-x}+L\sum_{t=q}^s\sum_{h=1}^H\sum_{j=1}^h\norm{x_{t-j}-x_{t-j+1}}.
\end{equation*}

Observe that
\begin{align*}
\sum_{t=q}^s\sum_{h=1}^H\sum_{j=1}^h\norm{x_{t-j}-x_{t-j+1}}\leq~&\frac{1}{2}H(H+1)\sum_{t=q}^{s-1}\norm{x_t-x_{t+1}}+\sum_{h=1}^H\frac{1}{2}(H+1-h)(H+2-h)\norm{x_{q-h}-x_{q-h+1}}\\
\leq~&\frac{1}{2}H(H+1)\sum_{t=q}^{s-1}\norm{x_t-x_{t+1}}+R\sum_{h=1}^Hh(h+1)\\
=~&\frac{1}{2}H(H+1)\sum_{t=q}^{s-1}\norm{x_t-x_{t+1}}+\frac{1}{3}RH(H+1)(H+2). 
\end{align*}
Therefore, combining the above and plugging in $\lambda$ for conciseness, 
\begin{align*}
\sum_{t=q}^s\spar{l_t(x_{t-H:t})-\tilde l_t(x)}&\leq \sum_{t=q}^s\inner{g_t}{x_t-x}+\frac{1}{2}LH(H+1)\sum_{t=q}^{s-1}\norm{x_t-x_{t+1}}+O(RLH^3)\\
&\leq \sum_{t=q}^s\inner{\tilde g_t}{\tilde x_t-x}+\frac{\lambda}{2}\sum_{t=q}^{s-1}\norm{\tilde x_t-\tilde x_{t+1}}+O(RLH^3),
\end{align*}
where the last line is due to Lemma~\ref{lemma:constraint} and the contraction property of Euclidean projection. 

\paragraph{Step 2}Showing that the ``temporary'' prediction $x^{(k^*)}_t$ after combining $\A^{(k^*)}_B$ and $\A^{(k^*)}_{1d}$ is good enough for the considered GC interval, although improper. \bigskip

Starting from the definition of $\tilde x^{(k^*)}_t$, for all $x\in \V$, 
\begin{equation*}
\inner{g^{(k^*)}_t}{\tilde x^{(k^*)}_t-x}=\inner{g^{(k^*)}_t}{w^{(k^*)}_t-x}+\rpar{-\inner{g^{(k^*)}_t}{x^{(k^*+1)}_t}}\rpar{z^{(k^*)}_t-1},
\end{equation*}
\begin{align}
\nonumber\norm{\tilde x^{(k^*)}_t-\tilde x^{(k^*)}_{t+1}}&=\norm{\rpar{1-z^{(k^*)}_t}x^{(k^*+1)}_t+w^{(k^*)}_t-\rpar{1-z^{(k^*)}_{t+1}}x^{(k^*+1)}_{t+1}-w^{(k^*)}_{t+1}}\\
\nonumber&\leq\norm{\rpar{1-z^{(k^*)}_t}\rpar{x^{(k^*+1)}_t-x^{(k^*+1)}_{t+1}}}+\norm{\rpar{z^{(k^*)}_{t}-z^{(k^*)}_{t+1}}x^{(k^*+1)}_{t+1}}+\norm{w^{(k^*)}_t-w^{(k^*)}_{t+1}}\\
\nonumber&\leq \norm{x^{(k^*+1)}_t-x^{(k^*+1)}_{t+1}}+R\abs{z^{(k^*)}_{t}-z^{(k^*)}_{t+1}}+\norm{w^{(k^*)}_t-w^{(k^*)}_{t+1}}\\
&\leq \norm{\tilde x^{(k^*+1)}_t-\tilde x^{(k^*+1)}_{t+1}}+R\abs{z^{(k^*)}_{t}-z^{(k^*)}_{t+1}}+\norm{w^{(k^*)}_t-w^{(k^*)}_{t+1}}\label{eq:proofonestep}\\
\nonumber&\leq R\sum_{k=k^*}^{K^*}\abs{z^{(k)}_{t}-z^{(k)}_{t+1}}+\sum_{k=k^*}^{K^*}\norm{w^{(k)}_t-w^{(k)}_{t+1}}.
\end{align}
The second line is due to triangle inequality. The third line is due to $z^{(k^*)}_t\in[0,1]$ and $\norms{x^{(k^*+1)}_{t+1}}\leq R$. The fourth line is due to the contraction of Euclidean projection, and the last line follows from a recursion. Combining the above, 
\begin{multline*}
\sum_{t=q}^s\inner{g^{(k^*)}_t}{\tilde x^{(k^*)}_t-x}+\frac{\lambda}{2}\sum_{t=q}^{s-1}\norm{\tilde x^{(k^*)}_t-\tilde x^{(k^*)}_{t+1}}\\
\leq \sum_{t=q}^s\inner{g^{(k^*)}_t}{w^{(k^*)}_t-x}+\frac{\lambda}{2}\sum_{t=q}^{s-1}\norm{w^{(k^*)}_t-w^{(k^*)}_{t+1}}+\sum_{t=q}^s\rpar{-\inner{g^{(k^*)}_t}{x^{(k^*+1)}_t}}\rpar{z^{(k^*)}_t-1}\\
+\frac{\lambda R}{2}\sum_{t=q}^{s-1}\abs{z^{(k^*)}_{t}-z^{(k^*)}_{t+1}}+\frac{\lambda R}{2}\sum_{k=k^*+1}^{K^*}\sum_{t=q}^{s-1}\abs{z^{(k)}_{t}-z^{(k)}_{t+1}}+\frac{\lambda}{2}\sum_{k=k^*+1}^{K^*}\sum_{t=q}^{s-1}\norm{w^{(k)}_t-w^{(k)}_{t+1}}.
\end{multline*}

Note that from Lemma~\ref{lemma:metaconstraints}, $\norms{g^{(k^*)}_t}\leq\norms{g^{(k^*-1)}_t}\leq\ldots\leq \norms{g_t}\leq \tilde G$. Moreover, $\eps_0\leq \tilde GR/(T+1)$ leads to $2^{k^*}\eps_0\leq \tilde G R$. Therefore, we can use the performance guarantees of the subroutine for the sums on the RHS. Applying Part 1 of Theorem~\ref{thm:subroutineoned} and Theorem~\ref{thm:subroutineball}, 
\begin{equation*}
\sum_{t=q}^s\inner{g^{(k^*)}_t}{w^{(k^*)}_t-x}+\frac{\lambda}{2}\sum_{t=q}^{s-1}\norm{w^{(k^*)}_t-w^{(k^*)}_{t+1}}\leq 2^{k^*}\eps_0+\tilde O\rpar{R\max\{\lambda,\tilde G\}+R\sqrt{\max\{\lambda,\tilde G\}\sum_{t=q}^s\norm{g_t}}}, 
\end{equation*}
\begin{equation*}
\sum_{t=q}^s\rpar{-\inner{g^{(k^*)}_t}{x^{(k^*+1)}_t}}\rpar{z^{(k^*)}_t-1}+\frac{\lambda R}{2}\sum_{t=q}^{s-1}\abs{z^{(k^*)}_{t}-z^{(k^*)}_{t+1}}
\leq2^{k^*}\eps_0+\tilde O\rpar{R\max\{\lambda,\tilde G\}+R\sqrt{\max\{\lambda,\tilde G\}\sum_{t=q}^s\norm{g_t}}}.
\end{equation*}
Also note that GC intervals longer than $\I^{k^*,i^*}$ cannot be initialized in the duration of $\I^{k^*,i^*}$. Therefore applying Part 2 of Theorem~\ref{thm:subroutineoned} and Theorem~\ref{thm:subroutineball}, for all $k\in[k^*+1:K^*]$, 
\begin{equation*}
\sum_{t=q}^{s-1}\norm{w^{(k)}_t-w^{(k)}_{t+1}}\leq 50R\rpar{1+\sqrt{\frac{\sum_{t=q}^{s-1}\norm{g^{(k)}_t}}{\max\{\lambda,\tilde G\}}}}\leq 50R\rpar{1+\sqrt{\frac{\sum_{t=q}^{s}\norm{g_t}}{\max\{\lambda,\tilde G\}}}}.
\end{equation*}
\begin{equation*}
\sum_{t=q}^{s-1}\abs{z^{(k)}_{t}-z^{(k)}_{t+1}}\leq 48\rpar{1+\sqrt{\frac{\sum_{t=q}^{s-1}\norm{g^{(k)}_t}}{\max\{\lambda,\tilde G\}}}}\leq 48\rpar{1+\sqrt{\frac{\sum_{t=q}^{s}\norm{g_t}}{\max\{\lambda,\tilde G\}}}}.
\end{equation*}
Notice that $K^*=O(\log T)$ and $\lambda\leq \tilde G$ from our definition. Combining everything so far, we have
\begin{equation*}
\sum_{t=q}^s\inner{g^{(k^*)}_t}{\tilde x^{(k^*)}_t-x}+\frac{\lambda}{2}\sum_{t=q}^{s-1}\norm{\tilde x^{(k^*)}_t-\tilde x^{(k^*)}_{t+1}}
\leq 2^{k^*+1}\eps_0+\tilde O\rpar{\lambda R+R\sqrt{\lambda\sum_{t=q}^s\norm{g_t}}}.
\end{equation*}

Intuitively, suppose we are allowed to predict the improper prediction $\tilde x^{(k^*)}_t$ on $\I^{k^*,i^*}$ that may not comply with the constraint $\V$, and suppose $g^{(k^*)}_t=g_t$. Then, the above result shows that on $\I^{k^*,i^*}$ we have the desirable $\tilde O(\sqrt{|\I^{k^*,i^*}|})$ regret bound. The rest of the proof aims to show that adding predictions from shorter subroutines does not ruin the performance on $\I^{k^*,i^*}$.

\paragraph{Step 3} Analyzing the effect of adding shorter subroutines. \bigskip

The goal of this step is to quantify the difference between
\begin{equation*}
\sum_{t=q}^s\inner{g^{(k^*)}_t}{\tilde x^{(k^*)}_t-x}+\frac{\lambda}{2}\sum_{t=q}^{s-1}\norm{\tilde x^{(k^*)}_t-\tilde x^{(k^*)}_{t+1}},
\end{equation*}
and
\begin{equation*}
\sum_{t=q}^s\inner{g^{(0)}_t}{\tilde x^{(0)}_t-x}+\frac{\lambda}{2}\sum_{t=q}^{s-1}\norm{\tilde x^{(0)}_t-\tilde x^{(0)}_{t+1}}.
\end{equation*}

For all $k\in[0:k^*-1]$, applying the definition of $\tilde x^{(k)}$ and Part 2 of Lemma~\ref{lemma:metaconstraints}, 
\begin{align*}
\inner{g^{(k)}_t}{\tilde x^{(k)}_t-x}&=\inner{g^{(k)}_t}{x^{(k+1)}_t-x}+\inner{g^{(k)}_t}{w^{(k)}_t-0}+\rpar{-\inner{g^{(k)}_t}{x^{(k+1)}_t}}\rpar{z^{(k)}_t-0}\\
&\leq\inner{g^{(k+1)}_t}{\tilde x^{(k+1)}_t-x}+\inner{g^{(k)}_t}{w^{(k)}_t-0}+\rpar{-\inner{g^{(k)}_t}{x^{(k+1)}_t}}\rpar{z^{(k)}_t-0}.
\end{align*}
Similar to Equation (\ref{eq:proofonestep}), 
\begin{equation*}
\norm{\tilde x^{(k)}_t-\tilde x^{(k)}_{t+1}}\leq \norm{\tilde x^{(k+1)}_t-\tilde x^{(k+1)}_{t+1}}+R\abs{z^{(k)}_{t}-z^{(k)}_{t+1}}+\norm{w^{(k)}_t-w^{(k)}_{t+1}}.
\end{equation*}
Therefore, 
\begin{multline*}
\sum_{t=q}^s\inner{g^{(k)}_t}{\tilde x^{(k)}_t-x}+\frac{\lambda}{2}\sum_{t=q}^{s-1}\norm{\tilde x^{(k)}_t-\tilde x^{(k)}_{t+1}}\leq \sum_{t=q}^s\inner{g^{(k+1)}_t}{\tilde x^{(k+1)}_t-x}
+\frac{\lambda}{2}\sum_{t=q}^{s-1}\norm{\tilde x^{(k+1)}_t-\tilde x^{(k+1)}_{t+1}}\\
+\sum_{t=q}^{s}\inner{g^{(k)}_t}{w^{(k)}_t-0}+\frac{\lambda}{2}\sum_{t=q}^{s-1}\norm{w^{(k)}_t-w^{(k)}_{t+1}}\\
+\sum_{t=q}^{s}\rpar{-\inner{g^{(k)}_t}{x^{(k+1)}_t}}\rpar{z^{(k)}_t-0}+\frac{\lambda R}{2}\sum_{t=q}^{s-1}\abs{z^{(k)}_{t}-z^{(k)}_{t+1}}.
\end{multline*}

We next bound the last four sums on the RHS using Theorem~\ref{thm:subroutineoned} and Theorem~\ref{thm:subroutineball}. Let $[a,b]$ be any GC interval of length $2^{k}$ contained in $[q:s]$. Note that by our definition, the subroutines $\A^k_B$ and $\A^k_{1d}$ are initialized at $0$. That is, $w^{(k)}_a=w^{(k)}_{b+1}=0\in\R^d$, $z^{(k)}_a=z^{(k)}_{b+1}=0\in\R$. From Theorem~\ref{thm:subroutineball},
\begin{align*}
\sum_{t=a}^{b}\inner{g^{(k)}_t}{w^{(k)}_t-0}+\frac{\lambda}{2}\sum_{t=a}^{b}\norm{w^{(k)}_t-w^{(k)}_{t+1}}
&=\sum_{t=a}^{b}\inner{g^{(k)}_t}{w^{(k)}_t-0}+\frac{\lambda}{2}\sum_{t=a}^{b-1}\norm{w^{(k)}_t-w^{(k)}_{t+1}}+\frac{\lambda}{2}\norm{w^{(k)}_b}\\
&\leq\sum_{t=a}^{b}\inner{g^{(k)}_t}{w^{(k)}_t-0}+\lambda\sum_{t=a}^{b-1}\norm{w^{(k)}_t-w^{(k)}_{t+1}}
\leq2^{k}\eps_0. 
\end{align*}
Summed over all GC intervals of length $2^k$ contained in $[q:s]$, 
\begin{align*}
\sum_{t=q}^{s}\inner{g^{(k)}_t}{w^{(k)}_t-0}+\frac{\lambda}{2}\sum_{t=q}^{s-1}\norm{w^{(k)}_t-w^{(k)}_{t+1}}&\leq \sum_{t=q}^{s}\inner{g^{(k)}_t}{w^{(k)}_t-0}+\frac{\lambda}{2}\sum_{t=q}^{s}\norm{w^{(k)}_t-w^{(k)}_{t+1}}\\
&\leq 2^{k^*-k}\cdot 2^k\eps_0=2^{k^*}\eps_0.
\end{align*}
Similarly, 
\begin{equation*}
\sum_{t=q}^{s}\rpar{-\inner{g^{(k)}_t}{x^{(k+1)}_t}}\rpar{z^{(k)}_t-0}+\frac{\lambda R}{2}\sum_{t=q}^{s-1}\abs{z^{(k)}_{t}-z^{(k)}_{t+1}}\leq 2^{k^*}\eps_0.
\end{equation*}
Therefore, 
\begin{equation*}
\sum_{t=q}^s\inner{g^{(k)}_t}{\tilde x^{(k)}_t-x}+\frac{\lambda}{2}\sum_{t=q}^{s-1}\norm{\tilde x^{(k)}_t-\tilde x^{(k)}_{t+1}}\leq \sum_{t=q}^s\inner{g^{(k+1)}_t}{\tilde x^{(k+1)}_t-x}+\frac{\lambda}{2}\sum_{t=q}^{s-1}\norm{\tilde x^{(k+1)}_t-\tilde x^{(k+1)}_{t+1}}+2^{k^*+1}\eps_0.
\end{equation*}
Completing the recursion, we have
\begin{align*}
\sum_{t=q}^s\inner{g^{(0)}_t}{\tilde x^{(0)}_t-x}+\frac{\lambda}{2}\sum_{t=q}^{s-1}\norm{\tilde x^{(0)}_t-\tilde x^{(0)}_{t+1}}
&\leq \sum_{t=q}^s\inner{g^{(k^*)}_t}{\tilde x^{(k^*)}_t-x}+\frac{\lambda}{2}\sum_{t=q}^{s-1}\norm{\tilde x^{(k^*)}_t-\tilde x^{(k^*)}_{t+1}}+k^*\cdot2^{k^*+1}\eps_0\\
&\leq (k^*+1)\cdot2^{k^*+1}\eps_0+\tilde O\rpar{\lambda R+R\sqrt{\lambda\sum_{t=q}^s\norm{g_t}}}\\
&\leq\tilde O\rpar{\lambda R+R\sqrt{\lambda\sum_{t=q}^s\norm{g_t}}},
\end{align*}
where the last line follows from $(k^*+1)\cdot2^{k^*+1}\eps_0\leq 2\tilde GR\lceil\log_2(T+1)\rceil$. Plugging this into the result from Step 1, 
\begin{equation*}
\sum_{t=q}^s\spar{l_t(x_{t-H:t})-\tilde l_t(x)}\leq O(RLH^3)+\tilde O\rpar{\lambda R+R\sqrt{\lambda\sum_{t=q}^s\norm{g_t}}}.
\end{equation*}
This bound holds for all GC intervals contained in $[1:T]$. The final step is to extend this property to general intervals, following the classical idea from \citep{daniely2015strongly}. 

\paragraph{Step 4} Extension to general intervals.\bigskip

From Lemma~5 of \citep{daniely2015strongly}, we have the following result: any interval $\I\subset[1:T]$ can be partitioned into two finite sequences of disjoint and consecutive GC intervals, denoted as $(\I_{-k},\ldots,\I_0)$ and $(\I_{1},\ldots,\I_p)$. Moreover, for all $i\geq 1$, $\abs{\I_{-i}}/\abs{\I_{-i+1}}\leq 1/2$; for all $i\geq 2$, $\abs{\I_{i}}/\abs{\I_{i-1}}\leq 1/2$. 

The strongly adaptive regret of our meta-algorithm (Equation~\ref{eq:regret}) over $\I$ can be bounded by the sum of regret over $(\I_{-k},\ldots,\I_0)$ and $(\I_{1},\ldots,\I_p)$. For an index $i$, denote the regret over $\I_i$ as $\reg_i$. Then, \begin{equation*}
\sum_{t\in\I}l_t(x_{t-H:t})-\min_{x\in \V}\sum_{t\in\I}\tilde l_t(x)\leq \sum_{i=0}^k\reg_{-i}+\sum_{i=1}^p\reg_{i},
\end{equation*}
where $k\leq \log_2\abs{\I}$ and $p\leq 1+\log_2\abs{\I}$. Consider the first sum on the RHS,
\begin{align*}
\sum_{i=0}^k\reg_{-i}&\leq (k+1)O(RLH^3)+\sum_{i=0}^k\tilde O\rpar{\lambda R+R\sqrt{\lambda\sum_{t\in\I_{-i}}\norm{g_t}}}\\
&\leq O(RLH^3\log\abs{\I})+\tilde O\rpar{\lambda R+R\sqrt{\lambda\sum_{t\in\I}\norm{g_t}}}.
\end{align*}
The second sum can be bounded similarly. Combining everything completes the proof. 
\end{proof}

\section{Details on adversarial tracking control}\label{section:appendixtracking}

This section presents our results on adversarial tracking. We first prove its reduction to strongly adaptive OCOM. Then, we consider a special case that induces a non-comparative tracking error bound. 

\subsection{Details on the reduction}

We present a few lemmas before proving Theorem~\ref{thm:tracking}. First we bound the norm of state and action. Similar to Section~\ref{section:tracking} we expand the dependence of state on past actions; that is, let $x_t(u_{1:t-1})$ be the state induced by the action sequence $u_{1:t-1}$. Note that $u_{1:t-1}$ is a dummy variable, not necessarily a comparator or the action sequence generated by Algorithm~\ref{algorithm:controller}. 

\begin{lemma}\label{lemma:stateactionbound}
For all $t\geq 1$, with any $u_{1:t-1}$, 
\begin{equation*}
\norms{x_t(u_{1:t-1})},\norm{y_{t}(u_{t-H:t-1})}\leq \gamma^{-1}(\kappa U+W),
\end{equation*}
\begin{equation*}
\norm{x_t(u_{1:t-1})-y_{t}(u_{t-H:t-1})}\leq \gamma^{-1}(\kappa U+W)(1-\gamma)^H.
\end{equation*}
\end{lemma}

\begin{proof}[Proof of Lemma~\ref{lemma:stateactionbound}]
From the evolution of states we have
\begin{equation*}
\norm{x_t(u_{1:t-1})}=\norm{\sum_{i=0}^{t-1}\rpar{\prod_{j=i+1}^{t-1}A_j}\rpar{B_iu_{i}+w_i}}\leq \norm{(\kappa U+W)\sum_{i=0}^{t-1}(1-\gamma)^{t-i-1}}\leq \gamma^{-1}(\kappa U+W).
\end{equation*}
Similarly, 
\begin{equation*}
\norm{y_{t}(u_{t-H:t-1})}=\norm{\sum_{i=t-H}^{t-1}\rpar{\prod_{j=i+1}^{t-1}A_j}\rpar{B_iu_{i}+w_i}}\leq \gamma^{-1}(\kappa U+W). 
\end{equation*}
If $t\leq H$, then $x_t(u_{1:t-1})=y_{t}(u_{t-H:t-1})$. Otherwise, 
\begin{align*}
\norm{x_t(u_{1:t-1})-y_{t}(u_{t-H:t-1})}&=\norm{\sum_{i=0}^{t-H-1}\rpar{\prod_{j=i+1}^{t-1}A_j}\rpar{B_iu_{i}+w_i}}\\
&\leq (\kappa U+W)(1-\gamma)^H\sum_{i=0}^{t-H-1}(1-\gamma)^i\\
&\leq \gamma^{-1}(\kappa U+W)(1-\gamma)^H.\qedhere
\end{align*}
\end{proof}

Next, we characterize the approximation error between $f_t$ and $l^*_t$. This directly follows from the previous lemma and the Lipschitzness of $l^*_t$. 

\begin{lemma}\label{lemma:loss_approx}
For all $t\geq 1$, with any $u_{1:t}$, 
\begin{equation*}
\norm{l^*_t\rpar{x_t(u_{1:t-1}),u_{t}}-f_t(u_{t-H:t})}\leq \gamma^{-1}L^*(\kappa U+W)(1-\gamma)^H.
\end{equation*}
\end{lemma}

Finally, we characterize the Lipschitzness of the ideal loss function $f_t$. 

\begin{lemma}\label{lemma:f_lipschitz}
For all $t\geq 1$ and $h\in[0:H]$, with any $\tilde u_{t-h}$ and $u_{t-H:t}$, 
\begin{equation*}
\left|f_t(u_{t-H:t})-f_t(u_{t-H:t-h-1},\tilde u_{t-h},u_{t-h+1:t})\right|\leq
\kappa L^*\norm{u_{t-h}-\tilde u_{t-h}}.
\end{equation*}
\end{lemma}

\begin{proof}[Proof of Lemma~\ref{lemma:f_lipschitz}]
If $h\neq 0$, we consider the difference in the ideal state. 
\begin{align*}
&\norm{y_{t}(u_{t-H:t-1})-y_t(u_{t-H:t-h-1},\tilde u_{t-h},u_{t-h+1:t-1})}\\
=~&\norm{\rpar{\prod_{j=t-h+1}^{t-1}A_j}B_{t-h}\rpar{u_{t-h}-\tilde u_{t-h}}}\leq \kappa(1-\gamma)^{h-1}\norm{u_{t-h}-\tilde u_{t-h}}.
\end{align*}
Applying the Lipschitzness of $l^*_t$, 
\begin{equation*}
\abs{f_t(u_{t-H:t})-f_t(u_{t-H:t-h-1},\tilde u_{t-h},u_{t-h+1:t})}
\leq \kappa L^*(1-\gamma)^{h-1}\norm{u_{t-h}-\tilde u_{t-h}}.
\end{equation*}
If $h=0$, then directly from the Lipschitzness of $l^*_t$, 
\begin{equation*}
\abs{f_t(u_{t-H:t})-f_t(u_{t-H:t-h-1},\tilde u_{t-h},u_{t-h+1:t})}\leq L^*\norm{u_{t-h}-\tilde u_{t-h}}.
\end{equation*}
Combining the above completes the proof. 
\end{proof}

\begin{lemma}\label{lemma:ftilde_lipschitz}
For all $t$, let $\tilde f_t(u)=f_t(u,\ldots,u)$. Then, for all $u,\tilde u\in\ball^{d_u}(0,U)$, 
\begin{equation*}
\abs{\tilde f_t(u)-\tilde f_t(\tilde u)}\leq 2\kappa\gamma^{-1}L^*\norm{u-\tilde u}.
\end{equation*}
\end{lemma}

\begin{proof}[Proof of Lemma~\ref{lemma:ftilde_lipschitz}]
For conciseness, let $\tilde y_t(u)=y_t(u,\ldots,u)$. Then, 
\begin{equation*}
\norm{\tilde y_t(u)-\tilde y_t(\tilde u)}=\norm{\sum_{i=t-H}^{t-1}\rpar{\prod_{j=i+1}^{t-1}A_j}B_i\rpar{u-\tilde u}}\leq \kappa\norm{u-\tilde u}\sum_{i=t-H}^{t-1}(1-\gamma)^{t-i-1}
\leq \kappa\gamma^{-1}\norm{u-\tilde u}.
\end{equation*}
The result follows from the Lipschitzness of $l^*_t$. 
\end{proof}

Now we are ready to prove Theorem~\ref{thm:tracking}. 

\tracking*

\begin{proof}[Proof of Theorem~\ref{thm:tracking}]
For all $u^C_{1:T}\in \mathcal{C}_\I$,
\begin{align*}
&\sum_{t=a}^bl^*_t\rpar{x_t(u^A_{1:t-1}),u^A_{t}}-\sum_{t=a}^bl^*_t\rpar{x_t(u^C_{1:t-1}),u^C_{t}}\\
=~&\sum_{t=a}^b\spar{l^*_t\rpar{x_t(u^A_{1:t-1}),u^A_{t}}-f_t(u^A_{t-H:t})}+\sum_{t=a}^b\spar{f_t(u^A_{t-H:t})-f_t(u^C_{t-H:t})}+\sum_{t=a}^b\spar{f_t(u^C_{t-H:t})-l^*_t\rpar{x_t(u^C_{1:t-1}),u^C_{t}}}\\
\leq~&2\gamma^{-1}L^*(\kappa U+W)(1-\gamma)^H\abs{\I}+\sum_{t=a}^b\spar{f_t(u^A_{t-H:t})-f_t(u^C_{b},\ldots,u^C_{b})},
\end{align*}
where the last inequality is due to Lemma~\ref{lemma:loss_approx}. From our choice of $H$, we have $(1-\gamma)^H\abs{\I}\leq 1$. Therefore, the first term on the RHS is a constant that can be neglected in our result. The second term on the RHS is upper-bounded by the strongly adaptive regret on $f_t$. 
\end{proof}

\subsection{Non-comparative tracking error bound}

In the following we consider the example from Section~\ref{section:tracking}. The comparative regret guarantee from Theorem~\ref{thm:tracking} translates to a non-comparative tracking error bound. 

\special*

\begin{proof}[Proof of Corollary~\ref{thm:special}]
We start by characterizing the power of the comparator class (Definition~\ref{definition:comparator}). From Example~\ref{example}, there exists some $u^*$ such that $x^*_\I=B_\I u^*$. Consider the comparator sequence $u^C_{1:T}\in\mathcal{C}_\I$ such that $u^C_t=u^*$ for all $t\in[1:T]$. From the system equation, for all $t\in[a:b]$,
\begin{align*}
x_{t+1}(u^C_{1:t})&=A_tx_{t}(u^C_{1:t-1})+B_t u^*+w_t\\
&=A_t\spar{x_{t}(u^C_{1:t-1})-(I-A_t)^{-1}B_tu^*}+(I-A_t)^{-1}B_tu^*+w_t\\
&=A_t\spar{x_{t}(u^C_{1:t-1})-x^*_\I}+x^*_\I+w_t.
\end{align*}
Rearranging the terms and applying norms on both sides, 
\begin{equation*}
\norm{x_{t+1}(u^C_{1:t})-x^*_\I}\leq (1-\gamma)\norm{x_{t}(u^C_{1:t-1})-x^*_\I}+W.
\end{equation*}
From Lemma~\ref{lemma:stateactionbound}, 
\begin{equation*}
\norm{x_{a}(u^C_{1:a-1})-x^*_\I}\leq \norm{x_{a}(u^C_{1:a-1})}+\norm{x^*_\I}\leq \gamma^{-1}(\kappa U+W)+\kappa U\norm{(I-A_t)^{-1}}\leq \gamma^{-1}(2\kappa U+W).
\end{equation*}
Following a recursion, for all $t\in[a:b]$,
\begin{align*}
\norm{x_{t}(u^C_{1:t-1})-x^*_\I}&\leq \gamma^{-1}(2\kappa U+W)(1-\gamma)^{t-a}+W\sum_{i=0}^{t-a}(1-\gamma)^i\\
&\leq \gamma^{-1}W+\gamma^{-1}(2\kappa U+W)(1-\gamma)^{t-a}.
\end{align*}
\begin{align}
\nonumber\frac{1}{t-a+1}\sum_{i=a}^t\norm{x_{i}(u^C_{1:i-1})-x^*_\I}&\leq \gamma^{-1}W+\gamma^{-1}(2\kappa U+W)\cdot \frac{1}{t-a+1}\sum_{i=0}^{t-a}(1-\gamma)^i\\
&\leq \gamma^{-1}W+\gamma^{-2}(2\kappa U+W)(t-a+1)^{-1}.\label{eq:comparatorspecial}
\end{align}

Next, consider the regret of Algorithm~\ref{algorithm:controller}. Applying Theorem~\ref{thm:tracking} on all time intervals $[a:t]$ with $t\in[a:b]$, we have
\begin{equation*}
\sum_{i=a}^{t}\norm{x_i(u^A_{1:i-1})-x^*_\I}-\sum_{i=a}^{t}\norm{x_i(u^C_{1:i-1})-x^*_\I}=\tilde O\rpar{\sqrt{t-a+1}},
\end{equation*}
where $\tilde O(\cdot)$ subsumes polynomial factors on problem constants and poly-logarithmic factors on $T$. Normalizing on both sides, 
\begin{equation}
\frac{1}{t-a+1}\rpar{\sum_{i=a}^{t}\norm{x_i(u^A_{1:i-1})-x^*_\I}-\sum_{i=a}^{t}\norm{x_i(u^C_{1:i-1})-x^*_\I}}=\tilde O\rpar{(t-a+1)^{-1/2}}.\label{eq:regretspecial}
\end{equation}
Combining (\ref{eq:comparatorspecial}) and (\ref{eq:regretspecial}) completes the proof. 
\end{proof}

\section{Experiments}\label{section:experiments}

In this section we test the proposed approach on three separate levels: (i) One-dimensional movement-aware OLO (Algorithm~\ref{algorithm:1d}); (ii) Strongly adaptive OCOM (Algorithm~\ref{algorithm:meta}); (iii) Adversarial tracking control (Algorithm~\ref{algorithm:controller}). 

\subsection{One-dimensional movement-aware OLO}\label{subsection:experiment1d}

First, we test our one-dimensional movement-aware OLO algorithm (Algorithm~\ref{algorithm:1d}). The domain is the interval $[0,R]\subset\R$, and the loss functions are defined as $l_t(x)=|x-x^*|$, where $x^*$ is a fixed ``target''. Throughout this experiment, we set hyperparameters $\gamma=0$, $\eps=1$ and $G=1$ since the loss functions are 1-Lipschitz. 

\begin{figure}[ht]
\vspace{-10pt}
\centering
\begin{subfigure}{.49\textwidth}
  \centering
  \includegraphics[width=\textwidth]{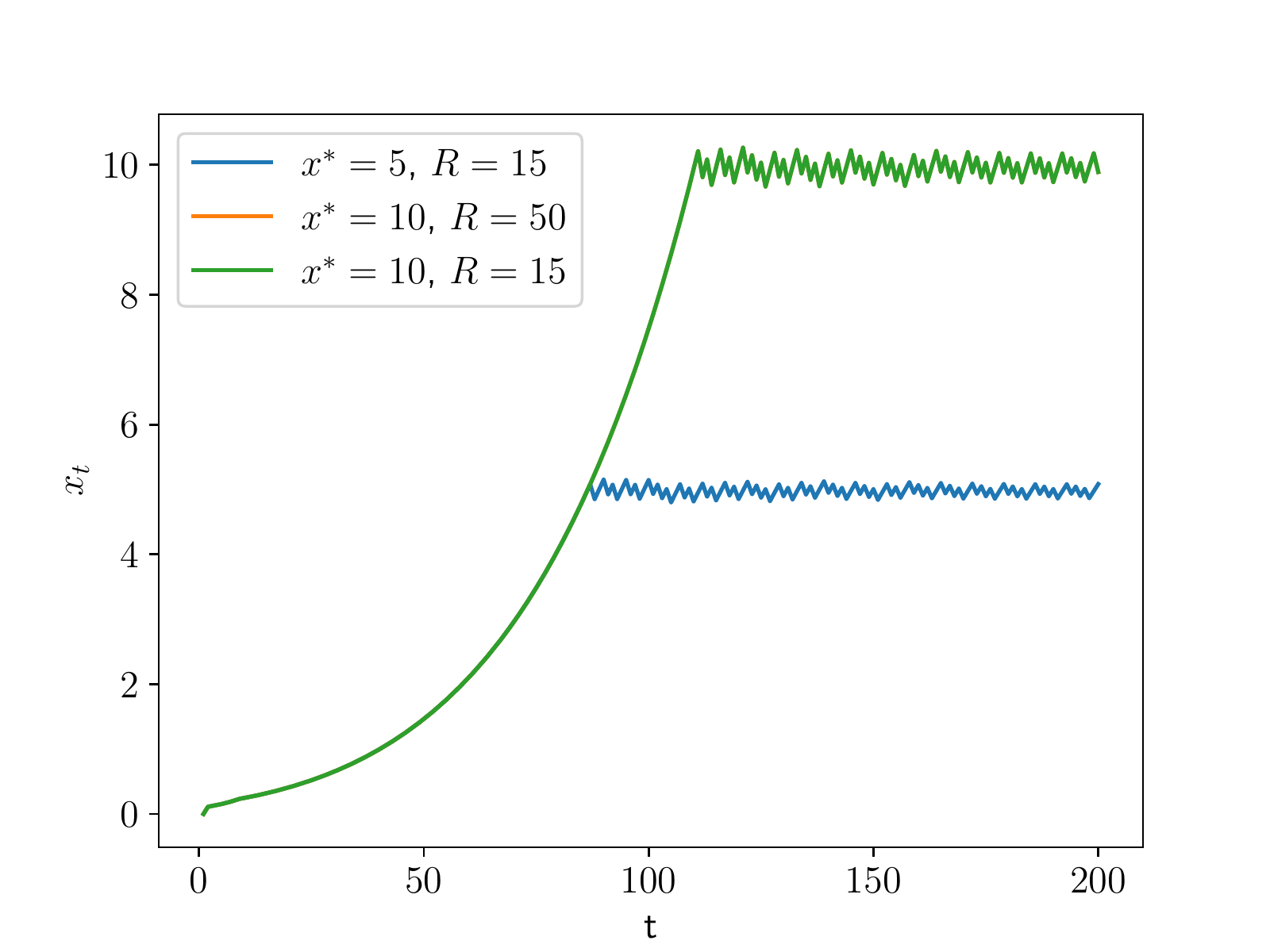}
  \caption{Varying $x^*$ and $R$.}
  \label{figure:exp1_fig1}
\end{subfigure}%
\hfill
\begin{subfigure}{.49\textwidth}
  \centering
  \includegraphics[width=\textwidth]{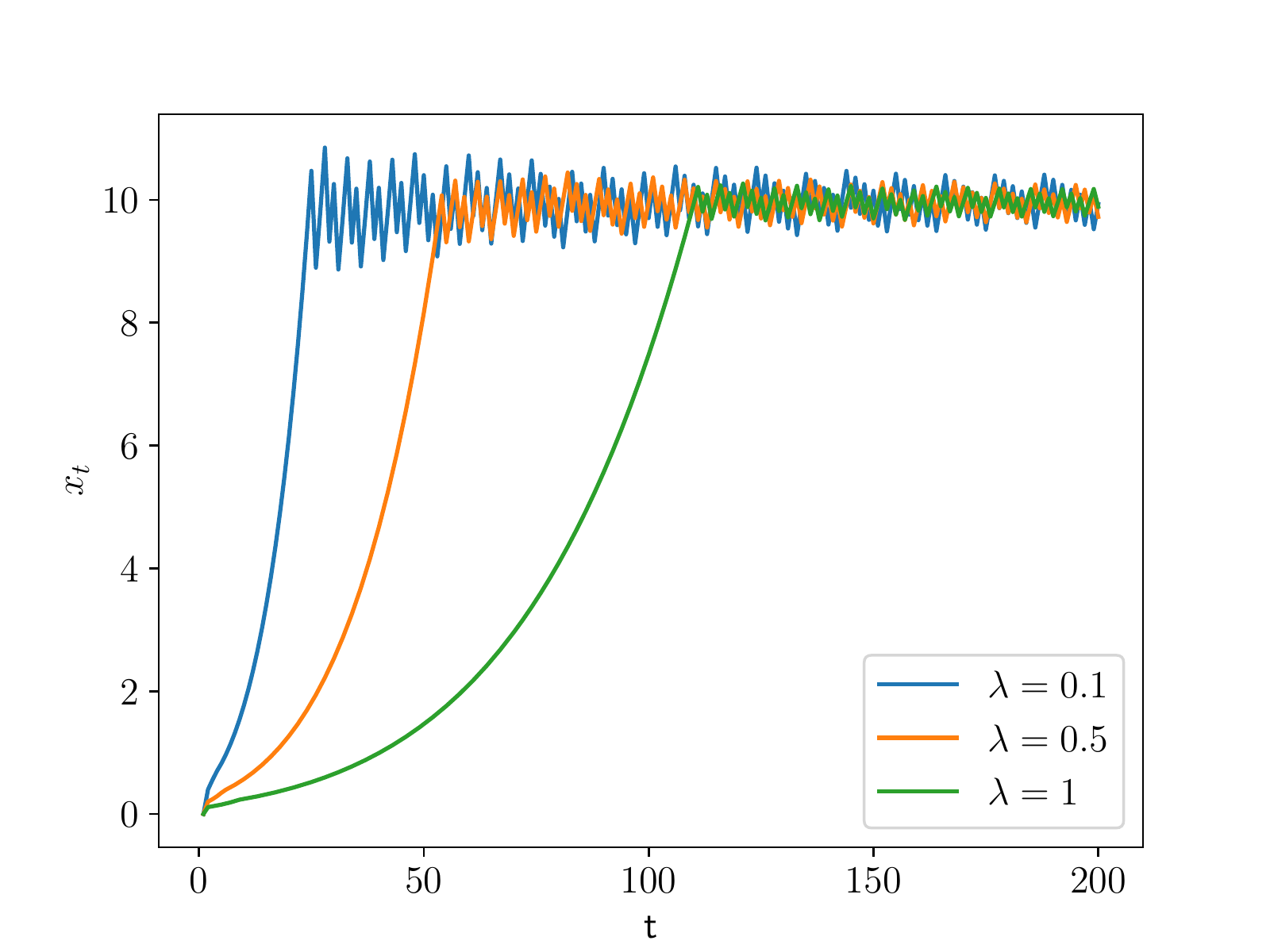}
  \caption{Varying $\lambda$.}
  \label{figure:exp1_fig2}
\end{subfigure}
\caption{Experiments on Algorithm~\ref{algorithm:1d}.}
\end{figure}

In Figure~\ref{figure:exp1_fig1} we vary (i) the target $x^*$; and (ii) the size of the domain $R$. 
\begin{enumerate}
\item Consider the green line and the orange line (which is completely covered by the former). In this case, increasing the size of the domain leaves the performance of the algorithm unchanged. This is different from standard Online Gradient Descent (OGD) where the correct learning rate depends on the size of the domain. 
\item Consider the blue line and the green line. Starting from the origin, the predictions of the algorithm approach the target with \emph{exponentially increasing speed}, without knowing the target in advance. This is also different from OGD, where the speed of approaching the target is constant (with constant learning rate) or decreasing (with time-varying learning rate). 
\end{enumerate}

In general, Algorithm~\ref{algorithm:1d} exhibits the advantage of \emph{parameter-free} online learning algorithms: the algorithm works well without depending on the optimal comparator norm ($\norms{x^*}$) or its (possibly very loose) upper bound $R$, hence requiring less tuning than OGD.

In Figure~\ref{figure:exp1_fig2} we vary $\lambda$, the weight of movement costs. Practically, it yields another ``degree of freedom'' (beside $\eps$) for tuning the algorithm's transient response. Larger $\lambda$ means larger weight on movement costs: the algorithm moves slower initially, but has less fluctuation around the target. 

\subsection{Strongly adaptive OCOM}
Next, we test our strongly adaptive OCOM algorithm (Algorithm~\ref{algorithm:meta}). For easier visualization, we set the domain as $\V=[-5,5]\subset\R$. Let the memory constant $H=5$. With a time-varying target $x^*_t$, we define the loss functions as
\begin{equation*}
l_t(x_{t-H},\ldots,x_{t})=\sum_{h=0}^H\norm{x_{t-h}-x^*_t}. 
\end{equation*}
Note that the Lipschitz constants can be chosen as $L=1$ and $\tilde G=H+1$. 

Our theoretical result requires $\eps_0=O(T^{-1})$. Although asymptotically this is correct, in practice such a small $\eps_0$ makes the algorithm too conservative at the beginning. In other words, it can take a long time for the algorithm to \emph{warm up}. Therefore, we set $\eps_0=1$ in our experiments. 

\begin{figure}[ht]
\vspace{-10pt}
\centering
\begin{subfigure}{.49\textwidth}
  \centering
  \includegraphics[width=\textwidth]{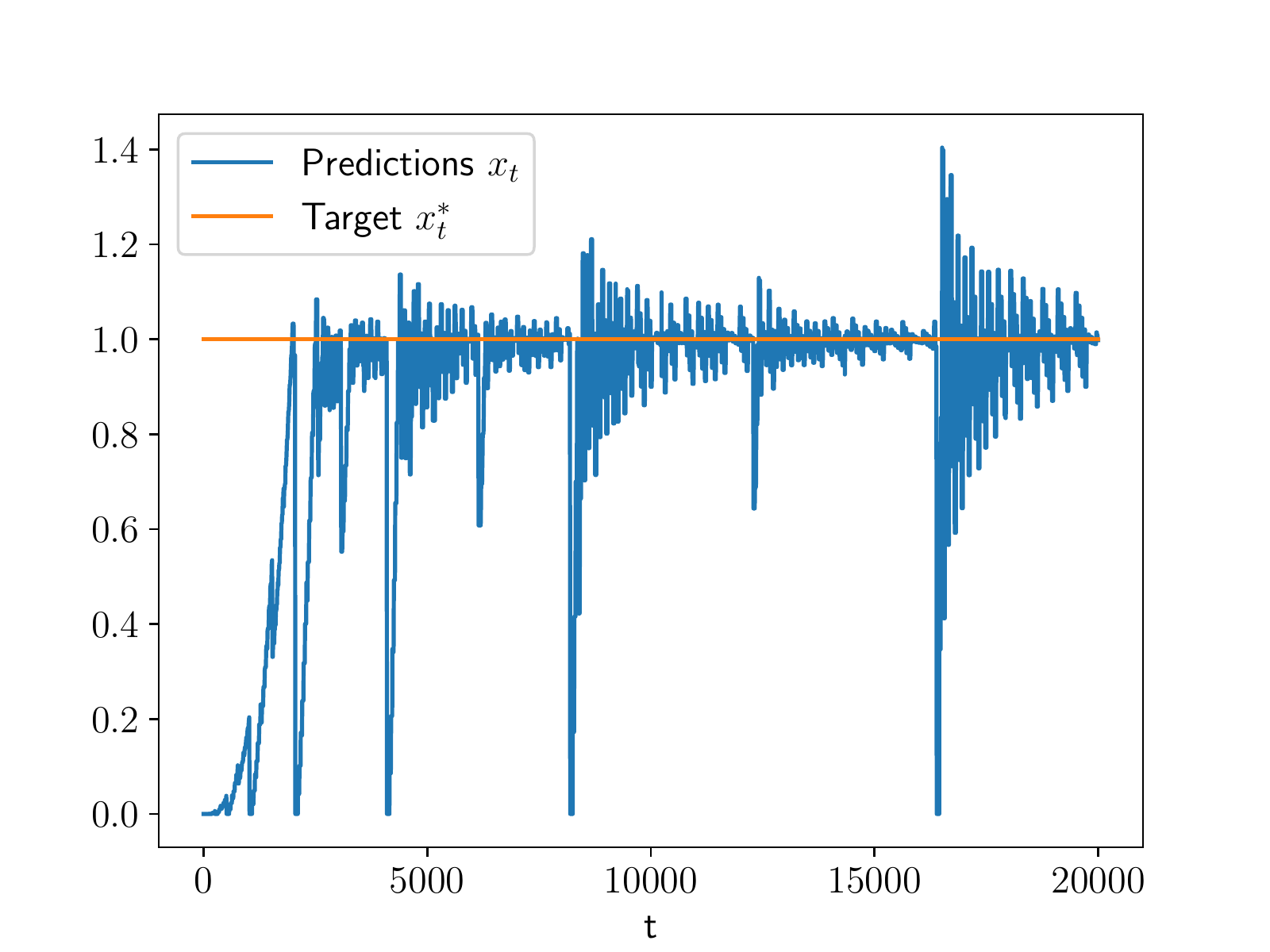}
  \caption{$x^*_t$ is a step signal.}
  \label{figure:exp2_fig1}
\end{subfigure}%
\hfill
\begin{subfigure}{.49\textwidth}
  \centering
  \includegraphics[width=\textwidth]{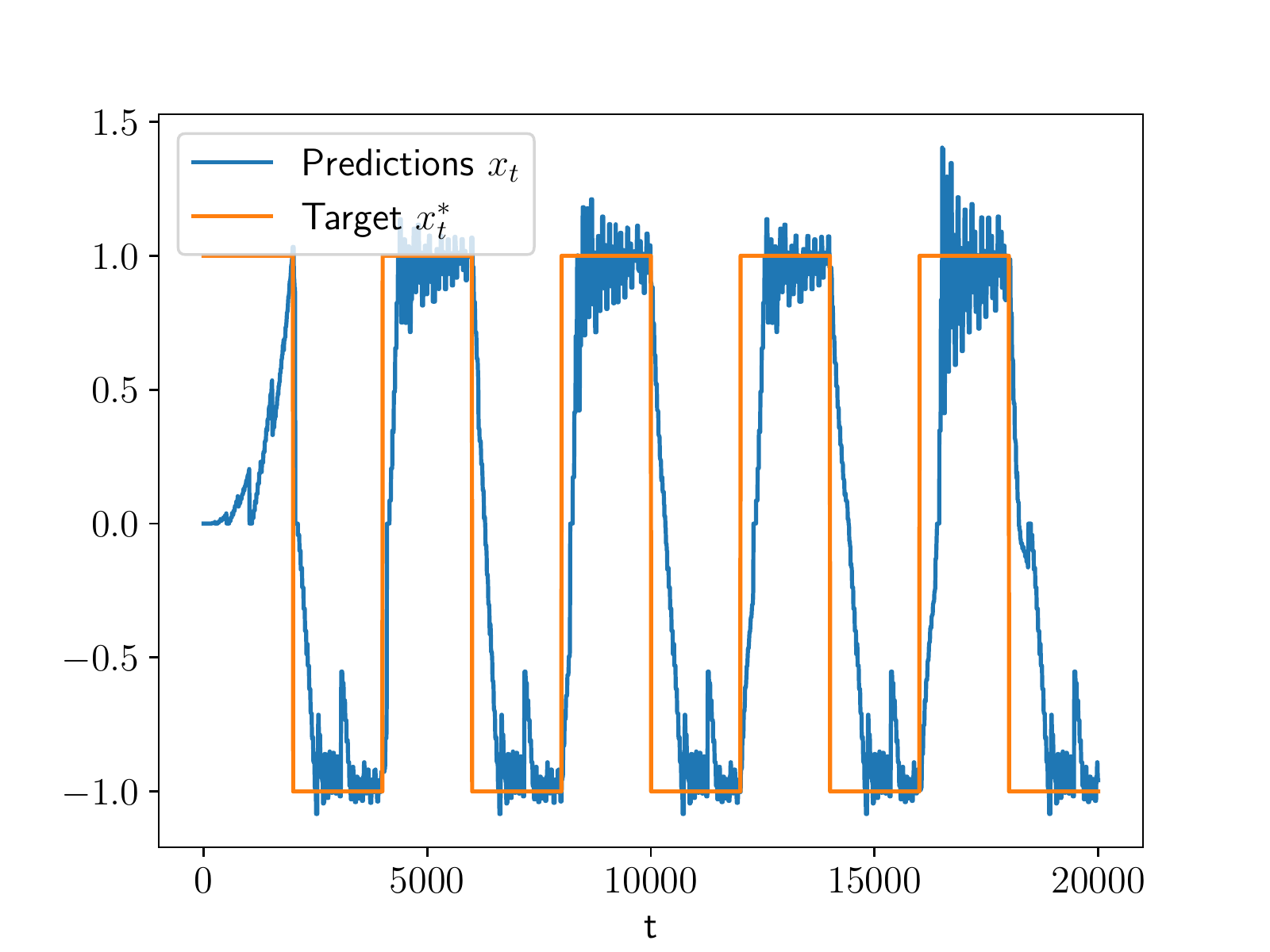}
  \caption{$x^*_t$ is a square wave.}
  \label{figure:exp2_fig2}
\end{subfigure}
\caption{Experiments on Algorithm~\ref{algorithm:meta}.}
\label{figure:alg7}
\end{figure}

We plot the result of the experiments in Figure~\ref{figure:alg7}. On the left, the target is a step signal, $x^*_t=1$. On the right, the target $x^*_t$ is a square wave with period 4000. Several observations can be made: 

\begin{enumerate}
\item In the warm-up phase of the algorithm (the first 2000 rounds), similar to the previous subsection, Algorithm~\ref{algorithm:meta} approaches the fixed target with increasing speed (if the ``dips'' are ignored). 
\item Once the predictions reach the vicinity of the fixed target, they do not monotonically converge to the target as in standard online learning algorithms. Instead, the predictions fluctuate around the target, in a pattern determined by GC intervals. The rationale of this behavior is that, Algorithm~\ref{algorithm:meta} does not \emph{know} or \emph{assume} the target is fixed; to quickly adapt to possible sudden changes of the target, the algorithm regularly forgets the past and re-explores. 
\item There is a practical issue not captured in our analysis. Every time a GC interval of a new length becomes active ($t=2^n$ for some $n\in\N$), all the subroutines are reinitialized. Consequently, Algorithm~\ref{algorithm:meta} completely forgets all the received information and restarts from the origin (since the first output of the subroutine is always at the origin). This can cause large ``dips'' in the prediction sequence (e.g, $t\approx 8000$ and $t\approx 16000$ in Figure~\ref{figure:exp2_fig1}). Such a behavior is undesirable if smooth predictions are preferred, but when the target regularly moves around the origin (Figure~\ref{figure:exp2_fig2}) this can be acceptable.  
\end{enumerate}

\subsection{A shifted version of our OCOM algorithm}

To make the prediction sequence smoother, we also test a modified version of our strongly adaptive OCOM algorithm. The idea is simple: we incorporate a shifting procedure in the subroutines. Whenever a Subroutine-ball is reinitialized, its first prediction is set as the last prediction of the previous subroutine (before re-initialization). In this way, the meta-algorithm experiences less fluctuation due to the activation and deactivation of GC intervals. 

\begin{algorithm}[ht]
\caption{A shifted version of Algorithm~\ref{algorithm:higherd}. \label{algorithm:higherdshifted}}
\begin{algorithmic}[1]
\REQUIRE Hyperparameters $(\lambda,\eps,G)$ with $\lambda\geq 0$ and $\eps,G>0$; $g_1, g_2,\ldots\in\R^{d}$ with $\norms{g_t}\leq G$, $\forall t$; a shift vector $v\in\R^d$. 
\STATE Define $\A_{r}$ as Algorithm~\ref{algorithm:1d} on the domain $[0,R+\norms{v}]$, with hyperparameters $(\lambda,\lambda,\eps,G)$.
\STATE Define $\A_B$ as \emph{Online Gradient Descent} (OGD) on $\ball^{d}(0,1)$ with learning rate $\eta_t=1/(G\sqrt{t})$, initialized at $0$.
\FOR{$t=1,2,\ldots$}
\STATE Obtain $y_t\in\R$ from $\A_{r}$ and $z_t\in\R^d$ from $\A_B$. Predict $x_t=v+y_tz_t\in\R^d$, observe $g_t\in\R^d$.

\STATE Return $\langle g_t,z_t\rangle$ and $g_t$ as the $t$-th loss subgradient to $\A_{r}$ and $\A_B$, respectively.
\ENDFOR
\end{algorithmic}
\end{algorithm}

\begin{figure}[!ht]
\vspace{-10pt}
\centering
\begin{subfigure}{.49\textwidth}
  \centering
  \includegraphics[width=\textwidth]{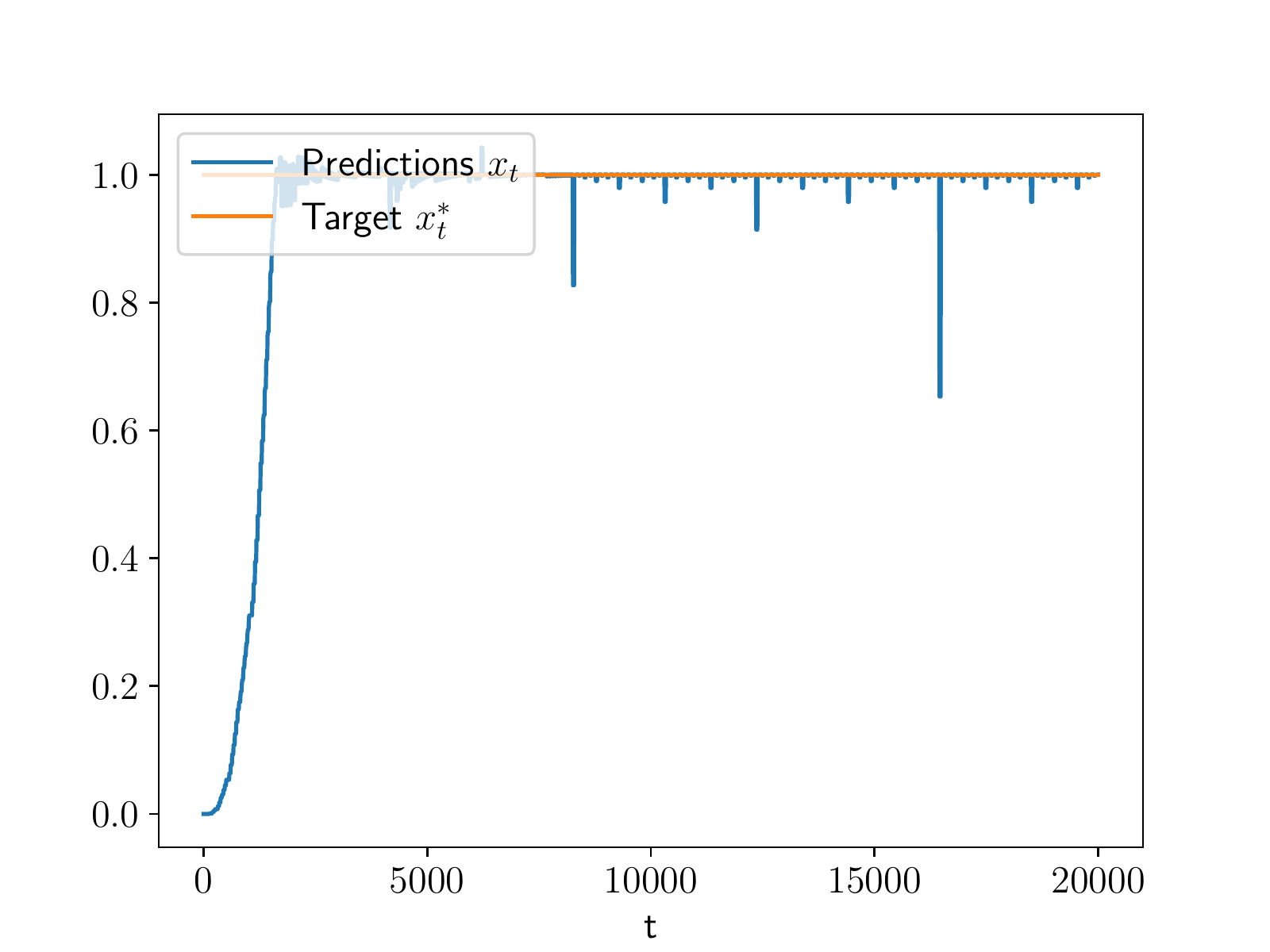}
  \caption{$x^*_t$ is a step signal.}
  \label{figure:exp3_fig1}
\end{subfigure}%
\hfill
\begin{subfigure}{.49\textwidth}
  \centering
  \includegraphics[width=\textwidth]{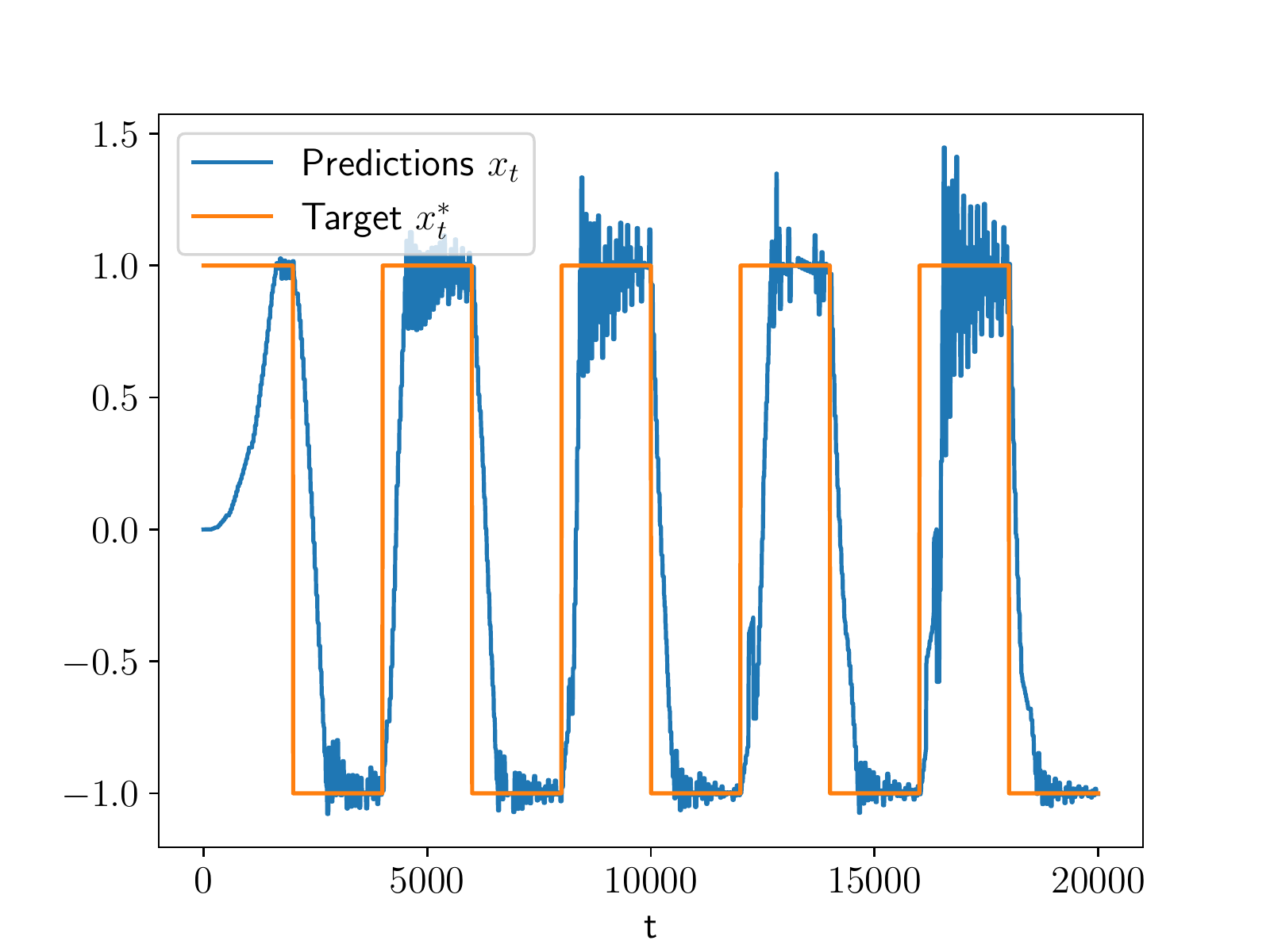}
  \caption{$x^*_t$ is a square wave.}
  \label{figure:exp3_fig2}
\end{subfigure}

\begin{subfigure}{.49\textwidth}
  \centering
  \includegraphics[width=\textwidth]{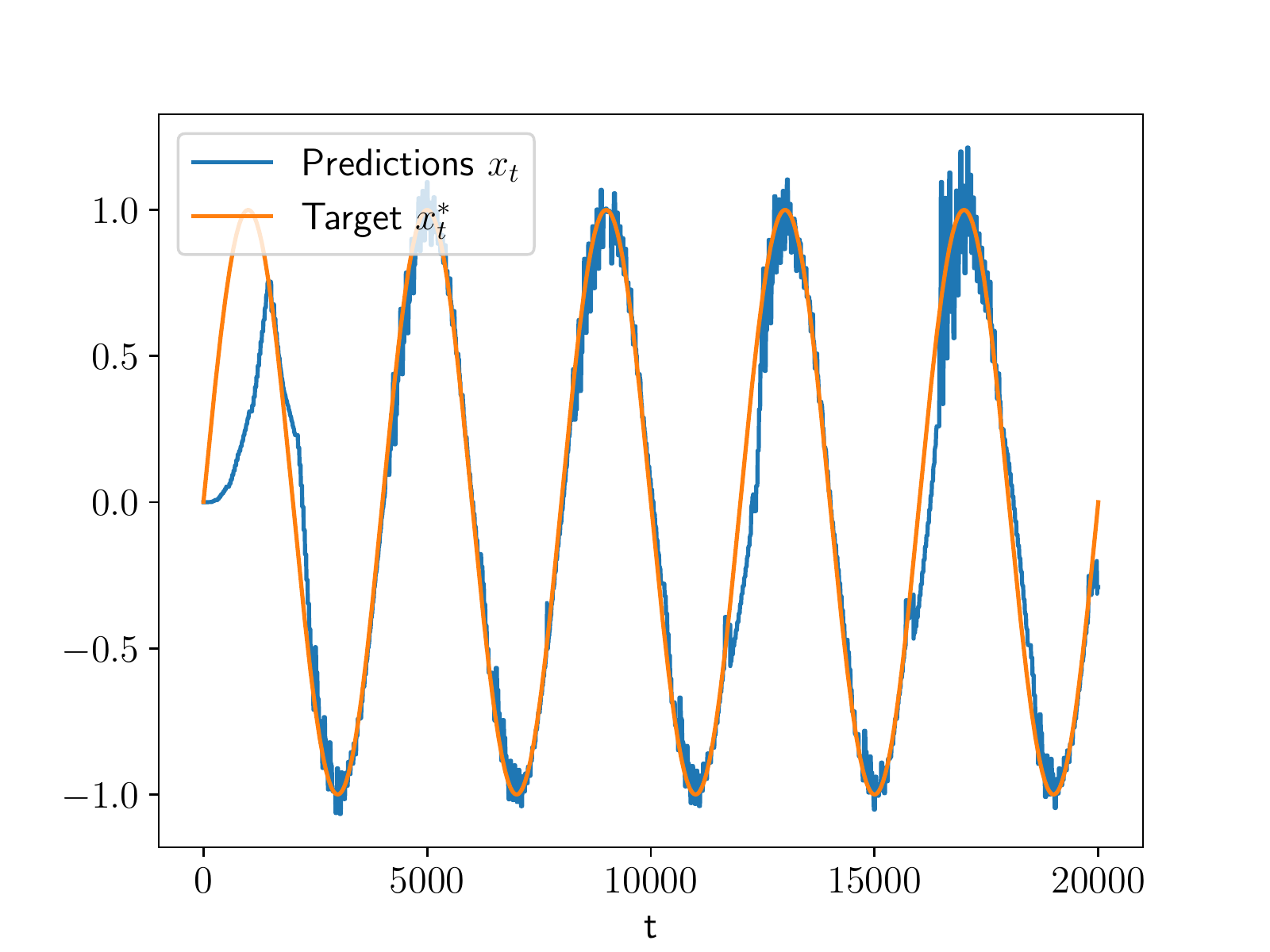}
  \caption{$x^*_t$ is a sinusoidal wave.}
  \label{figure:exp3_fig3}
\end{subfigure}
\hfill
\begin{subfigure}{.49\textwidth}
  \centering
  \includegraphics[width=\textwidth]{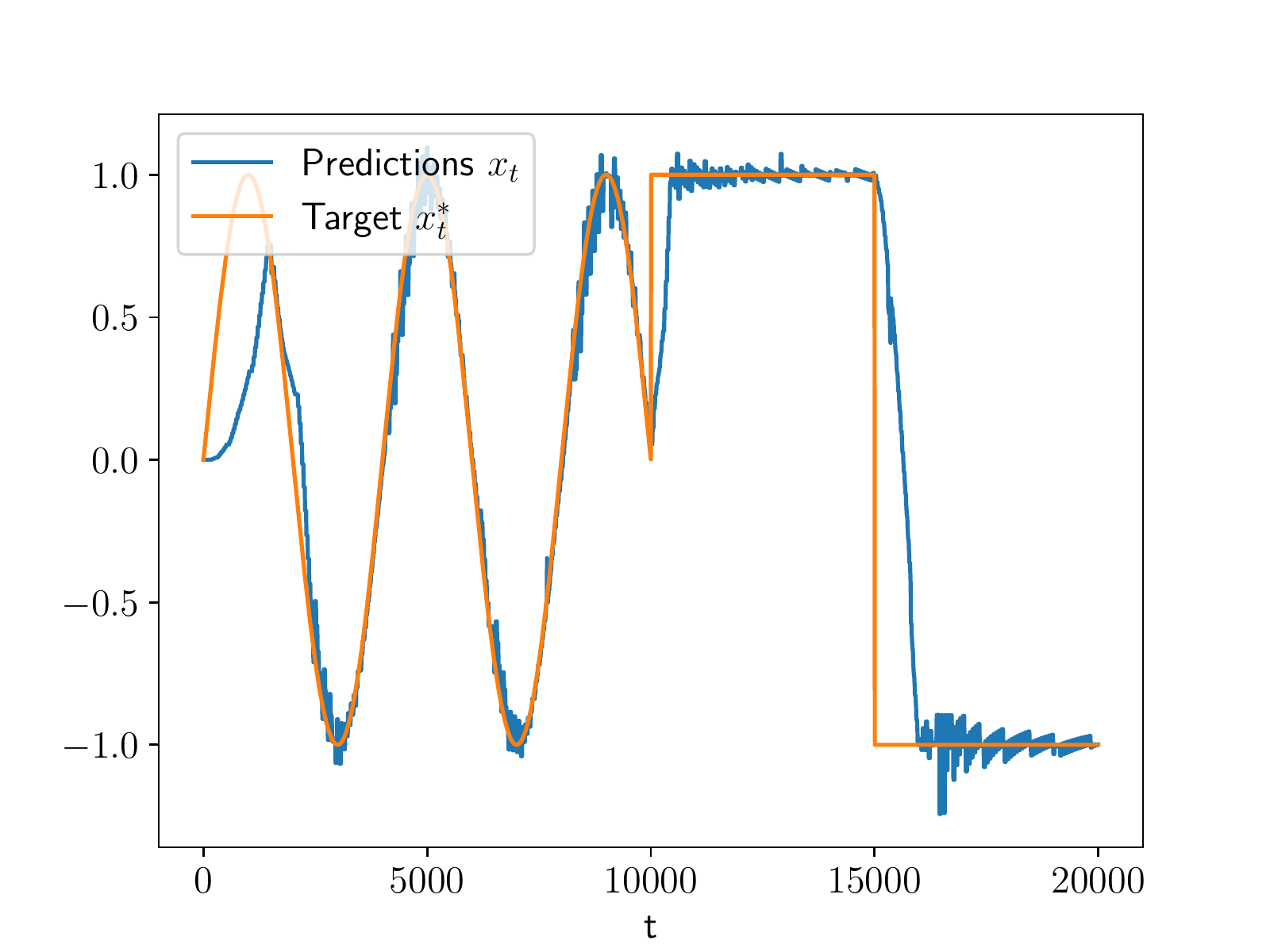}
  \caption{$x^*_t$ is a composite signal.}
  \label{figure:exp3_fig4}
\end{subfigure}
\caption{Experiments on the shifted OCOM algorithm.}
\label{figure:alg7shifted}
\end{figure}

Concretely, we first present a shifted version of Algorithm~\ref{algorithm:higherd} (high dimensional movement-aware OLO) as Algorithm~\ref{algorithm:higherdshifted}. Given a shift vector $v\in\R^d$, Algorithm~\ref{algorithm:higherdshifted} starts from predicting $v$, and all the predictions are within a larger norm ball $\ball^d(v,R+\norms{v})$ centered at $v$. Using Algorithm~\ref{algorithm:higherdshifted} as the base algorithm of Subroutine-ball, we obtain a shifted version of the latter. When using this shifted Subroutine-ball in the meta-algorithm (Algorithm~\ref{algorithm:meta}),
\begin{enumerate}
\item All the Subroutine-ball on GC intervals with indices $(k,1)$ are initialized with shift vector $v=0$. For example, at the beginning of the 2nd round, the meta-algorithm initializes $A^1_B$ with shift vector $v=0$. 
\item At the beginning of the $2^k i$-th round (with $i>1$), when reinitializing $\A^k_B$, the shift vector $v$ is set as the last prediction of the previous $\A^k_B$. For example, on the GC interval $[2:3]$ the meta-algorithm employs a Subroutine-ball $\A^1_B$. At the beginning of the 4th round, the meta-algorithm queries $A^1_B$ and assigns its prediction to a vector $v$. Then, $A^1_B$ is reinitialized with shift vector $v$. 
\end{enumerate}

Empirical results for this \emph{shifted OCOM algorithm} are presented in Figure~\ref{figure:alg7shifted}. Specifically, the targets $x^*_t$ in Figure~\ref{figure:exp3_fig1} and \ref{figure:exp3_fig2} are the same as in Figure~\ref{figure:exp2_fig1} and \ref{figure:exp2_fig2}. In Figure~\ref{figure:exp3_fig3}, $x^*_t$ is a sinusoidal wave with period 4000. In Figure~\ref{figure:exp3_fig4}, $x^*_t$ is the concatenation of a sinusoidal wave and a square wave: 
\begin{equation*}
x^*_t=\begin{cases}
\sin(\pi t/2000), &\textrm{if~}t<T/2,\\
1, &\textrm{if~}T/2\leq t< 3T/4,\\
-1, &\textrm{otherwise}. 
\end{cases}
\end{equation*}
In general, the shifted version of Algorithm~\ref{algorithm:meta} tracks the target quite well, even when the target exhibits large, sudden changes. (The ``tracking'' here refers to the concept in online learning, not linear control.) Especially, the prediction sequence exhibits less fluctuation due to the reset of GC intervals. 

\subsection{Adversarial tracking}

Finally, we test our adversarial tracking controller (Algorithm~\ref{algorithm:controller}). We consider two cases: (i) $d_x=1$; (ii) $d_x=2$. 

\paragraph{One-dimensional control}Starting from one-dimensional control, let $d_x=d_u=1$, $U=5$. The dynamics are time-varying: for all $t$, $A_t=0.55+0.05\cdot\sin(\pi t/10000)$; $B_t=0.95+0.05\cdot\sin(\pi t/5000)$. Therefore, $\kappa=1$ and $\gamma=0.4$. Further, we define the disturbances as $w_t=0.05\cdot\sin(\pi t/4000)$, $\forall t\in\N$. 

The loss functions are $l^*_t(x,u)=\norms{x-x^*_t}$, where $x^*_t$ is the adversarial reference trajectory. It is globally 1-Lipschitz, therefore $L^*=1$. 

We use the shifted version of Algorithm~\ref{algorithm:meta} as the base algorithm of our controller. Similar to the previous subsection, we set the hyperparameter as $\eps_0=0.5$. Following the procedure in Algorithm~\ref{algorithm:controller}, we set the problem constants in OCOM as: $\V\leftarrow\ball^1(0,5)$, $R\leftarrow 5$, $L\leftarrow 1$ and $\tilde G\leftarrow 5$. There is one exception: the memory $H$ defined in Algorithm~\ref{algorithm:controller} is conservative. In our experiment, we treat $H$ as a hyperparameter; specifically for the one-dimensional control experiment, we set $H=8$. Intuitively, the choice of $H$ trades off the responsiveness of the controller and its steady-state error. 

\begin{figure}[!ht]
\vspace{-10pt}
\centering
\begin{subfigure}{.49\textwidth}
  \centering
  \includegraphics[width=\textwidth]{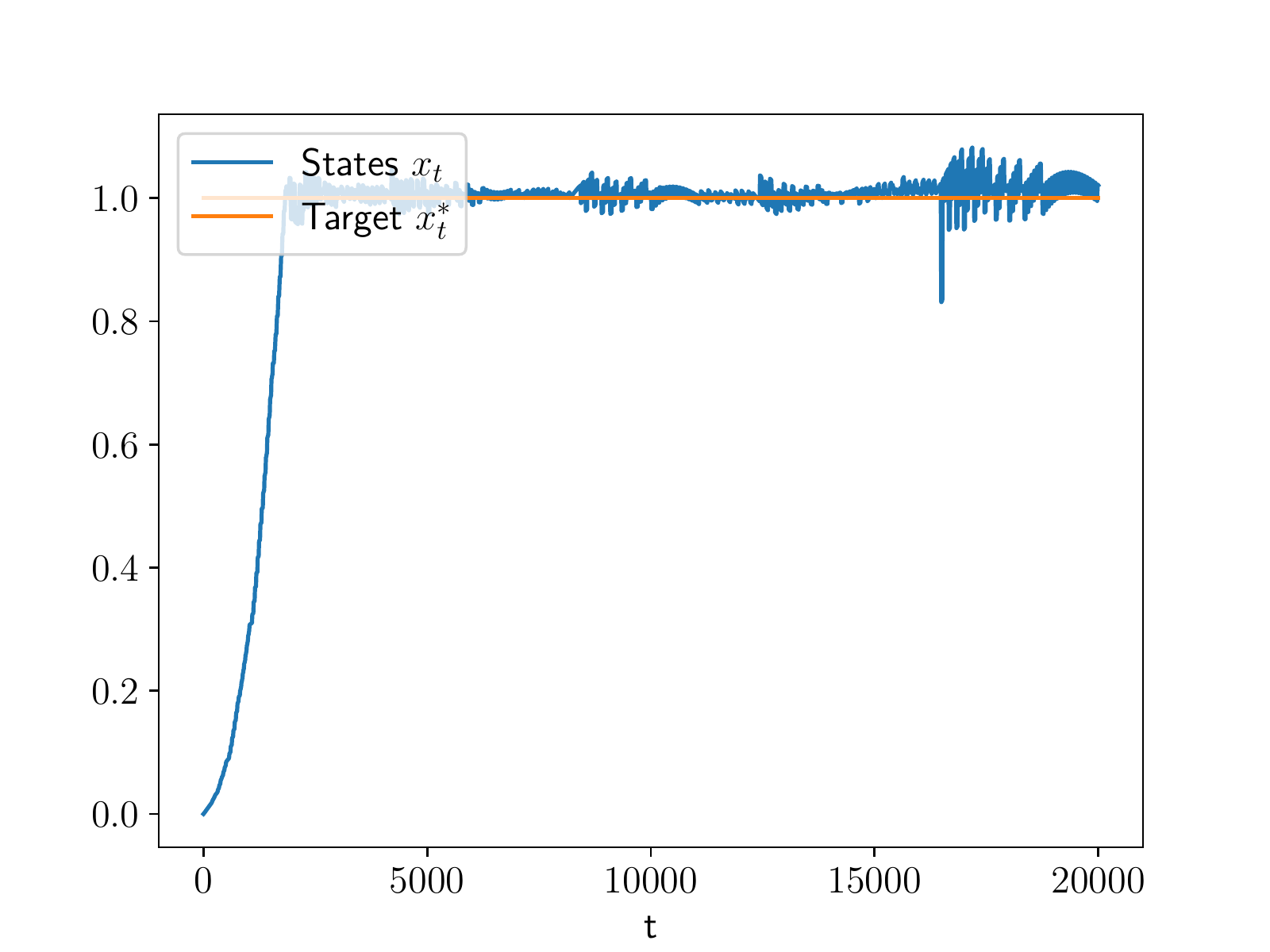}
  \caption{$x^*_t$ is a step signal.}
  \label{figure:exp5_fig1}
\end{subfigure}%
\hfill
\begin{subfigure}{.49\textwidth}
  \centering
  \includegraphics[width=\textwidth]{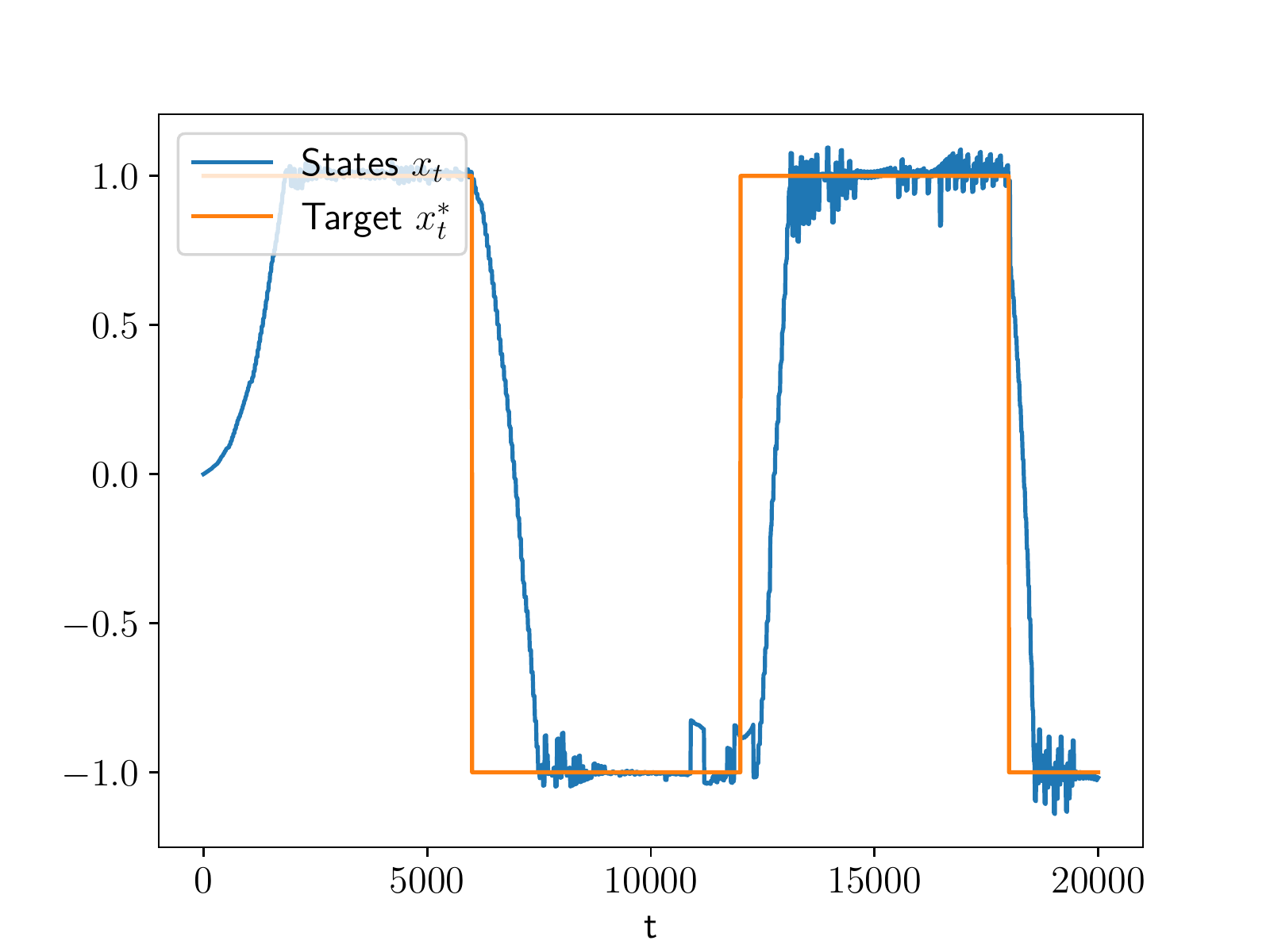}
  \caption{$x^*_t$ is a square wave.}
  \label{figure:exp5_fig2}
\end{subfigure}

\begin{subfigure}{.49\textwidth}
  \centering
  \includegraphics[width=\textwidth]{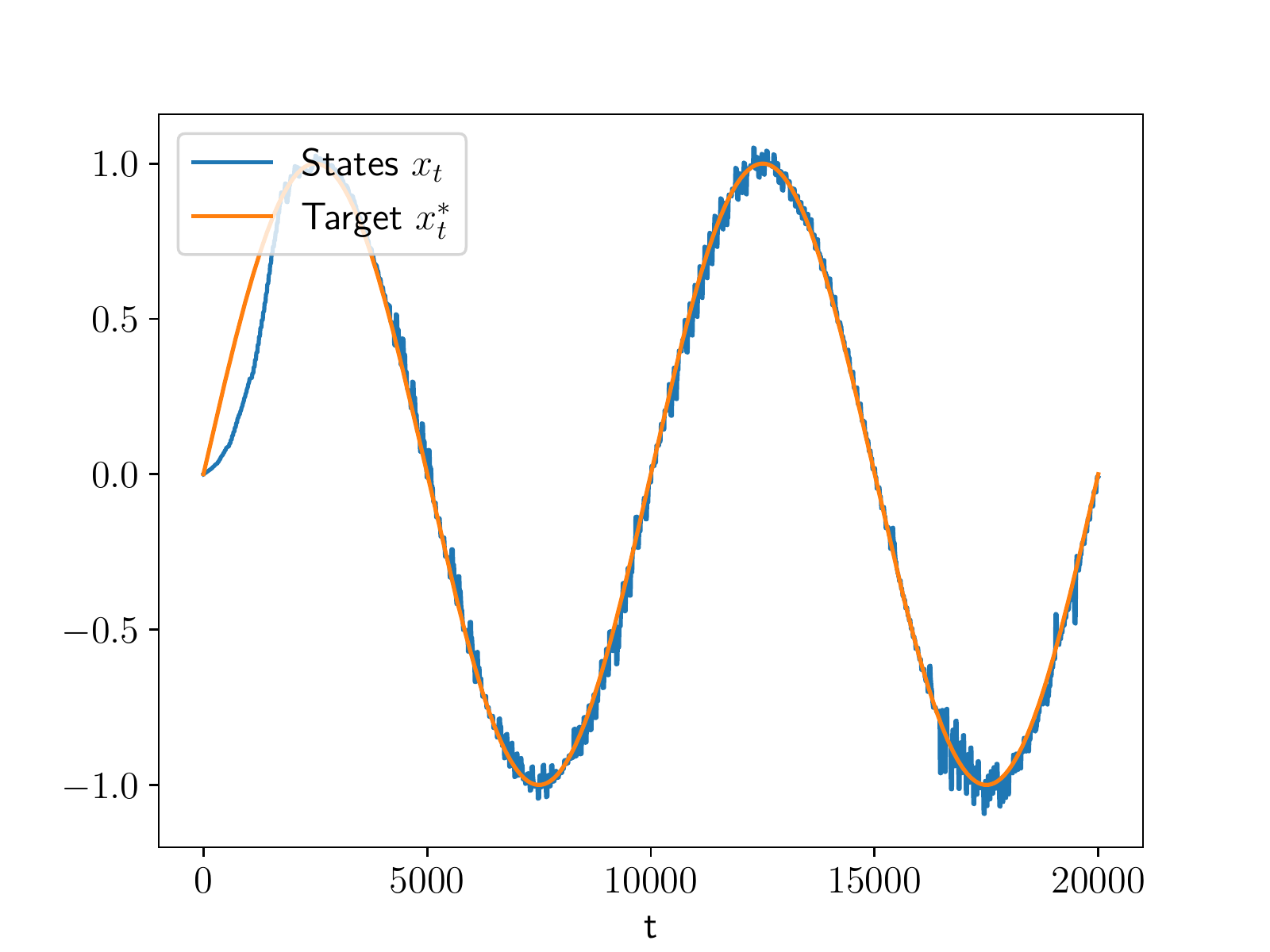}
  \caption{$x^*_t$ is a sinusoidal wave.}
  \label{figure:exp5_fig3}
\end{subfigure}
\hfill
\begin{subfigure}{.49\textwidth}
  \centering
  \includegraphics[width=\textwidth]{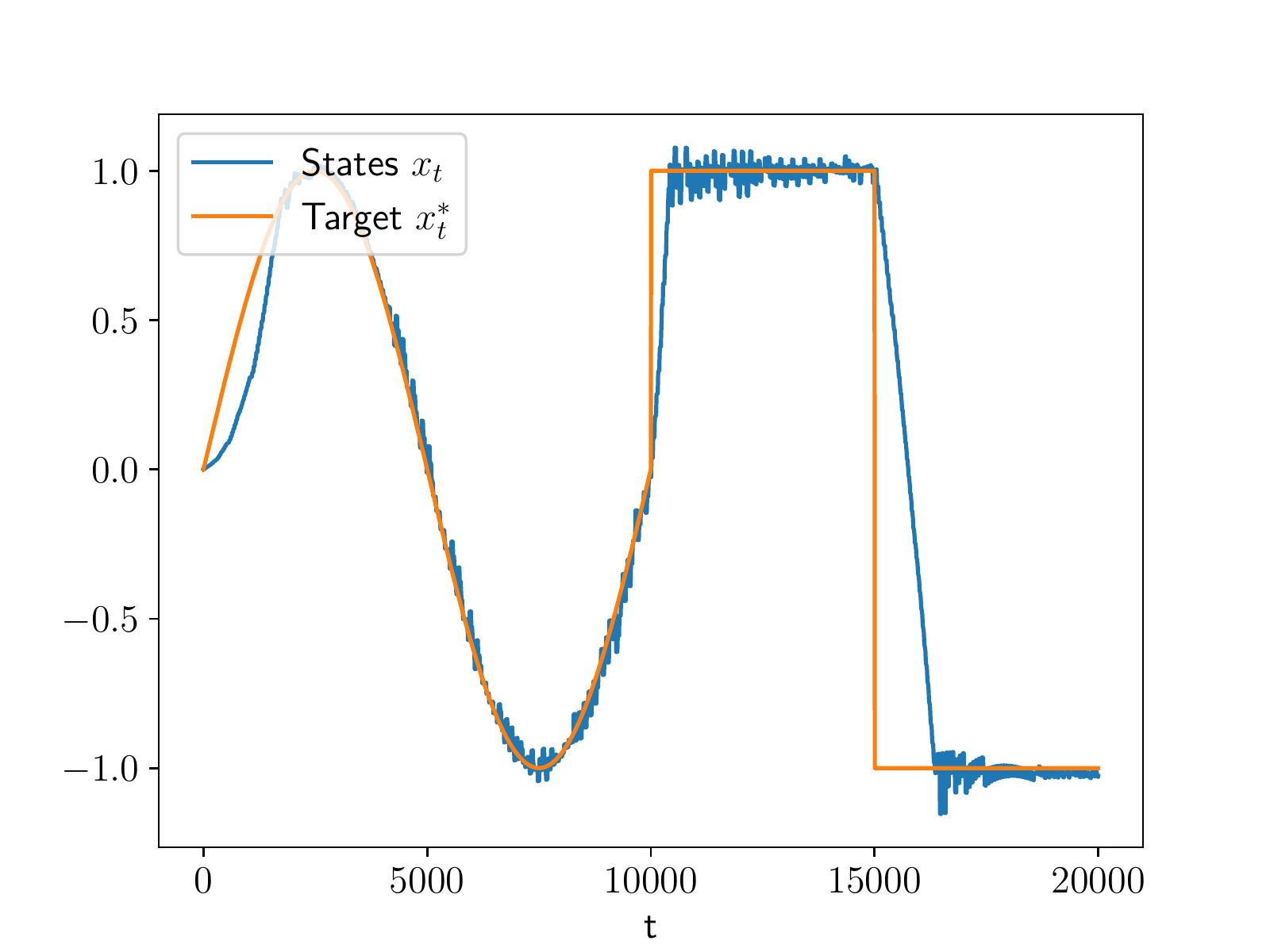}
  \caption{$x^*_t$ is a composite signal.}
  \label{figure:exp5_fig4}
\end{subfigure}
\caption{Testing the controller (Algorithm~\ref{algorithm:controller}) in $\R$.}
\label{figure:control}
\end{figure}

Empirical results for our controller are presented in Figure~\ref{figure:control}. Compared to the tracking results in online learning (Figure~\ref{figure:alg7shifted}), we shoot for lower bandwidth since the dynamics introduce additional fluctuations. Figure~\ref{figure:exp5_fig1} considers a fixed target $x^*_t=1$. In Figure~\ref{figure:exp5_fig2}, $x^*_t$ is a square wave with period 12000. In Figure~\ref{figure:exp5_fig3}, $x^*_t$ is a sinusoidal wave with period 10000. Finally, we consider a composite target in Figure~\ref{figure:exp5_fig4}:

\begin{equation*}
x^*_t=\begin{cases}
\sin(\pi t/5000), &\textrm{if~}t<T/2,\\
1, &\textrm{if~}T/2\leq t< 3T/4,\\
-1, &\textrm{otherwise}. 
\end{cases}
\end{equation*}

\paragraph{Two-dimensional control}We also test the controller in a two-dimensional state space. Here, $d_x=d_u=2$,
\begin{equation*}
A_t=\begin{bmatrix}
0.55 & 0.3 \\
0 & 0.55
\end{bmatrix}+I_2\cdot0.05\cos(\pi t/10000),
\end{equation*}
\begin{equation*}
B_t=I_2\cdot\spar{0.95+0.05\cos(\pi t/5000)},
\end{equation*}
where $I_2$ is the two-dimensional identity matrix. Same as before, $U=5$, $\kappa=1$ and $\gamma=0.4$. 

The loss functions are still $l^*_t(x,u)=\norms{x-x^*_t}$, therefore $L^*=1$. For all $t\in\N$, the disturbances are
\begin{equation*}
w_t=0.05\sin(\pi t/4000)\cdot[1,-1]^\top.
\end{equation*}

Same as before, $H=8$, and we choose $\eps_0=0.2$. The task is to track a circular reference trajectory (in an adversarial manner): 
\begin{equation*}
x^*_t=\begin{cases}
[t/4000,0]^\top, &\textrm{if~}t\leq 4000,\\
[\cos(\pi(t-4000)/8000),\sin(\pi(t-4000)/8000)]^\top, &\textrm{if~}4000< t\leq 20000.
\end{cases}
\end{equation*}

\begin{figure}[!h]
\vspace{-10pt}
    \centering
    \includegraphics[width=0.5\textwidth]{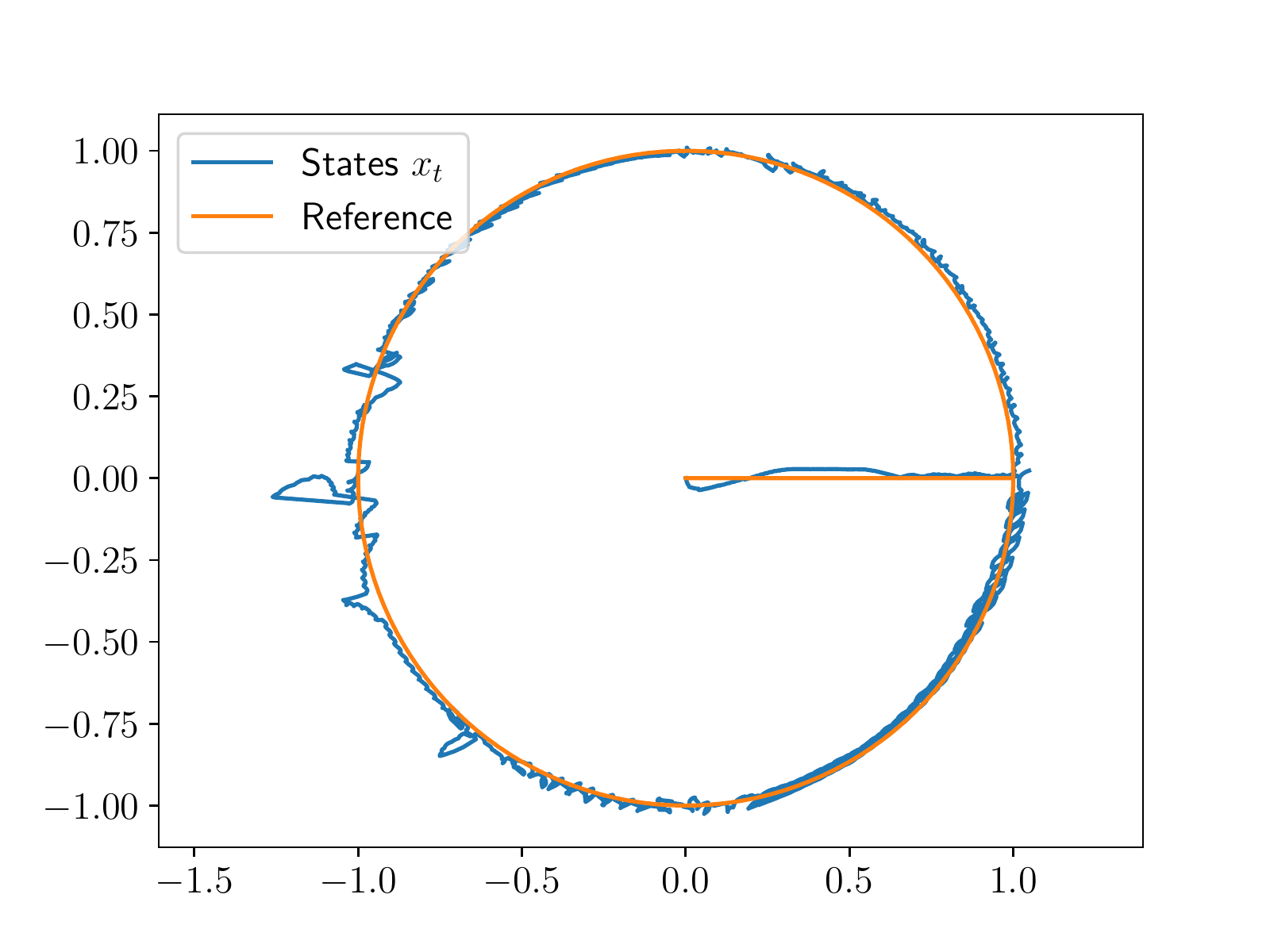}
    \caption{Testing the controller (Algorithm~\ref{algorithm:controller}) in $\R^2$. \label{figure:control_circle}}
\end{figure}
The result is shown in Figure~\ref{figure:control_circle}. Both experiments show that the proposed controller tracks the adversarial reference trajectory quite well.

\end{document}